\documentclass[11pt,oneside,leqno]{article}

\RequirePackage{amsthm,amsmath,amsfonts,amssymb}
\RequirePackage[numbers]{natbib}

\RequirePackage[colorlinks,citecolor=blue,urlcolor=blue]{hyperref}
\RequirePackage{graphicx}

\theoremstyle{plain}

\newtheorem{theorem}{Theorem}[section]
\newtheorem{lemma}[theorem]{Lemma}

\newtheorem{corollary}[theorem]{Corollary}

\theoremstyle{remark}
\newtheorem{definition}[theorem]{Definition}

\newtheorem{assumption}[theorem]{Assumption}

\usepackage[english]{babel}

\usepackage{amssymb}
\usepackage{algorithm}
\usepackage{algorithmic}
\usepackage{tabularx}
\usepackage{url}
\usepackage{enumitem}
\usepackage{mathrsfs}
\usepackage{mathtools}
\usepackage{tikz}
\usetikzlibrary{patterns}
\usepackage{pifont}
\usepackage{xcolor}

\DeclareMathOperator*{\argmin}{argmin}

\newcommand{\betahat}{{\hat{\beta}}}
\newcommand{\betatilde}{{\tilde{\beta}}}
\newcommand{\betastar}{\beta^*}
\DeclareMathOperator{\sgn}{sgn}
\DeclareMathOperator{\supp}{supp}
\newcommand{\x}{\mathbf{x}}

\title{Adaptive Lasso, Transfer Lasso, and Beyond:\\
An Asymptotic Perspective}
\author{Masaaki Takada\thanks{Toshiba Corporation} \thanks{The Institute of Statistical Mathematics}\\ \texttt{masaaki1.takada@toshiba.co.jp} \and Hironori Fujisawa\footnotemark[2]\\ \texttt{fujisawa@ism.ac.jp}}
\date{}

\begin{document}

\maketitle

\begin{abstract}
This paper presents a comprehensive exploration of the theoretical properties inherent in the Adaptive Lasso and the Transfer Lasso. The Adaptive Lasso, a well-established method, employs regularization divided by initial estimators and is characterized by asymptotic normality and variable selection consistency. In contrast, the recently proposed Transfer Lasso employs regularization subtracted by initial estimators with the demonstrated capacity to curtail non-asymptotic estimation errors. A pivotal question thus emerges: Given the distinct ways the Adaptive Lasso and the Transfer Lasso employ initial estimators, what benefits or drawbacks does this disparity confer upon each method? This paper conducts a theoretical examination of the asymptotic properties of the Transfer Lasso, thereby elucidating its differentiation from the Adaptive Lasso. Informed by the findings of this analysis, we introduce a novel method, one that amalgamates the strengths and compensates for the weaknesses of both methods. The paper concludes with validations of our theory and comparisons of the methods via simulation experiments.
\end{abstract}
    
\section{Introduction}

We consider an ordinary high-dimensional regression problem.
Let $X = (\x_1, \dots, \x_p) = (x_1^\top, \dots, x_n^\top)^\top \in \mathbb{R}^{n \times p}$ and $y \in \mathbb{R}^n$ be a feature matrix and response vector, respectively.
We suppose a true model is linear with independent and identically distributed (i.i.d.) Gaussian noise, that is,
\begin{align}
    y &= X \betastar + \varepsilon, ~
    \varepsilon_i \overset{i.i.d.}{\sim} \mathcal{N}(0, \sigma^2),
\end{align}
where $\betastar \in \mathbb{R}^p$ is a true regression parameter and $\varepsilon \in \mathbb{R}^n$ is a Gaussian noise.
We presume that $\betastar$ is sparse, and designate the active and inactive parameters as 
$S$ and $S^c$, namely $S := \{j : \betastar_j \neq 0\}$ and $S^c := \{ j : \betastar_j = 0 \}$, respectively.

The \textit{Lasso}~\citep{tibshirani1996regression} is a classical regression method for high-dimensional data, defined by
\begin{align}
    \hat{\beta}_n^\mathcal{L} = \argmin_\beta \left\{
    \frac{1}{n} \| y - X \beta \|_2^2 + \frac{\lambda_n}{n} \sum_j |\beta_j|
    \right\}.
    \label{eq:lasso}
\end{align}
Owing to $\ell_1$ regularization, the solution exhibits sparsity.
We denote $\hat{S}_n^\mathcal{L} := \{ j: \betahat_j^\mathcal{L} \neq 0 \}$.

Numerous theoretical studies have elucidated the strengths and limitations of the Lasso.
According to asymptotic theory, the Lasso estimator is consistent if $\lambda_n = o(n)$ and is $\sqrt{n}$-consistent if $\lambda_n = O(\sqrt{n})$~\citep{fu2000asymptotics}.
However, \citep{zou2006adaptive} demonstrates that the Lasso has \textit{inconsistent} variable selection if $\lambda_n = O(\sqrt{n})$, 
while it does not have $\sqrt{n}$-consistency
if $\lambda_n = o(n)$ and $\lambda_n / \sqrt{n} \to \infty$.
Hence, the Lasso cannot achieve both $\sqrt{n}$-consistency and consistent variable selection simultaneously (see Figure~\ref{fig:lasso_lambda} left).

To improve the asymptotic properties of the Lasso, one of the most well-known methods is the \textit{Adaptive Lasso}~\citep{zou2006adaptive,huang2008adaptive}, which is given by
\begin{align}
    \hat{\beta}_n^\mathcal{A} = \argmin_\beta \left\{
    \frac{1}{n} \| y - X\beta \|_2^2 + \frac{\lambda_n}{n} \sum_j w_j |\beta_j|
    \right\},~
    w_j := \frac{1}{|\betatilde_j|^\gamma},
    \label{eq:alasso}
\end{align}
where $\betatilde$ is an initial estimator of the true parameter $\betastar$ and $\gamma > 0$ is a hyperparameter.
We denote $\hat{S}_n^\mathcal{A} := \{ j: \betahat_j^\mathcal{A} \neq 0 \}$.
If $\betatilde$ is a $\sqrt{n}$-consistent estimator, $\lambda_n = o(\sqrt{n})$, and $\lambda_n n^{(\gamma - 1)/2} \rightarrow \infty$, then the Adaptive Lasso satisfies both $\sqrt{n}$-consistency and consistent variable selection, as well as asymptotic normality (Figure~\ref{fig:lasso_lambda} right).
This is known as the \textit{oracle property} because it behaves as if the true active variables were given in advance.
The Adaptive Lasso assumes the existence of a $\sqrt{n}$-consistent initial estimator and uses it as the weight of the $\ell_1$ regularization.

\begin{figure*}[t]
\centering
\begin{minipage}{0.45\textwidth}
\begin{tikzpicture}[scale=4, font = \small]
    \begin{scope}
    
    \def\sqrtn{0.4}
    \def\n{0.8}
    
    \draw[->] (0,0) -- (0,1) node[above] {$\lambda_n$};
    
    \node[left] at (0, \n) {$n$};
    \node[left] at (0, \sqrtn) {$\sqrt{n}$};
    
    \draw (0, \n) -- (0.05, \n - 0.02) -- (0.05, \sqrtn + 0.02) -- (0, \sqrtn);
    \node[right] at (0.05, {0.5*\n + 0.5*\sqrtn}) {(ii)};
    \node[right] at (0.2, {0.5*\n + 0.5*\sqrtn + 0.1}) {\textcolor{purple}{\ding{56}} sub-$\sqrt{n}$-consistent};
    \node[right] at (0.2, {0.5*\n + 0.5*\sqrtn}) {\textcolor{teal}{\ding{52}} consistent variable selection};
    \node[right] at (0.2, {0.5*\n + 0.5*\sqrtn - 0.1}) {~~~ (under the incoherence condition)};
    \draw (0, \sqrtn - 0.03) -- (0.05, \sqrtn - 0.05) -- (0.05, 0);
    \node[right] at (0.05, {0.5*\sqrtn}) {(i)};
    \node[right] at (0.2, {0.5*\sqrtn + 0.05}) {\textcolor{teal}{\ding{52}} $\sqrt{n}$-consistent};
    \node[right] at (0.2, {0.5*\sqrtn - 0.05}) {\textcolor{purple}{\ding{56}} inconsistent variable selection};
    
    \draw[fill=white] (0, \n) circle (0.02);
    \draw[fill=white] (0, \sqrtn) circle (0.02);
    \fill (0, \sqrtn - 0.03) circle (0.02);
    
    \end{scope}
\end{tikzpicture}
\end{minipage}\hfill
\begin{minipage}{0.45\textwidth}
\begin{tikzpicture}[scale=4, font = \small]
    \begin{scope}
    
    \def\sqrtn{0.4}
    \def\sqrtnone{0.1}
    
    \draw[->] (0,0) -- (0,1) node[above] {$\lambda_n$};
    
    \node[left] at (0, \sqrtnone) {$\sqrt{n / n^\gamma}$};
    \node[left] at (0, \sqrtn) {$\sqrt{n}$};
    
    \draw (0, \sqrtn) -- (0.05, \sqrtn - 0.02) -- (0.05, \sqrtnone + 0.02) -- (0, \sqrtnone);
    \node[right] at (0.1, {0.5*\sqrtn + 0.5*\sqrtnone + 0.05}) {\textcolor{teal}{\ding{52}} $\sqrt{n}$-consistent + normality};
    \node[right] at (0.1, {0.5*\sqrtn + 0.5*\sqrtnone - 0.05}) {\textcolor{teal}{\ding{52}} consistent variable selection};
    
    \draw[fill=white] (0, \sqrtn) circle (0.02);
    \draw[fill=white] (0, \sqrtnone) circle (0.02);
    
    \end{scope}
\end{tikzpicture}

\end{minipage}

\caption{
Phase diagrams with the order of $\lambda_n$ for the Lasso~(left) and the Adaptive Lasso~(right).
The Lasso does not achieve $\sqrt{n}$-consistent and consistent variable selection simultaneously, while the Adaptive Lasso satisfies both.
}
\label{fig:lasso_lambda}
\end{figure*}
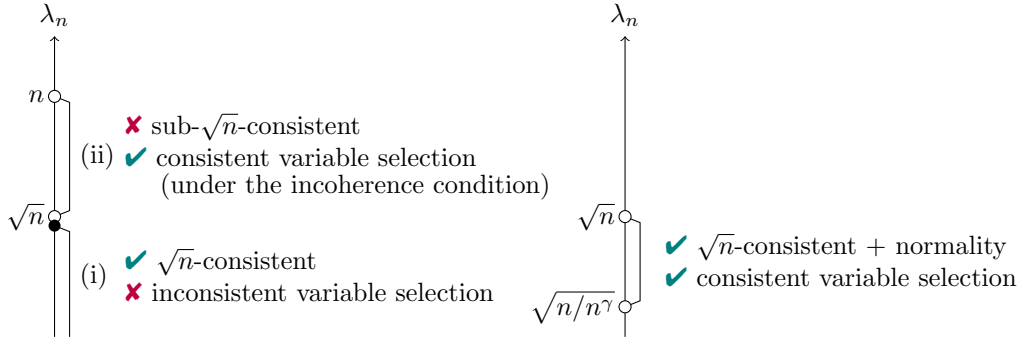

Recently, a different use of an initial estimator has been proposed~\citep{takada2020transfer,bastani2021predicting}, which is given by
\begin{align}
    \hat{\beta}_n^\mathcal{T} = \argmin_\beta \left\{
    \frac{1}{n} \| y - X\beta \|_2^2 + \frac{\lambda_n}{n} \sum_j |\beta_j| + \frac{\eta_n}{n} \sum_j \left| \beta_j - \betatilde_j \right|
    \right\},
    \label{trlasso}
\end{align}
where $\betatilde$ is an initial estimator (``source parameter'' in the field of transfer learning).
We denote $\hat{S}_n^\mathcal{T} := \{ j: \betahat_j^\mathcal{T} \neq 0 \}$.
This method is called \textit{Transfer Lasso}.
The first regularization term in \eqref{trlasso} shrinks the estimator to zero and induces sparsity.
The second regularization term in \eqref{trlasso}, on the other hand, shrinks the estimator to the initial estimator and induces the sparsity of changes from the initial estimator.
The $\ell_1$ regularization of the difference between the initial estimator and the target estimator plays a key role in sparse updating, in which only a small number of parameters are changed from the initial estimator.
Non-asymptotic analysis reveals that a small $\Delta := \betatilde - \betastar$ brings advantageous on its estimation error bounds for the Transfer Lasso over the Lasso~\citep{takada2020transfer}.

The Adaptive Lasso and the Transfer Lasso have similarities and differences.
They are similar in that they both use an initial estimator in $\ell_1$ regularization.
However, the way the initial estimator is used is different: the Adaptive Lasso uses the parameter ``divided'' by the initial estimator in the regularization, whereas Transfer Lasso uses the parameter ``subtracted'' by the initial estimator in the regularization.
In addition, the original motivations are different:
The Adaptive Lasso aims to reduce estimation bias as well as satisfy consistency in variable selection, whereas the Transfer Lasso aims to sparsify both the estimator itself and the change from the initial estimator, leveraging the knowledge of the initial estimator.

These raise major questions:
How do these similarities and differences between Adaptive Lasso and Transfer Lasso affect the theoretical properties and empirical results of each method?
In this paper, we highlight the asymptotic properties of each method and seek to answer the following research questions.
\begin{enumerate}
    \item 
    Does the Transfer Lasso have the same properties as the Adaptive Lasso? 
    Specifically, does the Transfer Lasso have the oracle property that the Adaptive Lasso has?
    \item 
    Does the Transfer Lasso have different properties from the Adaptive Lasso?
    If so, under what conditions of initial estimators, does the Transfer Lasso have an advantage over the Adaptive Lasso, or vice versa?
    \item If these two methods have their specific advantages and disadvantages, are there any ways to compensate for the disadvantages of both and to reconcile their advantages?
    \item How does the asymptotic property of the estimator change as the order of the hyperparameters changes for each method?
\end{enumerate}

Our theoretical analysis led us to the following findings.
\begin{enumerate}
    \item The Transfer Lasso does not have the oracle property in general. This is an unfavorable property compared to the Adaptive Lasso. 
    \item The Transfer Lasso has an advantage in convergence rate if the initial estimator is estimated from sufficiently large data.
    The Adaptive Lasso, in contrast, does not benefit from such an initial estimator. 
    \item We found that a non-trivial integration of the Adaptive Lasso and the Transfer Lasso provides a combination of the benefits of both. The superiority of this integration was shown by asymptotic analysis and empirical simulations.
    \item We comprehensively analyzed the relation between hyperparameters and asymptotic properties and drew phase diagrams representing them.
    Figure~\ref{fig:trlasso-consistent-region} illustrates the phase diagram of the Adaptive Lasso and the Transfer Lasso, and Figure~\ref{fig:adaptrlasso-consistent-selection-region} illustrates the phase diagram of the proposed method.
    These theoretical results were reproduced empirically by numerical simulations in Figure~\ref{fig:empirical-phase-diagram}.
\end{enumerate}

This paper discusses the above research questions in the following organization.
First, we review the asymptotic properties of the Lasso and Adaptive Lasso (Section 2).
Then, we define a setup for our analysis and theoretically analyze the asymptotic properties of the Adaptive Lasso and the Transfer Lasso (Section 3).
This elucidates the advantages and disadvantages of each method.
Furthermore, to compensate for their disadvantages and to reconcile their advantages, we propose a novel method, which effectively integrates both of them (Section 4).
We demonstrate its superiority through theoretical analysis.
We then compare the Adaptive Lasso, the Transfer Lasso, and their integrated method through numerical experiments (Section 5).
Finally, we provide additional discussion and conclusions (Sections 6 and 7).

\subsection*{Notations}

Consider a vector $v \in \mathbb{R}^{p}$. We denote the element-wise absolute vector by $|v|$, with the $j$-th element given by $|v_j|$. The sign vector is represented as ${\rm sgn}(v)$, with its elements being $1$ for $v_j > 0$, $-1$ for $v_j < 0$, and $0$ for $v_j = 0$. The support set of $v$ is denoted as $\supp(v)$ and defined as $\supp(v) := \{ j \in \{ 1, \dots, p \} | v_j \neq 0 \}$. The $\ell_q$-norm of $v$ is expressed as $\| v \|_q = (\sum_{j=1}^p |v_j|^q )^{1/q}$.

For a matrix $M \in \mathbb{R}^{p \times p}$, we use $M \succeq O$ for a positive semi-definite matrix and $M \succ O$ for a positive definite matrix, implying $v^\top M v \geq 0$ for all $v \in \mathbb{R}^p$ and $v^\top M v > 0$ for all non-zero $v \in \mathbb{R}^p$, respectively.

Given a subset $S$ of $\{ 1, \dots, p \}$, we denote its cardinality as $|S|$, and the complement set as $S^c = \{ 1, \dots, p \} \backslash S$. The vector $v_S$ represents $v$ restricted to the index set $S$. The matrix $M_{S_1 S_2}$ denotes the submatrix with row indices in $S_1$ and column indices in $S_2$.

For sequences $a_n$ and $b_n$, we use $a_n = O(b_n)$ to indicate that $|a_n / b_n|$ converges to a finite value, and $a_n = o(b_n)$ to signify $|a_n / b_n|$ converging to zero as $n \to \infty$.

\section{Literature Review}

We review some asymptotic properties for the Lasso and the Adaptive Lasso based on \cite{fu2000asymptotics} and \cite{zou2006adaptive}, and then present other related studies.
All of the proofs in this section are essentially the same as those in \cite{fu2000asymptotics} and \cite{zou2006adaptive}, but for the sake of readability, we provide them in Appendix~\ref{proof-lasso}.

We make the following assumption throughout this paper as in \cite{fu2000asymptotics,zou2006adaptive}.
\begin{assumption}
\label{assumption:psd}
\begin{align}
    C_n := \frac{1}{n} X^\top X \rightarrow C \succ O ~ (n \rightarrow \infty),
\end{align}
\begin{align}
    \frac{1}{n} \max_i \|x_i\|_2^2 \rightarrow 0 ~ (n \rightarrow \infty).
\end{align}
\end{assumption}
 Let $W$ be a random variable of a Gaussian distribution with mean $0$ and covariance $\sigma^2 C$, that is, $W \sim \mathcal{N} (0, \sigma^2 C)$.

\subsection{Asymptotic Properties for the Lasso}

The Lasso is given by \eqref{eq:lasso}.
According to ~\cite{fu2000asymptotics,zou2006adaptive}, several asymptotic properties have been obtained for the Lasso: consistency (Lemma~\ref{lemma:lasso-consistency} and Corollary~\ref{corollary:lasso-consistency}), convergence rate (Lemma~\ref{lemma:lasso-root-n-consistency}, Corollary~\ref{cor-lasso-root-n-consistency}, Lemma~\ref{lemma:lasso-slower-consistency}, and Corollary~\ref{cor-lasso-slower-consistency}), and variable selection consistency (Lemma~\ref{lemma:lasso-variable-selection}).

\begin{lemma}[Theorem 1 in \cite{fu2000asymptotics} and Lemma 1 in \cite{zou2006adaptive}]
\label{lemma:lasso-consistency}
If $\lambda_n/n \rightarrow \lambda_0 \geq 0$, 
then
\begin{align}
    \betahat_n^\mathcal{L} \rightarrow^p 
    \argmin_\beta \left\{ (\beta - \betastar)^\top C (\beta - \betastar) + \lambda_0 \sum_j |\beta_j| \right\}.
\end{align}
\end{lemma}

% \begin{proof}
% The proof is given in \ref{proof:lasso-consistency}.
% \end{proof}

\begin{corollary}[Consistency for Lasso]
\label{corollary:lasso-consistency}
If $\lambda_n = o(n)$, then $\betahat_n^\mathcal{L}$ is consistent.
\end{corollary}

\begin{lemma}[Theorem 2 in \cite{fu2000asymptotics} and Lemma 2 in \cite{zou2006adaptive}]
\label{lemma:lasso-root-n-consistency}
If $\lambda_n/\sqrt{n} \rightarrow \lambda_0 \geq 0$, then
\begin{align}
    &\sqrt{n} (\betahat_n^\mathcal{L} - \betastar) \\
    \overset{d}{\rightarrow}&
    \argmin_u \left\{
    u^\top C u -2 u^\top W
    + \lambda_0 \sum_j \left( u_j \sgn(\betastar_j) I(\betastar_j \neq 0) + |u_j| I(\betastar_j=0) \right)
    \right\}.\nonumber
\end{align}
\end{lemma}

% \begin{proof}
% The proof is given in \ref{proof:lasso-root-n-consistency}.
% \end{proof}

\begin{corollary}[$\sqrt{n}$-consistency for Lasso]
\label{cor-lasso-root-n-consistency}
If $\lambda_n = O(\sqrt{n})$, then $\betahat_n^\mathcal{L}$ is $\sqrt{n}$-consistent.
\end{corollary}

\begin{lemma}[Inconsistent Variable Selection; Proposition 1 in \cite{zou2006adaptive}]
\label{lemma:lasso-variable-selection}
Let $\hat{S}_n^\mathcal{L} := \{ j: \betahat_j^\mathcal{L} \neq 0 \}$.
If $\lambda_n/\sqrt{n} \rightarrow \lambda_0 \geq 0$, then
\begin{align}
    \limsup_{n\to\infty} P(\hat{S}_n^\mathcal{L} = S) \leq c < 1
\end{align}
where 
$c$ is a constant.
\end{lemma}

% \begin{proof}
% The proof is given in \ref{proof:lasso-variable-selection}.
% \end{proof}

\begin{lemma}[Lemma 3 in \cite{zou2006adaptive}]
\label{lemma:lasso-slower-consistency}
If $\lambda_n / n \rightarrow 0$ and $\lambda_n / \sqrt{n} \rightarrow \infty$, then 
\begin{align}
    \frac{n}{\lambda_n} (\betahat_n^\mathcal{L} - \betastar) 
    \overset{d}{\rightarrow} 
    \argmin_u \left\{ 
    u^\top C u + \sum_{j=1}^p \left( u_j \sgn(\betastar_j) I(\betastar_j \neq 0) + |u_j| I(\betastar_j = 0) \right)\right\}.
\end{align}
\end{lemma}

% \begin{proof}
% The proof is given in \ref{proof:lasso-slower-consistency}.
% \end{proof}

\begin{corollary}[Slower Rate Consistency for Lasso]
\label{cor-lasso-slower-consistency}
If $\lambda_n / n \rightarrow 0$ and $\lambda_n / \sqrt{n} \rightarrow \infty$, then the convergence rate of $\betahat_n^\mathcal{L}$ is slower than $\sqrt{n}$.
\end{corollary}

We first obtain a convergence result for $\lambda_n = O(n)$~(Lemma~\ref{lemma:lasso-consistency}).
If $\lambda_n = o(n)$, then we have consistency for the Lasso~(Corollary~\ref{corollary:lasso-consistency}).
Although $\lambda_n = o(n)$ is sufficient for consistency, it is not always $\sqrt{n}$-consistent.
We obtain an asymptotic distribution for $\lambda_n = O(\sqrt{n})$ (Lemma~\ref{lemma:lasso-root-n-consistency}).
This implies $\sqrt{n}$-consistency for the Lasso (Corollary~\ref{cor-lasso-root-n-consistency}).
Unfortunately, $\lambda_n = O(\sqrt{n})$ leads to inconsistent variable selection~(Lemma~\ref{lemma:lasso-variable-selection}).
This implies that $\lambda_n = O(\sqrt{n})$ achieves $\sqrt{n}$-consistency but inconsistent variable selection for the Lasso.
In contrast, if $\lambda_n$ is greater than $O(\sqrt{n})$ and $\lambda_n = o(n)$, we obtain an asymptotic distribution (Lemma~\ref{lemma:lasso-slower-consistency}).
This implies that the convergence rate is slower than $\sqrt{n}$ (Corollary~\ref{cor-lasso-slower-consistency}), although it can be a consistent variable selection under the incoherence conditions~\cite{zou2006adaptive,zhao2006model}.

Figure~\ref{fig:lasso_lambda} (left) summarizes the asymptotic properties for the Lasso.
It cannot simultaneously achieve both $\sqrt{n}$-consistent estimation and consistent variable selection.
This is a major limitation of the Lasso and is the motivation to develop the Adaptive Lasso.

\subsection{Asymptotic Properties for Adaptive Lasso}
\label{sec:alasso-asymptotic-review}

Adaptive Lasso is given by \eqref{eq:alasso}.
It is known that the Adaptive Lasso has the so-called ``oracle property''~\cite{zou2006adaptive}.

\begin{lemma}[Oracle Property for Adaptive Lasso; Theorem 2 in \cite{zou2006adaptive}]
\label{lemma:adapt-lasso-oracle}
Suppose that $\betatilde_n$ is a $\sqrt{n}$-consistent estimator.
If $\lambda_n/\sqrt{n} \rightarrow 0$ and $\lambda_n n^{(\gamma - 1)/2} \rightarrow \infty$, then the Adaptive Lasso estimator \eqref{eq:lasso} satisfies the oracle property, that is, consistent variable selection and $\sqrt{n}$-consistency with asymptotic normality:
\begin{align}
    \lim_{n\rightarrow \infty} P(\hat{S}_n^\mathcal{A} = S) = 1,
\end{align}
\begin{align}
    \sqrt{n} (\betahat_S^\mathcal{A} - \betastar_S) \overset{d}{\rightarrow} \mathcal{N}(0, \sigma^2 C_{SS}^{-1}).
\end{align}
\end{lemma}

\begin{proof}
The proof is given in \ref{proof:adapt-lasso-oracle}.
\end{proof}

The oracle property demonstrates a clear advantage of the Adaptive Lasso over the Lasso.
With a $\sqrt{n}$-consistent initial estimator, the Adaptive Lasso can simultaneously achieve both $\sqrt{n}$-consistent estimation and consistent variable selection (Figure~\ref{fig:lasso_lambda} right).
Thus, the Adaptive Lasso performs as well as if the true active variables were given in advance.

\subsection{Other Related Work}

Besides the Adaptive Lasso and the Transfer Lasso, several related methods have been studied.
In this subsection, we review related methods in three categories:
(I) methods with the oracle property similar to the Adaptive Lasso,
(II) methods with two-stage estimation to eliminate bias, similar to the Adaptive Lasso, and
(III) methods using the $\ell_1$ norm to transfer knowledge about the source data, similar to the Transfer Lasso.

(I) The oracle property is known to hold not only for the Adaptive Lasso but also for the SCAD~\cite{fan2001variable} and MCP~\cite{zhang2010nearly}.
These methods use nonconvex regularization, instead of using an initial estimator.
Because of the nonconvexity, the algorithm converges to a local minimum and the oracle property holds only for some local minima or under restricted conditions.
The Adaptive Lasso, on the other hand, uses convex regularization and always converges to a global minimum, although it requires an appropriate initial estimator.

(II) The Lasso penalizes the $\ell_1$ norm of the parameters and thus introduces a bias, leading to the failure of the oracle property.
Several two-step estimation methods have been proposed to eliminate the bias~\cite{meinshausen2007relaxed,hastie2020best,chzhen2019lasso}.
In \cite{meinshausen2007relaxed}, after the Lasso estimation in the first stage, the second stage is another Lasso estimation using only the selected variables.
In \cite{hastie2020best}, after the Lasso estimation in the first stage, the second stage is estimated by a linear combination of the first stage estimator and the OLS estimator of the selected variables.
These methods are called Relaxed Lasso.
\cite{chzhen2019lasso} generalized these refitting methods as ``methods that minimize the loss function with regularization and then decrease the loss function without regularization''. 
Based on this idea, they developed several refitting methods.

(III) Regularization of $\ell_1$-norm between target and initial estimators was proposed by \cite{bastani2021predicting,li2022transfer,tian2022transfer} as well as the Transfer Lasso~\cite{takada2020transfer}.
\cite{bastani2021predicting} corresponds to the case where $\lambda_n = 0$ in Transfer Lasso~\cite{takada2020transfer}.
In the TransLasso~\cite{li2022transfer} and its GLM extension~\cite{tian2022transfer}, two-stage estimation methods were proposed for the case of multiple source data, where the initial estimator is estimated using both the source and target data.
The Transfer Lasso~\cite{takada2020transfer}, in contrast, is performed on target data using the initial estimator without the need for source data.

\section{Asymptotic Properties for Adaptive Lasso and Transfer Lasso}
\label{sec:alasso-tlasso}

We will perform asymptotic analysis based on the following general settings throughout this paper.
\begin{assumption}
    \label{assumption:m}
    Let $m \geq 0$ be an integer satisfying $n / m \to r_0 \geq 0$.
    The initial estimator $\betatilde$ is a $\sqrt{m}$-consistent estimator and $z:= \sqrt{m} (\betatilde - \betastar)$ converges to some distribution.
\end{assumption}
Assumption~\ref{assumption:m} implies that the initial estimator is estimated on source data of size $m$, and then the final estimator is estimated on target data of size $n$ using the initial estimator.
The case $m = n ~(r_0 = 1)$ corresponds to the existing results for the Adaptive Lasso, whereas $m \gg n ~(r_0 = 0)$ corresponds to the typical transfer learning setup.
The source and target data are assumed to be independent of each other.
We also make assumption~\ref{assumption:psd} in our analysis.

We note that the initial estimator $\betatilde$ is \textit{not} a fixed (deterministic) source parameter, but an estimator (random variable).
This is the same as the previous studies.
The case where $\betatilde$ is fixed is discussed in Appendix~\ref{appendix:fixed-asymptotics}.

\subsection{Asymptotic Properties for Adaptive Lasso}
\label{sec:alasso-random}

We provide the property of the Adaptive Lasso for an initial estimator with source data of size $m$.
It is straightforward to extend the oracle property for $\sqrt{n}$-consistent initial estimators (Lemma~\ref{lemma:adapt-lasso-oracle}) to $\sqrt{m}$-consistent initial estimators (Lemma~\ref{lemma:adapt-lasso-oracle-large}).

\begin{lemma}[Oracle Property for Adaptive Lasso with Different Sample Size]
\label{lemma:adapt-lasso-oracle-large}
Suppose that $\betatilde$ is a $\sqrt{m}$-consistent estimator.
If $\lambda_n/\sqrt{n} \rightarrow 0$ and $\lambda_n \sqrt{m^{\gamma} / n} \rightarrow \infty$, then the Adaptive Lasso estimator \eqref{eq:alasso} satisfies the oracle property, that is, consistent variable selection and $\sqrt{n}$-consistency with asymptotic normality:
\begin{align}
    \lim_{n\rightarrow \infty} P(\hat{S}_n^\mathcal{A} = S) = 1,
\end{align}
\begin{align}
    \sqrt{n} (\betahat_S^\mathcal{A} - \betastar_S) \overset{d}{\rightarrow} \mathcal{N}(0, \sigma^2 C_{SS}^{-1}).
\end{align}
\end{lemma}
\begin{proof}
    The proof is given in \ref{proof:adapt-lasso-oracle-large}.
\end{proof}

Furthermore, we extensively analyze the convergence rate depending on the hyperparameter $\lambda_n$.
We obtain Theorem~\ref{theorem:alasso-consistency} and Corollary~\ref{theorem:alasso-convergence-rate}.

\begin{theorem}[Asymptotic Distribution for Adaptive Lasso]
\label{theorem:alasso-consistency}
We have the following asymptotic distributions for the Adaptive Lasso estimator \eqref{eq:alasso}.\\
(i) If $\sqrt{m^{\gamma} / n} ~ \lambda_n \to \lambda_1 \geq 0$, then 
\begin{align}
    \sqrt{n} (\betahat_n^\mathcal{A} - \betastar)
    \overset{d}{\to}
    \argmin_{u}
    \left\{  u^\top C u - 2 u^\top W
    + \sum_{j \in S^c} \frac{\lambda_1}{|z_j|^{\gamma}} |u_j|
    \right\}.
\end{align}
(ii) If $\sqrt{m^{\gamma} / n} ~ \lambda_n \to \infty$ and
$\lambda_n / \sqrt{n} \to \lambda_0 \geq 0$, then
\begin{align}
    \sqrt{n} (\betahat_n^\mathcal{A} - \betastar) 
    \overset{d}{\to}
    \argmin_{u \in \mathcal{U}} \left\{ u^\top C u - 2 u^\top W
    + \sum_{j \in S} \lambda_0 \frac{\sgn(\betastar_j)}{|\betastar_j|^{\gamma}} u_j \right\},~
    \mathcal{U} := \left\{ u ~\middle|~ u_{S^c} = 0 \right\}.
\end{align}
(iii) If $\lambda_n / \sqrt{n} \to \infty$ and $\lambda_n / n \to 0$, then
\begin{align}
    \frac{n}{\lambda_n} (\betahat_n^\mathcal{A} - \betastar) \overset{d}{\to}
    \argmin_{u \in \mathcal{U}} \left\{ u^\top C u + \sum_{j \in S} \frac{\sgn(\betastar_j)}{|\betastar_j|^{\gamma}} u_j \right\},~~~
    \mathcal{U} := \left\{ u ~\middle|~ u_{S^c} = 0 \right\}.
\end{align}
\end{theorem}

\begin{proof}
    This is a special case of Theorem~\ref{theorem:adaptrlasso-consistency} and the proof is the same as \ref{proof:adaptrlasso-consistency}.
\end{proof}

\begin{corollary}[Convergence Rate for Adaptive Lasso]
\label{theorem:alasso-convergence-rate}
We have the following convergence rates for the Adaptive Lasso estimator \eqref{eq:alasso}.
\begin{enumerate}[label={$(\roman*)$}]
  \item  If $\sqrt{m^{\gamma} / n} ~ \lambda_n \to \lambda_1 \geq 0$, then the convergence rate is $\sqrt{n}$.
  \item If $\sqrt{m^{\gamma} / n} ~ \lambda_n \to \infty$ and $\lambda_n / \sqrt{n} \to \lambda_0 \geq 0$, then convergence rate is $\sqrt{n}$.
  \item If $\lambda_n / \sqrt{n} \to \infty$ and $\lambda_n / n \to 0$, then the convergence rate is $n / \lambda_n$, which is slower than $\sqrt{n}$.
\end{enumerate}
\end{corollary}

Lemma~\ref{lemma:adapt-lasso-oracle-large} shows that the oracle property still holds for $\sqrt{m}$-consistent estimators, $\lambda_n/\sqrt{n} \rightarrow 0$, and $\lambda_n \sqrt{m^{\gamma} / n} \rightarrow \infty$.
In addition, Theorem~\ref{theorem:alasso-consistency} and Corollary~\ref{theorem:alasso-convergence-rate} show that the convergence rate of the Adaptive Lasso estimator is equal to $\sqrt{n}$ in the case (i) and (ii) and is less than $\sqrt{n}$ in the case (iii).
The condition of (ii) in Theorem~\ref{theorem:alasso-consistency} includes the condition of the oracle property in Lemma~\ref{lemma:adapt-lasso-oracle-large}.
Figure~\ref{fig:trlasso-consistent-region}~(left) illustrates each hyperparameter region in Theorem~\ref{theorem:alasso-consistency} and Corollary~\ref{theorem:alasso-convergence-rate}.

These results imply both an advantage and a disadvantage.
The advantage is that the initial estimator does not necessarily require $\sqrt{n}$-consistent.
The Adaptive Lasso has the oracle property even when the source data is small compared to the target data ($m \lesssim n$) and the initial estimator is less than $\sqrt{n}$-consistent.
The disadvantage of the Adaptive Lasso, however, is that it does not take full advantage even when the sample size of the source data is very large ($m \gg n$).
This is because the convergence rate is equal to $\sqrt{n} ~ (\ll \sqrt{m})$.

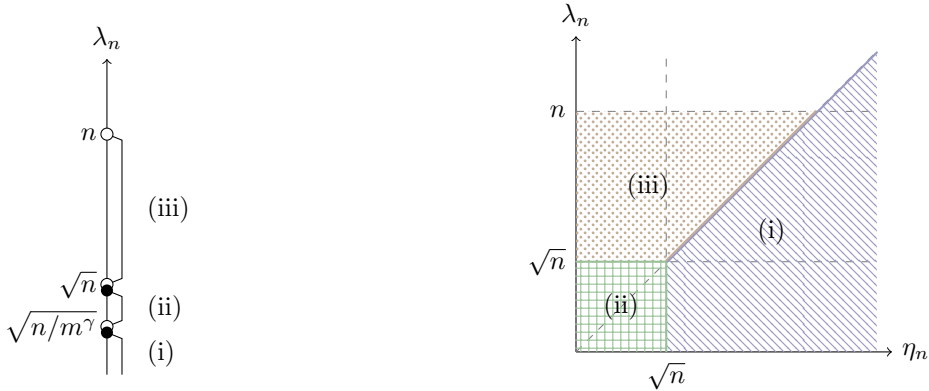
\begin{figure*}[t]
\centering
\begin{minipage}{0.4\textwidth}
\begin{tikzpicture}[scale=4, font = \small]
    \begin{scope}
    
    \def\sqrtn{0.3}
    \def\sqrtnone{0.16}
    \def\sqrtntwo{0.48}
    \def\n{0.8}
    
    \draw[->] (0,0) -- (0,1.05) node[above] {$\lambda_n$};
    
    \node[left] at (0, \n) {$n$};
    \node[left] at (0, \sqrtnone) {$\sqrt{n / m^\gamma}$};
    \node[left] at (0, \sqrtn) {$\sqrt{n}$};
    
    \draw (0, \n) -- (0.05, \n - 0.02) -- (0.05, \sqrtn + 0.02) -- (0, \sqrtn);
    \node[right] at (0.1, {0.5*\n + 0.5*\sqrtn}) {(iii)};
    
    \draw (0, \sqrtn - 0.02) -- (0.05, \sqrtn - 0.02 - 0.02) -- (0.05, \sqrtnone + 0.02) -- (0, \sqrtnone);
    \node[right] at (0.1, {0.5*\sqrtn + 0.5*\sqrtnone - 0.01}) {(ii)};
    
    \draw (0, \sqrtnone-0.02) -- (0.05, \sqrtnone-0.02 - 0.02) -- (0.05, 0);
    \node[right] at (0.1, {0.5*\sqrtnone - 0.01}) {(i)};
    
    \draw[fill=white] (0, \n) circle (0.02);
    \draw[fill=white] (0, \sqrtn) circle (0.02);
    \fill (0, \sqrtn-0.02) circle (0.02);
    \draw[fill=white] (0, \sqrtnone) circle (0.02);
    \fill (0, \sqrtnone-0.02) circle (0.02);
    
    \end{scope}
\end{tikzpicture}
\end{minipage}\hfill
\begin{minipage}{0.45\textwidth}
\begin{tikzpicture}[scale=4, font = \small]
    \begin{scope}
    
    \def\sqrtn{0.3}
    \def\n{0.8}

    \draw[->] (0,0) -- (1.05,0) node[right] {$\eta_n$};
    \draw[->] (0,0) -- (0,1.05) node[above] {$\lambda_n$};

    \draw[dashed, color=gray] (\sqrtn,0) -- (\sqrtn,1);

    \draw[dashed, color=gray] (0,\sqrtn) -- (1,\sqrtn);
    \draw[dashed, color=gray] (0,\n) -- (1,\n);

    \node[below] at (\sqrtn, 0) {$\sqrt{n}$};

    \node[left] at (0, \sqrtn) {$\sqrt{n}$};
    \node[left] at (0, \n) {$n$};

    \draw[dashed, color=gray] (0, 0) -- (1, 1);

    \fill[pattern=north west lines, pattern color=blue!40!black!40!white] (\sqrtn, 0) -- (1, 0) -- (1, 1) -- (\sqrtn, \sqrtn) -- cycle;
    \node at ({\sqrtn + 0.5*(1-\sqrtn)}, {0.4}) {(i)};

    \fill[pattern=grid, pattern color=green!40!black!40!white] (0, 0) -- (\sqrtn, 0) -- (\sqrtn, \sqrtn) -- (0, \sqrtn) -- cycle;
    \node at ({0.5*\sqrtn}, {0.5*\sqrtn}) {(ii)};

    \fill[pattern=crosshatch dots, pattern color=orange!40!black!40!white] (0, \sqrtn) -- (\sqrtn, \sqrtn) -- (\n, \n) -- (0, \n) -- cycle;
    \node at ({0.3*\n}, {0.5*\n + 0.5*\sqrtn}) {(iii)};

    \draw[color=blue!40!black!40!white, thick] ({\sqrtn + 0.002}, {\sqrtn - 0.002}) -- ({1 + 0.002}, {1 - 0.002});
    \draw[color=orange!40!black!40!white, thick] ({\sqrtn - 0.002}, {\sqrtn + 0.002}) -- ({\n - 0.002}, {\n + 0.002});
    \draw[color=green!40!black!40!white, thick] (\sqrtn, 0) -- (\sqrtn, \sqrtn);
    \draw[color=green!40!black!40!white, thick] (0, \sqrtn) -- (\sqrtn, \sqrtn);
                    
    \end{scope}
\end{tikzpicture}

\end{minipage}

\caption{
Phase diagrams with $\lambda_n$ for the Adaptive Lasso in Lemma~\ref{lemma:adapt-lasso-oracle-large}--Theorem~\ref{theorem:alasso-convergence-rate}~(left) and $\lambda_n$ and $\eta_n$ for the Transfer Lasso in Theorem~\ref{theorem:trlasso-consistency}--Theorem~\ref{theorem:trlasso-invariant-selection-consistency}~(right).
The Adaptive Lasso has $\sqrt{n}$-consistency in (i) and (ii) and active variable selection consistency in (ii), but the convergence rate in (iii) is slower than $\sqrt{n}$.
The Transfer Lasso has convergence rates of $\sqrt{m}$, $\sqrt{n}$, and $n/\lambda_n (< \sqrt{n})$ for (i), (ii), and (iii) respectively.
It has invariant variable selection consistency in (i) but does not have active variable selection consistency in (i) and (ii).
}
\label{fig:trlasso-consistent-region}
\end{figure*}

\subsection{Asymptotic Properties for Transfer Lasso}
\label{sec:trlasso-random}

Now we consider the asymptotic properties of the Transfer Lasso.
The Transfer Lasso has two hyperparameters, $\lambda_n$ and $\eta_n$, and various asymptotic properties appear depending on their values.
We first obtain several asymptotic distributions in Theorem~\ref{theorem:trlasso-consistency} and convergence rate in Corollary~\ref{corollary:trlasso-convergence}.
The illustration of the division of cases is shown in Figure~\ref{fig:trlasso-consistent-region}.

\begin{theorem}[Asymptotic distribution for Transfer Lasso]
\label{theorem:trlasso-consistency}
We have the following asymptotic distributions for the Transfer Lasso estimator \eqref{trlasso}.\\
(i) ~ If $\eta_n / \sqrt{n} \to \infty$ and $\lambda_n / \eta_n \to \rho_0$ with $0 \leq \rho_0 < 1$, then
\begin{align}
    \sqrt{m} (\betahat_n^\mathcal{T} - \betastar) \overset{d}{\to} z.
\end{align}
(ii) ~ If $\lambda_n / \sqrt{n} \to \lambda_0 \geq 0$ and $\eta_n / \sqrt{n} \to \eta_0 \geq 0$, then
\begin{align}
    &\sqrt{n} \left( \betahat_n^\mathcal{T} - \betastar \right) \\
    \overset{d}{\to}&
    \argmin_u \left\{ u^\top C u - 2 u^\top W
    + \lambda_0 \left(\sum_{j \in S} u_j \sgn(\betastar_j) + \sum_{j \in S^c} \left| u_j \right| \right) 
    + \eta_0 \sum_{j = 1}^p \left| u_j- \sqrt{r_0}z_j \right|
    \right\}. \nonumber
\end{align}
(iii) ~ If $\lambda_n / \sqrt{n} \to \infty$, $\lambda_n / n \to 0$, and 
$\eta_n / \lambda_n \to \rho_0' \geq 0$, 
then
\begin{align}
    \frac{n}{\lambda_n} (\betahat_n^\mathcal{T} - \betastar) 
    \overset{d}{\to}
    \argmin_u \left\{ u^\top C u + \sum_{j \in S} \left( u_j \sgn(\betastar_j) + \rho_0' \left| u_j \right| \right) + \sum_{j \in S^c} (1 + \rho_0') \left| u_j \right| \right\}.
\end{align}
\end{theorem}

\begin{proof}
    The proof is given in \ref{proof:trlasso-consistency}
\end{proof}

\begin{corollary}[Convergence Rate for Transfer Lasso]
\label{corollary:trlasso-convergence}
We have the following convergence rates for the Transfer Lasso estimator \eqref{trlasso}.
\begin{enumerate}[label={$(\roman*)$}]
  \item If $\eta_n / \sqrt{n} \to \infty$ and $\lambda_n / \eta_n \to \rho_0$ with $0 \leq \rho_0 < 1$, then the convergence rate is $\sqrt{m}$.
  \item  If $\lambda_n / \sqrt{n} \to \lambda_0 \geq 0$ and $\eta_n / \sqrt{n} \to \eta_0 \geq 0$, then the convergence rate is $\sqrt{n}$.
  \item If $\lambda_n / \sqrt{n} \to \infty$, $\lambda_n /n \to 0$, and $\eta_n / \lambda_n \to \rho_0'$ with $0 \leq \rho_0' < 1$, then the convergence rate is $n / \lambda_n$, which is slower than $\sqrt{n}$.
  On the other hand, if $\rho_0' \geq 1$, then the convergence rate is faster than $n / \lambda_n$.
\end{enumerate}
\end{corollary}

Theorem~\ref{theorem:trlasso-consistency} and Corollary~\ref{corollary:trlasso-convergence} show that the Transfer Lasso estimators achieve a convergence rate of $\sqrt{m}$ in the case (i).
This is beneficial when source data is large ($m \gg n$) and
is an advantage for the Transfer Lasso over the Adaptive Lasso.

Next, we provide the results of variable selection consistency.
We first define two types of variable selection consistency.

\begin{definition}[Active Variable Selection Consistency]
    We say that an estimator exhibits \textit{consistent active variable selection} when it estimates the true active variable to be nonzero and the true inactive variable to be zero, that is,
    \begin{align}
        P(\hat{S}_n = S) \to 1.
    \end{align}
    Conversely, we say that an estimator is an \textit{inconsistent active variable selection} when this is not the case, that is,
    \begin{align}
        \limsup_{n \to \infty} P(\hat{S}_n = S) \leq c < 1.
    \end{align} 
\end{definition}

\begin{definition}[Invariant Variable Selection Consistency]
    We say that an estimator exhibits \textit{consistent invariant variable selection} when the true active variable remains invariant from the initial estimator, that is,
    \begin{align}
        P(\betahat_S^\mathcal{T} = \betatilde_S) \to 1.
    \end{align}
\end{definition}

The concepts of ``active'' and ``invariant'' variable selection consistency are distinct yet interconnected.
``Active'' variable selection consistency aligns with conventional variable selection consistency, propelled by the estimator's sparsity.
This property ensures the correct identification of non-zero variables.
In contrast, ``invariant'' variable selection consistency is a unique feature of estimators like the Transfer Lasso, driven by the sparsity of the change from the initial estimator. 
This property ensures that the estimation of the true active variables remains unchanged and inherits the accuracy of the initial estimator.
This could be an advantage when the initial estimator is sufficiently accurate.
When both active and invariant variable selection consistencies are satisfied, the estimator effectively zeroes out the true inactive elements while the true active elements align with the initial estimator's values.
Consequently, sparsity is attained both in the estimator and in its change.

We present results on active/invariant variable selection consistency for the Transfer Lasso in Theorems~\ref{theorem:trlasso-active-selection-consistency}, \ref{theorem:trlasso-varying-selection-consistency}, and \ref{theorem:trlasso-invariant-selection-consistency}.
We assume that the initial estimator $\betatilde$ may not exhibit consistent active variable selection in our variable selection analysis.

\begin{theorem}[Inconsistent Active Variable Selection for Transfer Lasso]
\label{theorem:trlasso-active-selection-consistency}
Suppose that $\betatilde$ is inconsistent with active variable selection.
For the cases (i) and (ii) in Theorem~\ref{theorem:trlasso-consistency}, the Transfer Lasso estimator \eqref{trlasso} yields inconsistent active variable selection, that is,
\begin{align}
    \limsup_{n \to \infty} P(\hat{S}_n^\mathcal{T} = S) \leq c < 1,
\end{align}
where $c$ is a constant.
\end{theorem}

\begin{proof}
    The proof is given in \ref{proof:trlasso-active-selection-consistency}.
\end{proof}

\begin{theorem}[Consistent Invariant Variable Selection for Transfer Lasso]
\label{theorem:trlasso-varying-selection-consistency}
Suppose that $\betatilde$ is inconsistent with active variable selection.
For the case (i) in Theorem~\ref{theorem:trlasso-consistency}, the Transfer Lasso estimator \eqref{trlasso} yields consistent invariant variable selection, that is,
\begin{align}
    P(\betahat_S^\mathcal{T} = \betatilde_S) \to 1.
\end{align}
\end{theorem}

\begin{proof}
    The proof is given in \ref{proof:trlasso-varying-selection-consistency}.
\end{proof}

\begin{theorem}[Inconsistent Invariant Variable Selection for Transfer Lasso]
\label{theorem:trlasso-invariant-selection-consistency}
Suppose that $\betatilde$ is inconsistent with active variable selection.
For the case (ii) in Theorem~\ref{theorem:trlasso-consistency}, the Transfer Lasso estimators \eqref{trlasso} yield inconsistent invariant variable selection, that is,
\begin{align}
    \limsup_{n \to \infty} P(\betahat_S^\mathcal{T} = \betatilde_S) \leq c < 1.
\end{align}
where $c$ is a constant.
\end{theorem}

\begin{proof}
    The proof is given in \ref{proof:trlasso-invariant-selection-consistency}.
\end{proof}

Theorems~\ref{theorem:trlasso-active-selection-consistency}, \ref{theorem:trlasso-varying-selection-consistency}, and \ref{theorem:trlasso-invariant-selection-consistency} unveil the benefits and drawbacks of the Transfer Lasso.
Theorem~\ref{theorem:trlasso-active-selection-consistency} implies that the $\sqrt{m}$-consistent region (i) does not hold active variable selection consistency.
The $\sqrt{n}$-consistent region (ii) does not hold as well.
This is a disadvantage for the Transfer Lasso.
On the other hand, Theorem~\ref{theorem:trlasso-varying-selection-consistency} indicates that the Transfer Lasso in the case (i) has a property of consistent invariant variable selection, which the Adaptive Lasso does not have.
Theorem~\ref{theorem:trlasso-invariant-selection-consistency} implies that the estimators are inconsistent in terms of invariant variable selection in the case (ii).

As shown in Figure~\ref{fig:trlasso-consistent-region}, the Transfer Lasso cannot simultaneously achieve $\sqrt{m}$-consistency and consistent active/invariant variable selection in the regions (i), (ii), and (iii).
This is why we explore a new methodology in the next section.
We note that in regions other than (i), (ii), and (iii) (e.g., boundary regions), the asymptotic property is unclear.
Appendix \ref{sec:additional-tlasso} contains additional results for boundary regions.
At the very least, the above results imply that $\sqrt{m}$-consistency and consistent active/invariant variable selection are incompatible in most regions for the Transfer Lasso.

\section{Beyond Adaptive Lasso and Transfer Lasso}

\label{sec:adaptrlasso}

The Adaptive Lasso and the Transfer Lasso have their advantages and disadvantages, as seen in the previous section.
The Adaptive Lasso achieves both $\sqrt{n}$-consistency and consistent variable selection for $m \leq n$, but its convergence rate is $\sqrt{n} (\ll \sqrt{m})$ for $m \gg n$.
The Transfer Lasso, on the other hand, achieves a convergence rate of $\sqrt{m}$ for $m \gg n$, but it results in inconsistent variable selection.
Are there any ways to combine their benefits and compensate for their drawbacks?

\subsection{Adaptive Transfer Lasso: A Non-Trivial Integration}

To exploit their benefits and compensate for their drawbacks, we integrate the ideas of the Adaptive Lasso and the Transfer Lasso.
We propose a novel method using the initial estimator $\betatilde$ as
\begin{align}
    \label{adaptrlasso}
    \hat{\beta}_n^\# = 
    \argmin_\beta 
    &\left\{ 
    \frac{1}{n} \| y - X\beta \|_2^2 + \frac{\lambda_n}{n} \sum_j v_j |\beta_j| + \frac{\eta_n}{n} \sum_j w_j \left| \beta_j - \betatilde_j \right|
    \right\}, \\
    &~~~~~v_j := \frac{1}{|\betatilde_j|^{\gamma_1}}, ~ w_j := |\betatilde_j|^{\gamma_2},
\end{align}
where
$\gamma_1 \geq 0$ and $\gamma_2 \geq 0$ are new hyperparameters.
We denote $\hat{S}_n^\# := \{ j: \betahat_j^\# \neq 0 \}$.
The weight $v_j = 1 / |\betatilde_j|^{\gamma_1}$ is the same as that of the Adaptive Lasso, whereas the term $w_j = |\betatilde_j|^{\gamma_2}$ is a new non-trivial part.
Because $w_j \to 0$ as $\betatilde_j \to 0$, the effect of transfer learning from the initial estimator disappears for inactive parameters.
We call this method \textit{Adaptive Transfer Lasso} because it is a generalization of the Adaptive Lasso and the Transfer Lasso.
Indeed, if $\eta_n = 0$, then it reduces to the Adaptive Lasso, and if $\gamma_1 = \gamma_2 = 0$, then it reduces to the Transfer Lasso.

\subsection{Asymptotic Properties for Adaptive Transfer Lasso}

We present the asymptotic properties of the Adaptive Transfer Lasso.
The assumptions are the same as for the Adaptive Lasso and the Transfer Lasso.
To derive the asymptotic distribution and convergence rate, we need a more detailed case analysis than before.
The illustration of the division of cases is shown in Figure~\ref{fig:adaptrlasso-consistent-selection-region}.

\begin{theorem}[Asymptotic Distribution for Adaptive Transfer Lasso]
\label{theorem:adaptrlasso-consistency}
We have the following asymptotic distributions for the Adaptive Transfer Lasso estimator \eqref{adaptrlasso}.\\
(i) If $\eta_n / \sqrt{n m^{\gamma_2}} \to \infty$ and $\eta_n / \sqrt{m^{\gamma_1 + \gamma_2}} \lambda_n \to \infty$, then
\begin{align}
    \sqrt{m} (\betahat_n^\# - \betastar) \overset{d}{\to} z.
\end{align}
(ii) If $\sqrt{m^{\gamma_1} / n} ~ \lambda_n \to \infty$,
$\eta_n / \sqrt{n} \to \infty$, $\eta_n / \lambda_n \to \infty$, $\eta_n / \sqrt{m^{\gamma_1 + \gamma_2}}\lambda_n \to 0$, and $\sqrt{m^{\gamma_1}}\lambda_n / \eta_n \to \rho_0 \geq 0$, then
\begin{align}
    \sqrt{m} (\betahat_n^\# - \betastar) \overset{d}{\to} 
    \begin{dcases}
        0 & \text{for}~ j \in S^c,\\
        z_j & \text{for}~ j \in S.
    \end{dcases}
\end{align}
(iii) If $\sqrt{m^{\gamma_1} / n} ~ \lambda_n \to \lambda_1 \geq 0$ and $\eta_n / \sqrt{n} \to \eta_0 \geq 0$, then 
\begin{align}
    &\sqrt{n} (\betahat_n^\# - \betastar) \\
    \overset{d}{\to}&
    \argmin_{u}
    \left\{  u^\top C u - 2 u^\top W
    + \sum_{j \in S^c} \frac{\lambda_1}{|z_j|^{\gamma_1}} |u_j|
    + \sum_{j \in S} \eta_0 \left| \betastar_j \right|^{\gamma_2} \left| u_j - \sqrt{r_0} z_j \right| \right\}.
\end{align}
(iv) If $\sqrt{m^{\gamma_1} / n} ~ \lambda_n \to \lambda_1 \geq 0$, $\eta_n / \sqrt{n} \to \infty$, and $\eta_n / \sqrt{n m^{\gamma_2}} \to \eta_1 \geq 0$, then
\begin{align}
    &\sqrt{n} (\betahat_n^\# - \betastar) \\
    \overset{d}{\to}&
    \argmin_{u \in \mathcal{U}} \left\{ u^\top C u - 2 u^\top W 
    + \sum_{j \in S^c} \left( \frac{\lambda_1}{|z_j|^{\gamma_1}} \left| u_j \right| + \eta_1 |z_j|^{\gamma_2} |u_j - \sqrt{r_0} z_j| \right) \right\},\\
    &\mathcal{U} := \left\{ u ~\middle|~ u_S = r_0 z_S \right\}.
\end{align}
(v) If $\sqrt{m^{\gamma_1} / n} ~ \lambda_n \to \infty$,
$\lambda_n / \sqrt{n} \to \lambda_0 \geq 0$, and $\eta_n / \sqrt{n} \to \eta_0 \geq 0$, then
\begin{align}
    &\sqrt{n} (\betahat_n^\# - \betastar) \\
    \overset{d}{\to}&
    \argmin_{u \in \mathcal{U}} \left\{ u^\top C u - 2 u^\top W
    + \sum_{j \in S} \left( \lambda_0 \frac{\sgn(\betastar_j)}{|\betastar_j|^{\gamma_1}} u_j + \eta_0 \left| \betastar_j \right|^{\gamma_2} \left| u_j - \sqrt{r_0} z_j \right| \right) \right\},\\
    &\mathcal{U} := \left\{ u ~\middle|~ u_{S^c} = 0 \right\}.
\end{align}
(vi) If $\lambda_n / \sqrt{n} \to \infty$, $\lambda_n / n \to 0$, and $\lambda_n / \eta_n \to \infty$, then
\begin{align}
    \frac{n}{\lambda_n} (\betahat_n^\# - \betastar) \overset{d}{\to}
    \argmin_{u \in \mathcal{U}} \left\{ u^\top C u + \sum_{j \in S} \frac{\sgn(\betastar_j)}{|\betastar_j|^{\gamma_1}} u_j \right\},~~~
    \mathcal{U} := \left\{ u ~\middle|~ u_{S^c} = 0 \right\}.
\end{align}
\end{theorem}

\begin{proof}
    The proof is given in \ref{proof:adaptrlasso-consistency}
\end{proof}

\begin{corollary}[Convergence Rate for Adaptive Transfer Lasso]
\label{theorem:adaptrlasso-convergence-rate}
We have the following convergence rates for the Adaptive Transfer Lasso estimator  \eqref{adaptrlasso}.
\begin{enumerate}[label={$(\roman*)$}]
  \item $\eta_n / \sqrt{n m^{\gamma_2}} \to \infty$ and $\eta_n / \sqrt{m^{\gamma_1 + \gamma_2}} \lambda_n \to \infty$, then the convergence rate is $\sqrt{m}$.
  \item $\sqrt{m^{\gamma_1} / n} ~ \lambda_n \to \infty$, $\eta_n / \sqrt{n} \to \infty$, $\eta_n / \lambda_n \to \infty$, and $\eta_n / \sqrt{m^{\gamma_1 + \gamma_2}}\lambda_n \to 0$, then then the convergence rate is $\sqrt{m}$.
  \item  If $\sqrt{m^{\gamma_1} / n} ~ \lambda_n \to \lambda_1 \geq 0$ and $\eta_n / \sqrt{n} \to \eta_0 \geq 0$, then the convergence rate is $\sqrt{n}$.
  \item If $\sqrt{m^{\gamma_1} / n} ~ \lambda_n \to \lambda_1 \geq 0$, $\eta_n / \sqrt{n} \to \infty$, and $\eta_n / \sqrt{n m^{\gamma_2}} \to \eta_1 \geq 0$, then convergence rate is $\sqrt{n}$.
  \item If $\sqrt{m^{\gamma_1} / n} ~ \lambda_n \to \infty$, $\lambda_n / \sqrt{n} \to \lambda_0 \geq 0$, and $\eta_n / \sqrt{n} \to \eta_0 \geq 0$, then convergence rate is $\sqrt{n}$.
  \item If $\lambda_n / \sqrt{n} \to \infty$, $\lambda_n / n \to 0$, and $\lambda_n / \eta_n \to \infty$, then the convergence rate is $n / \lambda_n$, which is slower than $\sqrt{n}$.
\end{enumerate}
\end{corollary}

\begin{figure*}[t]
\centering
% \begin{minipage}{0.45\textwidth}
\begin{tikzpicture}[scale=3.3, font = \footnotesize]
    \begin{scope}
    
    \def\sqrtn{0.3}
    \def\sqrtnone{0.16}
    \def\sqrtntwo{0.48}
    \def\n{0.8}

    \draw[->] (0,0) -- (1.05,0) node[right] {$\eta_n$};
    \draw[->] (0,0) -- (0,1.05) node[above] {$\lambda_n$};

    \draw[dashed, color=gray] (\sqrtn,0) -- (\sqrtn,1);
    \draw[dashed, color=gray] (\sqrtntwo,0) -- (\sqrtntwo,1);

    \draw[dashed, color=gray] (0,\sqrtn) -- (1,\sqrtn);
    \draw[dashed, color=gray] (0,\sqrtnone) -- (1,\sqrtnone);
    \draw[dashed, color=gray] (0,\n) -- (1,\n);

    \node[below] at (\sqrtn-0.02, 0) {$\sqrt{n}$};
    \node[below] at (\sqrtntwo+0.04, 0) {$\sqrt{n m^{\gamma_2}}$};

    \node[left] at (0, \sqrtn) {$\sqrt{n}$};
    \node[left] at (0, \sqrtnone) {$\sqrt{n / m^{\gamma_1}}$};
    \node[left] at (0, \n) {$n$};

    \draw[dashed, color=gray] (0, 0) -- (1, 1);
    \draw[dashed, color=gray] (0, 0) -- (1, \sqrtnone/\sqrtntwo);

        \fill[pattern=north east lines, pattern color=blue!40!black!40!white] (\sqrtntwo, 0) -- (1, 0) -- (1, \sqrtnone/\sqrtntwo) -- (\sqrtntwo, \sqrtnone) -- cycle;
    \node at ({\sqrtntwo + 0.5*(1-\sqrtntwo)}, {0.4*\sqrtnone/\sqrtntwo}) {(i)};

    \fill[pattern=north west lines, pattern color=blue!40!black!40!white] (\sqrtn, \sqrtnone) -- (\sqrtntwo, \sqrtnone) -- (1, \sqrtnone/\sqrtntwo) -- (1, 1) -- (\sqrtn, \sqrtn) -- cycle;
    \node at ({\sqrtn + 0.5*(1-\sqrtn)}, {0.5*\sqrtnone + 0.5*(1-\sqrtnone/\sqrtntwo)}) {(ii)};

    \fill[pattern=grid, pattern color=green!40!black!40!white] (0, 0) -- (\sqrtn, 0) -- (\sqrtn, \sqrtnone) -- (0, \sqrtnone) -- cycle;
    \node at ({0.5*\sqrtn}, {0.5*\sqrtnone}) {(iii)};

    \fill[pattern=horizontal lines, pattern color=green!40!black!40!white] (\sqrtn, 0) -- (\sqrtntwo, 0) -- (\sqrtntwo, \sqrtnone) -- (\sqrtn, \sqrtnone) -- cycle;
    \node at ({\sqrtn + 0.5*(\sqrtntwo-\sqrtn)}, {0.5*\sqrtnone}) {(iv)};

    \fill[pattern=vertical lines, pattern color=green!40!black!40!white] (0, \sqrtnone) -- (\sqrtn, \sqrtnone) -- (\sqrtn, \sqrtn) -- (0, \sqrtn) -- cycle;
    \node at ({0.5*\sqrtn}, {0.5*\sqrtnone + 0.5*\sqrtn}) {(v)};

    \fill[pattern=crosshatch dots, pattern color=orange!40!black!40!white] (0, \sqrtn) -- (\sqrtn, \sqrtn) -- (\n, \n) -- (0, \n) -- cycle;
    \node at ({0.3*\n}, {0.5*\n + 0.5*\sqrtn}) {(vi)};
    
    \end{scope}
\end{tikzpicture}
\par
% \end{minipage}
% \caption{
% A phase diagram with $\lambda_n$ and $\eta_n$ for Adaptive Transfer Lasso
% in Theorem~\ref{theorem:adaptrlasso-consistency} and Corollary~\ref{theorem:adaptrlasso-convergence-rate}. 
% Their estimators are $\sqrt{m}$-consistent in (i - ii), $\sqrt{n}$-consistent in (iii - v), and sub-$\sqrt{n}$-consistent in (vi).
% }
% \label{fig:adaptrlasso-consistent-region}
% \end{figure*}
% \begin{figure*}[t]
% \centering
\begin{minipage}{0.5\textwidth}
\begin{tikzpicture}[scale=3.3, font = \footnotesize]
    \begin{scope}
    
    \def\sqrtn{0.3}
    \def\sqrtnone{0.16}
    \def\sqrtntwo{0.48}
    \def\n{0.8}

    \draw[->] (0,0) -- (1.05,0) node[right] {$\eta_n$};
    \draw[->] (0,0) -- (0,1.05) node[above] {$\lambda_n$};

    \draw[dashed, color=gray] (\sqrtn,0) -- (\sqrtn,1);
    \draw[dashed, color=gray] (\sqrtntwo,0) -- (\sqrtntwo,1);

    \draw[dashed, color=gray] (0,\sqrtn) -- (1,\sqrtn);
    \draw[dashed, color=gray] (0,\sqrtnone) -- (1,\sqrtnone);
    \draw[dashed, color=gray] (0,\n) -- (1,\n);

    \node[below] at (\sqrtn-0.02, 0) {$\sqrt{n}$};
    \node[below] at (\sqrtntwo+0.04, 0) {$\sqrt{n m^{\gamma_2}}$};

    \node[left] at (0, \sqrtn) {$\sqrt{n}$};
    \node[left] at (0, \sqrtnone) {$\sqrt{n / m^{\gamma_1}}$};
    \node[left] at (0, \n) {$n$};

    \draw[dashed, color=gray] (0, 0) -- (1, 1);
    \draw[dashed, color=gray] (0, 0) -- (1, \sqrtnone/\sqrtntwo);

    \fill[pattern=north west lines, pattern color=blue!40!black!40!white] (\sqrtn, \sqrtnone) -- (\sqrtntwo, \sqrtnone) -- (1, \sqrtnone/\sqrtntwo) -- (1, 1) -- (\sqrtn, \sqrtn) -- cycle;
    \node at ({\sqrtn + 0.5*(1-\sqrtn)}, {0.5*\sqrtnone + 0.5*(1-\sqrtnone/\sqrtntwo)}) {(ii)};

    \fill[pattern=vertical lines, pattern color=green!40!black!40!white] (0, \sqrtnone) -- (\sqrtn, \sqrtnone) -- (\sqrtn, \sqrtn) -- (0, \sqrtn) -- cycle;
    \node at ({0.5*\sqrtn}, {0.5*\sqrtnone + 0.5*\sqrtn}) {(v)};

    \fill[pattern=crosshatch dots, pattern color=orange!40!black!40!white] (0, \sqrtn) -- (\sqrtn, \sqrtn) -- (\n, \n) -- (0, \n) -- cycle;
    \node at ({0.3*\n}, {0.5*\n + 0.5*\sqrtn}) {(vi)};
    
    \end{scope}
\end{tikzpicture}
\end{minipage}\hfill
\begin{minipage}{0.5\textwidth}
\begin{tikzpicture}[scale=3.3, font = \footnotesize]
    \begin{scope}
    
    \def\sqrtn{0.3}
    \def\sqrtnone{0.16}
    \def\sqrtntwo{0.48}
    \def\n{0.8}

    \draw[->] (0,0) -- (1.05,0) node[right] {$\eta_n$};
    \draw[->] (0,0) -- (0,1.05) node[above] {$\lambda_n$};

    \draw[dashed, color=gray] (\sqrtn,0) -- (\sqrtn,1);
    \draw[dashed, color=gray] (\sqrtntwo,0) -- (\sqrtntwo,1);

    \draw[dashed, color=gray] (0,\sqrtn) -- (1,\sqrtn);
    \draw[dashed, color=gray] (0,\sqrtnone) -- (1,\sqrtnone);
    \draw[dashed, color=gray] (0,\n) -- (1,\n);

    \node[below] at (\sqrtn-0.02, 0) {$\sqrt{n}$};
    \node[below] at (\sqrtntwo+0.04, 0) {$\sqrt{n m^{\gamma_2}}$};

    \node[left] at (0, \sqrtn) {$\sqrt{n}$};
    \node[left] at (0, \sqrtnone) {$\sqrt{n / m^{\gamma_1}}$};
    \node[left] at (0, \n) {$n$};

    \draw[dashed, color=gray] (0, 0) -- (1, 1);
    \draw[dashed, color=gray] (0, 0) -- (1, \sqrtnone/\sqrtntwo);

    \fill[pattern=north east lines, pattern color=blue!40!black!40!white] (\sqrtntwo, 0) -- (1, 0) -- (1, \sqrtnone/\sqrtntwo) -- (\sqrtntwo, \sqrtnone) -- cycle;
    \node at ({\sqrtntwo + 0.5*(1-\sqrtntwo)}, {0.4*\sqrtnone/\sqrtntwo}) {(i)};

    \fill[pattern=north west lines, pattern color=blue!40!black!40!white] (\sqrtn, \sqrtnone) -- (\sqrtntwo, \sqrtnone) -- (1, \sqrtnone/\sqrtntwo) -- (1, 1) -- (\sqrtn, \sqrtn) -- cycle;
    \node at ({\sqrtn + 0.5*(1-\sqrtn)}, {0.5*\sqrtnone + 0.5*(1-\sqrtnone/\sqrtntwo)}) {(ii)};

    \fill[pattern=horizontal lines, pattern color=green!40!black!40!white] (\sqrtn, 0) -- (\sqrtntwo, 0) -- (\sqrtntwo, \sqrtnone) -- (\sqrtn, \sqrtnone) -- cycle;
    \node at ({\sqrtn + 0.5*(\sqrtntwo-\sqrtn)}, {0.5*\sqrtnone}) {(iv)};
    
    \end{scope}
\end{tikzpicture}
\end{minipage}

\caption{
Phase diagrams of convergence rate (top) and active/invariant variable selection (bottom left/right) with $\lambda_n$ and $\eta_n$ for the Adaptive Transfer Lasso
in Theorems~\ref{theorem:adaptrlasso-consistency}, \ref{theorem:adaptrlasso-active-selection-consistency}, \ref{theorem:adaptrlasso-varying-selection-consistency}, and Corollary~\ref{theorem:adaptrlasso-convergence-rate}.
They are $\sqrt{m}$-consistent in (i - ii), $\sqrt{n}$-consistent in (iii - v), and sub-$\sqrt{n}$-consistent in (vi).
% Phase diagrams of variable selection with $\lambda_n$ and $\eta_n$ for the Adaptive Transfer Lasso 
% in Theorems~\ref{theorem:adaptrlasso-active-selection-consistency} and \ref{theorem:adaptrlasso-varying-selection-consistency}. 
They yield consistent active variable selection in (ii), (v), and (vi) (left), while consistent invariant variable selection in (i), (ii), and (iv) (right).
Estimators in (ii) satisfy $\sqrt{m}$-consistency and active/invariant variable selection consistency.
}
\label{fig:adaptrlasso-consistent-selection-region}
\end{figure*}
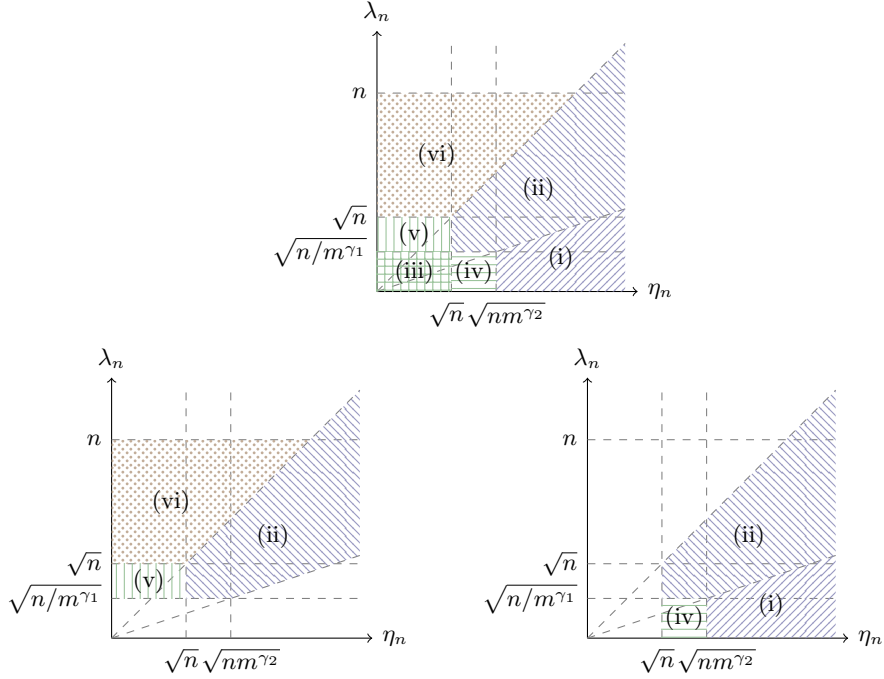

Theorem~\ref{theorem:adaptrlasso-consistency} and Corollary~\ref{theorem:adaptrlasso-convergence-rate} show that the Adaptive Transfer Lasso achieves a convergence rate of $\sqrt{m}$ in the case (i) and (ii).
This property is inherited from the Transfer Lasso.
The asymptotic distribution for (i) is the same as the initial estimator.
On the other hand, the asymptotic distribution for (ii) is remarkable.
The distribution is the same as the initial estimator for the active variables, whereas is zero for the inactive variables.
This implies that inactive parameters shrink to zero quickly.

We also provide the results of active/invariant variable selection consistency for the Adaptive Transfer Lasso.

\begin{theorem}[Consistent Active Variable Selection for Adaptive Transfer Lasso]
\label{theorem:adaptrlasso-active-selection-consistency}
For the cases (ii), (v), and (vi) in Theorem~\ref{theorem:adaptrlasso-consistency}, 
the Adaptive Transfer Lasso yields consistent active variable selection, that is,
\begin{align}
    P(\hat{S}_n^\# = S) \to 1.
\end{align}
\end{theorem}

\begin{proof}
    The proof is given in \ref{proof:adaptrlasso-active-selection-consistency}.
\end{proof}

\begin{theorem}[Consistent Invariant Variable Selection for Adaptive Transfer Lasso]
\label{theorem:adaptrlasso-varying-selection-consistency}
For the cases (i), (ii), and (iv) in Theorem~\ref{theorem:adaptrlasso-consistency}, 
the Adaptive Transfer Lasso yields consistent invariant variable selection, that is,
\begin{align}
    P(\betahat_S = \betatilde_S) \to 1.
\end{align}
\end{theorem}

\begin{proof}
    The proof is given in \ref{proof:adaptrlasso-varying-selection-consistency}.
\end{proof}

Theorems~\ref{theorem:adaptrlasso-active-selection-consistency} and \ref{theorem:adaptrlasso-varying-selection-consistency} imply that both active/invariant variable selection consistency hold in the case (ii).
Hence, we have the following corollary.

\begin{corollary}[Oracle Region for Adaptive Transfer Lasso]
\label{cor:adaptrlasso-oracle-region}
For the case (ii) in Theorem~\ref{theorem:adaptrlasso-consistency},
the Adaptive Transfer Lasso estimator satisfies
\begin{itemize}
\item $\sqrt{m}$-consistent: $\sqrt{m} (\betahat_n^\# - \betastar)$ converges to some distributions,
\item consistent active variable selection: $\hat{S}_n^\# = S$ with probability tending to $1$, 
\item consistent invariant variable selection: $\betahat_S = \betatilde_S$ with probability tending to $1$.
\end{itemize}
\end{corollary}

Corollary~\ref{cor:adaptrlasso-oracle-region} shows that the Adaptive Transfer Lasso incorporates the advantages of both the Adaptive Lasso and the Transfer Lasso.
The hyperparameters $\gamma_1$ and $\gamma_2$ play a crucial role in this property.
If $\gamma_1 = \gamma_2 = 0$, then the region (ii) disappears and it reduces to the Transfer Lasso.
If either $\gamma_1$ or $\gamma_2$ is positive, then the region (ii) appears and it holds $\sqrt{m}$-consistency and active/invariant variable selection consistency.
Both $\gamma_1$ and $\gamma_2$ contribute to expanding the region (ii).
One possible advantage of using $\gamma_2 > 0$ compared to $\gamma_1 > 0$ is that it is stable since there is no division by zero even when the initial estimator is sparse and the values are exactly zero.

Figure~\ref{fig:adaptrlasso-consistent-selection-region} are the phase diagrams that demonstrate the relation between hyperparameters $(\lambda_n, \eta_n)$ and their asymptotic properties for the Adaptive Transfer Lasso.
We see that the region (ii) is the intersection of the part with $\sqrt{m}$-consistency and the part with active/invariant variable selection consistency.
Such a region exists neither in the Adaptive Lasso nor in the Transfer Lasso.

\section{Empirical Results}

We first empirically validate the theoretical properties.
We then compare the performance of various methods through extensive simulations.
Appendix~\ref{sec:additional-experiments} provides additional experimental results.
The codes are available at \url{https://github.com/tkdmah/trlasso}.

\subsection{Empirical Validation of Theory}

In this subsection, we empirically validate the theoretical results for the Transfer Lasso and the Adaptive Transfer Lasso.

We first evaluated the $\ell_2$ norm of the estimation error with respect to sample size.
Theoretically, the convergence rate is $\sqrt{m}$, $\sqrt{n}$, and so on, depending on the hyperparameters.
Assuming the convergence rate is $l(n)$, we have $E[\log \| \betahat - \betastar \|_2] = \text{const.} - \log l(n) $, since $l(n) \| \betahat - \betastar \|_2$ converges to some distribution.
Therefore, by drawing a graph with $E[\log \| \betahat - \betastar \|_2]$ on the vertical axis and $\log n$ on the horizontal axis, the convergence rate can be empirically calculated from its slope.
Assuming $m = n^2$, the slope is $-1/2$ when $\sqrt{n}$-consistent, and $-1$ when $\sqrt{m}$-consistent.

We generated data by $y_i = x_i^\top \betastar + \varepsilon_i ~ (i = 1, \dots, n)$
where $x_i (\in \mathbb{R}^{10}) \overset{i.i.d.}{\sim} \mathcal{N}(0, \Sigma)$, $\Sigma_{jk} = 0.5^{|j - k|}$, 
$\varepsilon_i \overset{i.i.d.}{\sim} \mathcal{N}(0, \sigma^2)$, $\sigma = 1$, and
$\betastar = [3, 1.5, 0, 0, 2, 0, 0, \dots, 0]^\top (\in \mathbb{R}^{10})$ (as in \citep{zou2006adaptive}).
We generated source data of size $m$ and target data of size $n$ with $m = n^2$ and $n = 20, 50, 100, 200, 500, 1000, 2000, 5000$.
The initial estimators were obtained by the ordinary least squares using source data.
The hyperparameters for each method were determined as follows according to Figures~\ref{fig:lasso_lambda}, \ref{fig:trlasso-consistent-region}, and \ref{fig:adaptrlasso-consistent-selection-region}.
\begin{itemize}
    \item Lasso: $\lambda_n = n^{1/4}$ (i) and $\lambda_n = n^{3/4}$ (ii).
    \item Adaptive Lasso: $\gamma = 1$.
    $\lambda_n = n^{-1}$ (i), 
    $n^{1/4}$ (ii), and 
    $n^{3/4}$ (iii).
    \item Transfer Lasso: $(\lambda_n, \eta_n) = (n^{1/2}, n^{3/4})$ (i), 
    $(n^{1/4}, n^{1/4})$ (ii), and 
    $(n^{3/4}, n^{1/2})$ (iii).
    \item Adaptive Transfer Lasso: $\gamma_1 = \gamma_2 = 1$. 
    $(\lambda_n, \eta_n) = (n^{-1/2}, n^2)$ (i), 
    $(n^{1/2}, n^{3/2})$ (ii), 
    $(n ^{-1}, n^{1/4})$ (iii), 
    $(n^{-1}, n)$ (iv), 
    $(1, n^{1/4})$ (v), and
    $(n^{3/4}, n^{1/2})$ (vi).
\end{itemize}
We performed each experiment ten times and evaluated their averages and standard errors.

\begin{figure*}[t]
  \centering
  \includegraphics[width=12cm]{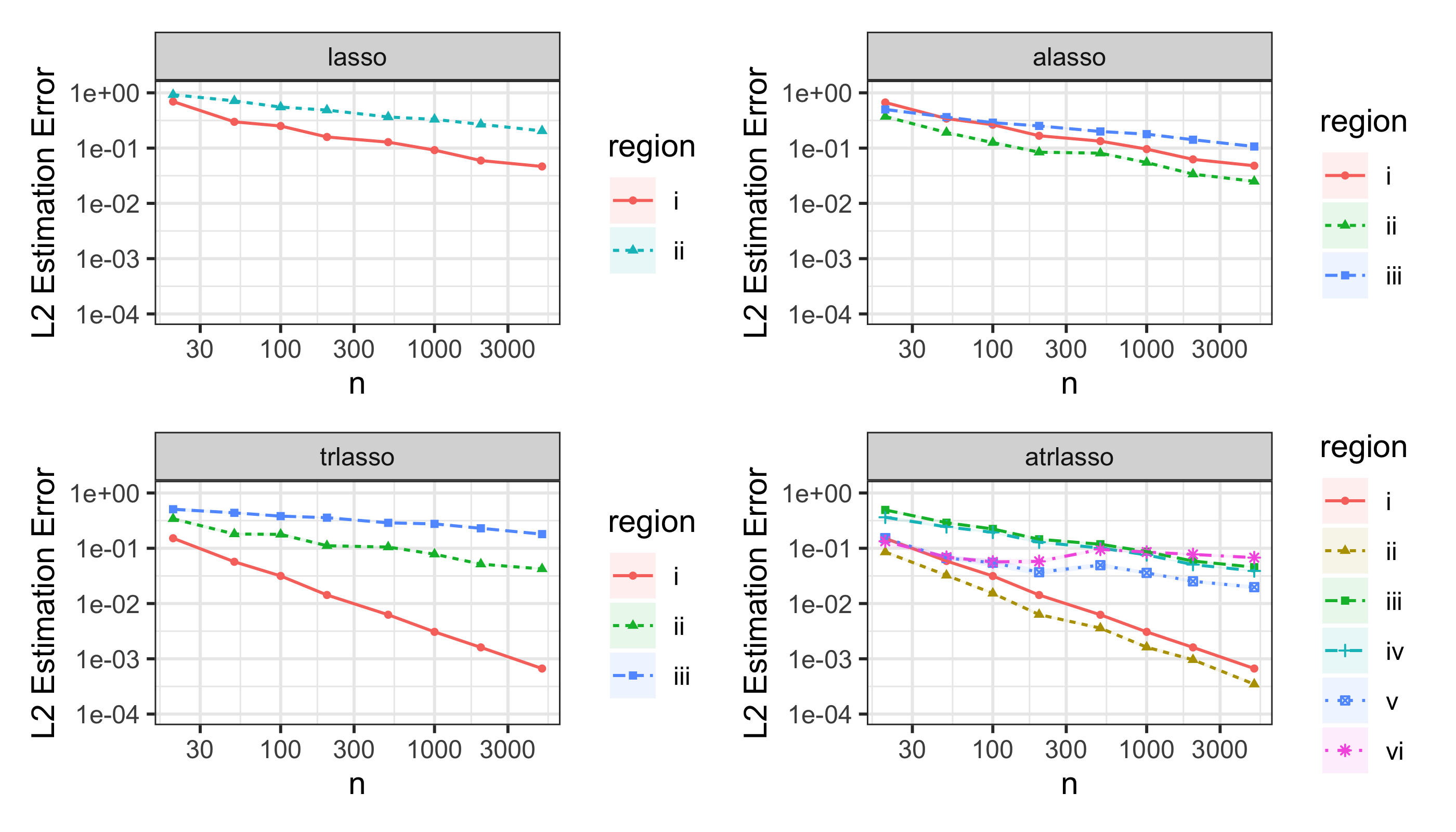}
  \caption{$\ell_2$ estimation errors for the Lasso (top left), the Adaptive Lasso (top right), the Transfer Lasso (bottom left), and the Adaptive Transfer Lasso (bottom right) with respect to sample size. The convergence rates of the Transfer Lasso in the region (i) and the Adaptive Transfer Lasso in the region (i) and (ii) are $\sqrt{m}$ (the slopes are $-1$), whereas the others are $\sqrt{n}$ or less (the slope are $-1/2$ or greater).}
  \label{fig:empirical-convergence}
\end{figure*}

Figure~\ref{fig:empirical-convergence} shows the $\ell_2$ estimation errors for the Lasso, the Adaptive Lasso, the Transfer Lasso, and the Adaptive Transfer Lasso with respect to sample size.
The slopes of the Transfer Lasso in the region (i) and the Adaptive Transfer Lasso in the region (i) and (ii) are $-1$, indicating that the convergence rate was $n = \sqrt{m}$.
For the other methods or regions, the slopes are $-0.5$ or greater, which confirms that the convergence rate is $\sqrt{n}$ or less.
These results were fully consistent with Theorems~\ref{theorem:trlasso-consistency}, \ref{theorem:adaptrlasso-consistency}, and \ref{theorem:adaptrlasso-active-selection-consistency}.

We can observe two potential advantages of Adaptive Transfer Lasso.
First, although the convergence rate (for $n \geq 500$) is $\sqrt{n}$ in regions (v) and (vi), the estimation error is on the line of the convergence rate $\sqrt{m}$ for $n < 500$.
In other words, even in regions where the convergence rate is $\sqrt{n}$, the estimation error can be reduced when the sample size is small.
Second, the estimation error is consistently smaller in the region (ii) than in the region (i), although the convergence rates are comparable between the two regions.
This might be because the estimator in (i) is more likely to be perfectly matched to the initial estimator, whereas the estimator in (ii) is more likely to be matched to the initial estimator for active variables, but not for the inactive variables, and is more likely to be zero.

Having found that the convergence rate can be empirically evaluated accurately, we next empirically drew phase diagrams for the Transfer Lasso and the Adaptive Transfer Lasso as in Figure~\ref{fig:adaptrlasso-consistent-selection-region}.
The experimental setup was the same as in the previous subsection and $m = n^2$.
The hyperparameters $\lambda_n$ and $\eta_n$ were set to $n^\delta$ with $\delta = -2, -1.75, -1.5, \dots, 1.75, 2$, respectively.
The convergence rates were calculated from the slopes of the $\ell_2$ errors for $n=1000$ and $n=5000$.
We plotted the exponential parts of $n$ in the convergence rates, taking the value $1$ if $\sqrt{m}$-consistent and $0.5$ if $\sqrt{n}$-consistent.
Active variable selection consistency was evaluated as the ratio of correctly estimated zeros/non-zeros among all variables for $n = 5000$.
Invariant variable selection consistency was evaluated by the ratio of variables that did not change from the initial estimator among the active variables for $n = 5000$.

\begin{figure*}[t]
    \centering
    \begin{tikzpicture}
        \node[anchor=south west,inner sep=0] (image) at (0,0) {\includegraphics[width=6cm]{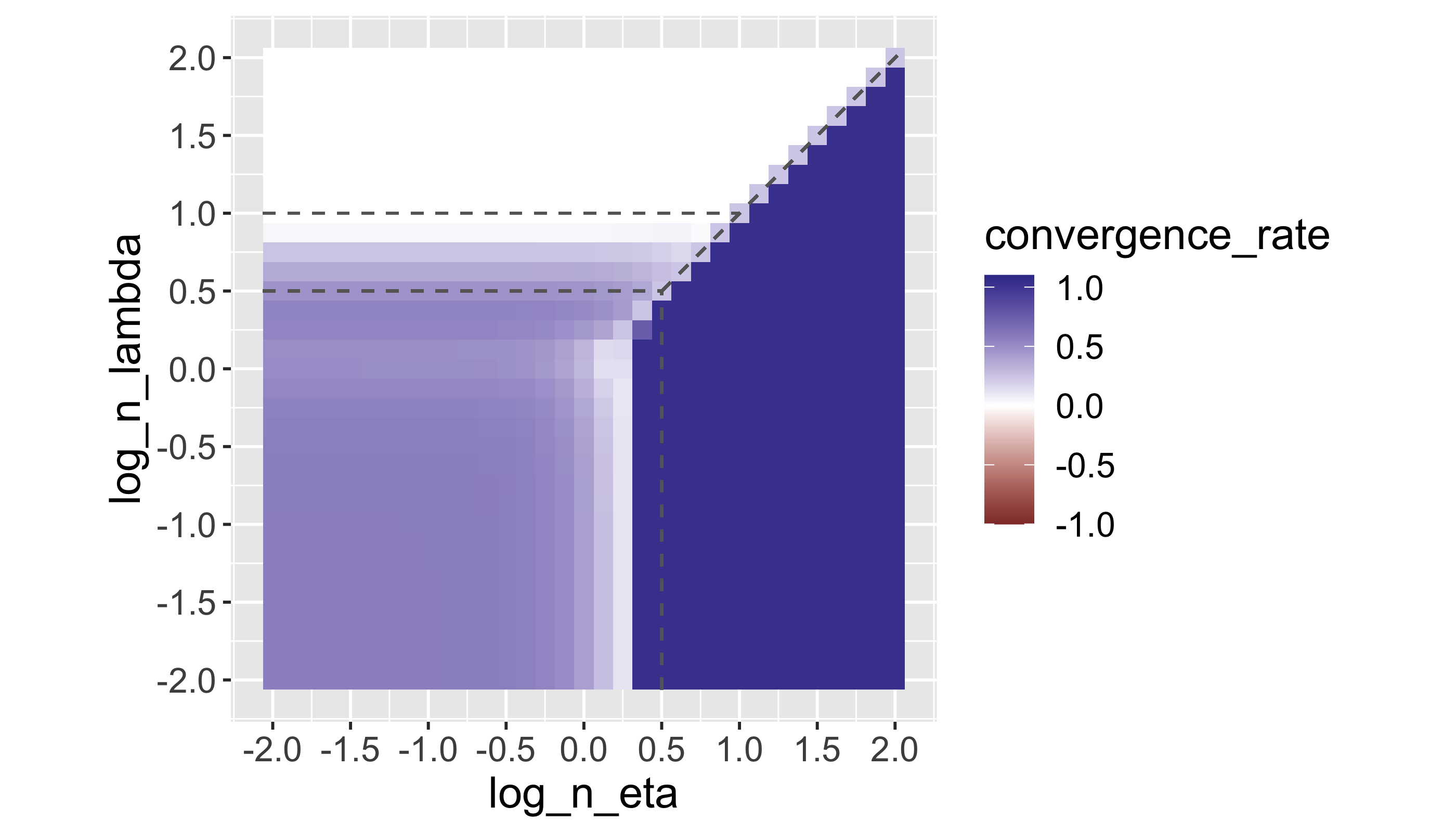}};
        \begin{scope}[x={(image.south east)},y={(image.north west)}]
            \node[text=gray] at (0.53, 0.38) {(i)};
            \node at (0.3, 0.38) {(ii)};
            \node at (0.3, 0.69) {(iii)};
        \end{scope}
    \end{tikzpicture}
    \begin{tikzpicture}
        \node[anchor=south west,inner sep=0] (image) at (0,0) {\includegraphics[width=6cm]{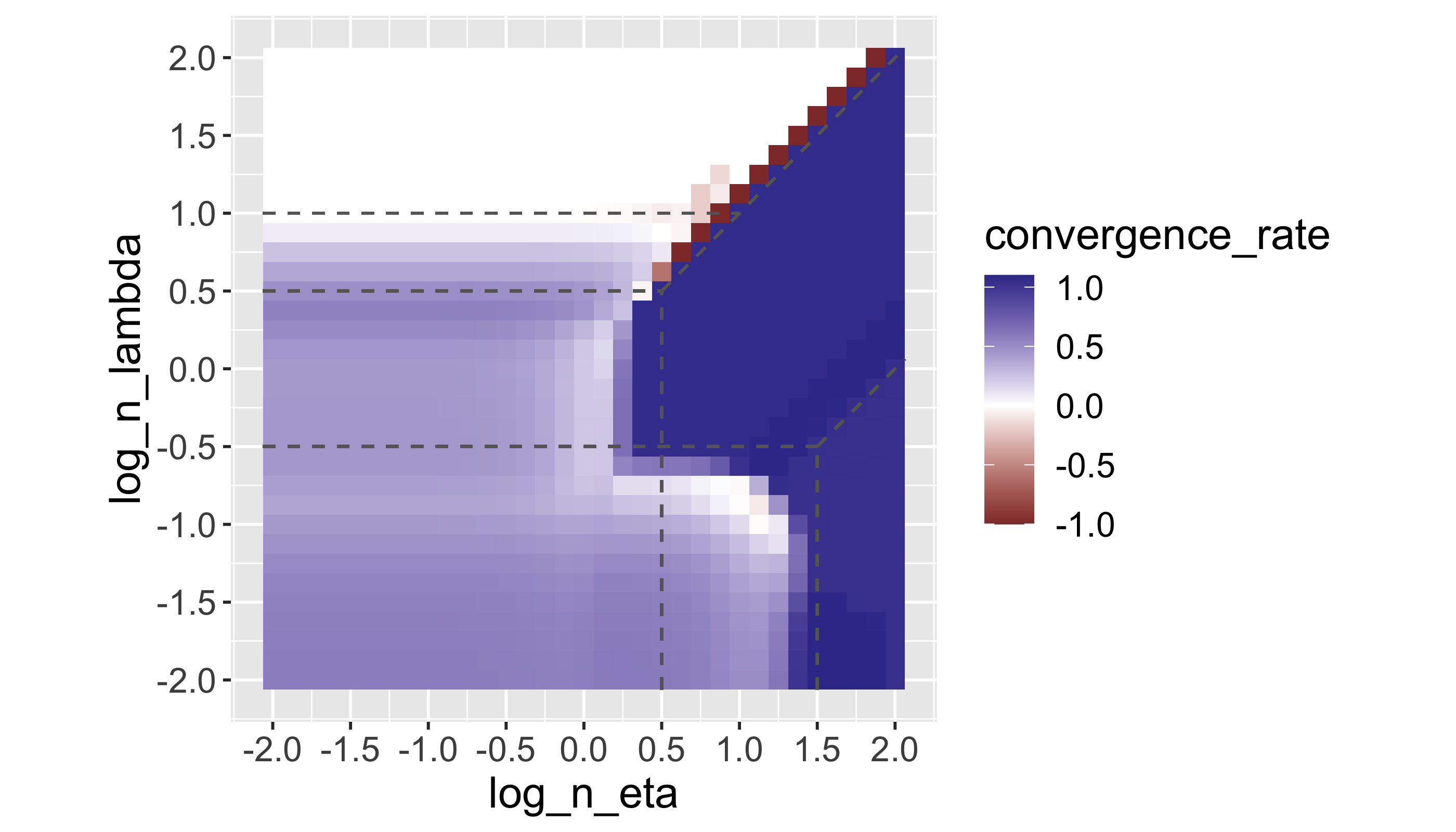}};
        \begin{scope}[x={(image.south east)},y={(image.north west)}]
            \node[text=gray] at (0.59, 0.3) {(i)};
            \node[text=gray] at (0.51, 0.54) {(ii)};
            \node at (0.3, 0.3) {(iii)};
            \node at (0.5, 0.3) {(iv)};
            \node at (0.3, 0.54) {(v)};
            \node at (0.3, 0.69) {(vi)};
        \end{scope}
    \end{tikzpicture}
    \begin{tikzpicture}
        \node[anchor=south west,inner sep=0] (image) at (0,0) {\includegraphics[width=6cm]{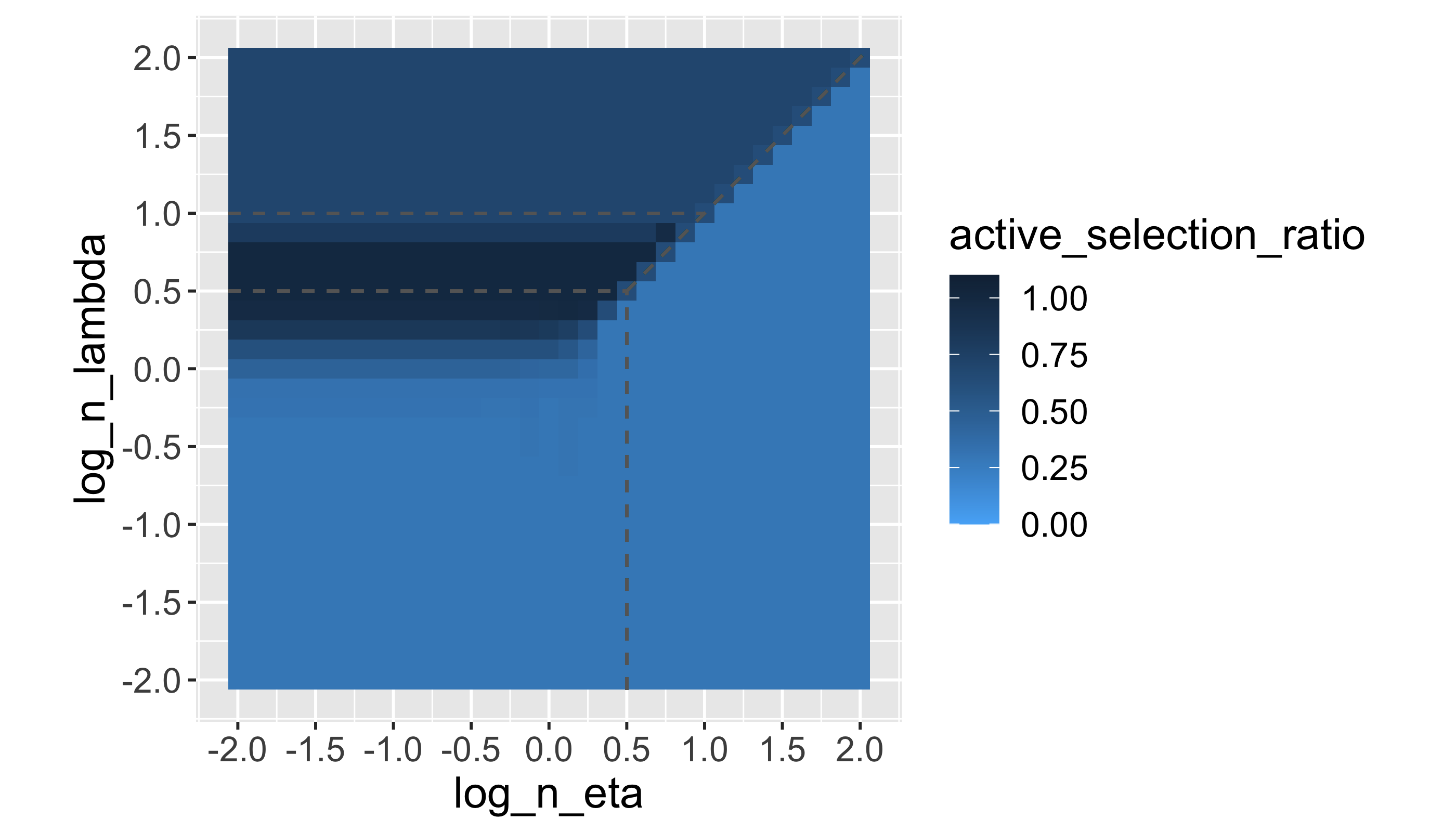}};
        \begin{scope}[x={(image.south east)},y={(image.north west)}]
            \node at (0.51, 0.38) {(i)};
            \node at (0.29, 0.38) {(ii)};
            \node[text=gray] at (0.29, 0.69) {(iii)};
        \end{scope}
    \end{tikzpicture}
    \begin{tikzpicture}
        \node[anchor=south west,inner sep=0] (image) at (0,0) {\includegraphics[width=6cm]{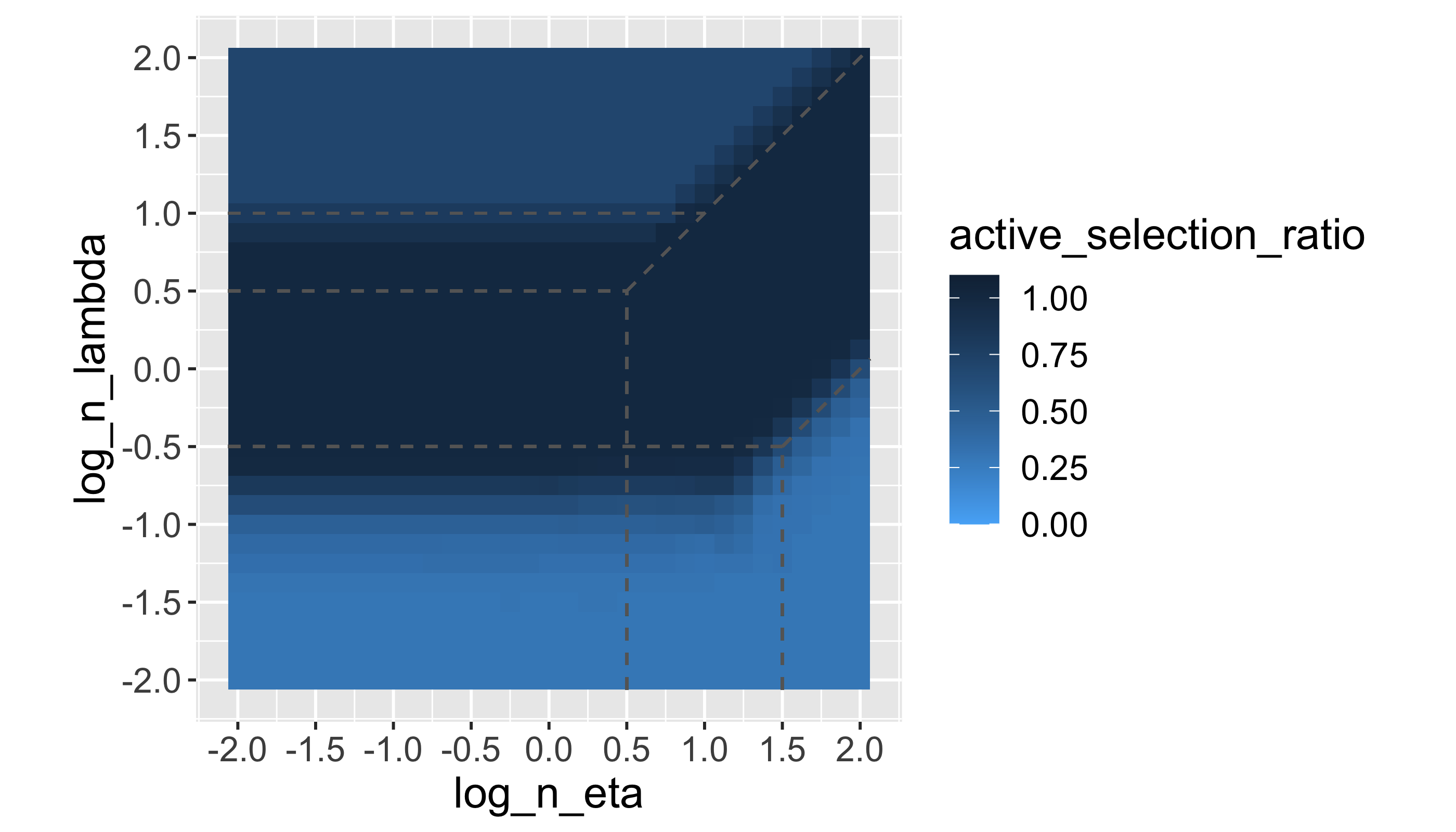}};
        \begin{scope}[x={(image.south east)},y={(image.north west)}]
            \node at (0.57, 0.3) {(i)};
            \node[text=gray] at (0.49, 0.54) {(ii)};
            \node at (0.29, 0.3) {(iii)};
            \node at (0.48, 0.3) {(iv)};
            \node[text=gray] at (0.29, 0.54) {(v)};
            \node[text=gray] at (0.29, 0.69) {(vi)};
        \end{scope}
    \end{tikzpicture}
    \begin{tikzpicture}
        \node[anchor=south west,inner sep=0] (image) at (0,0) {\includegraphics[width=6cm]{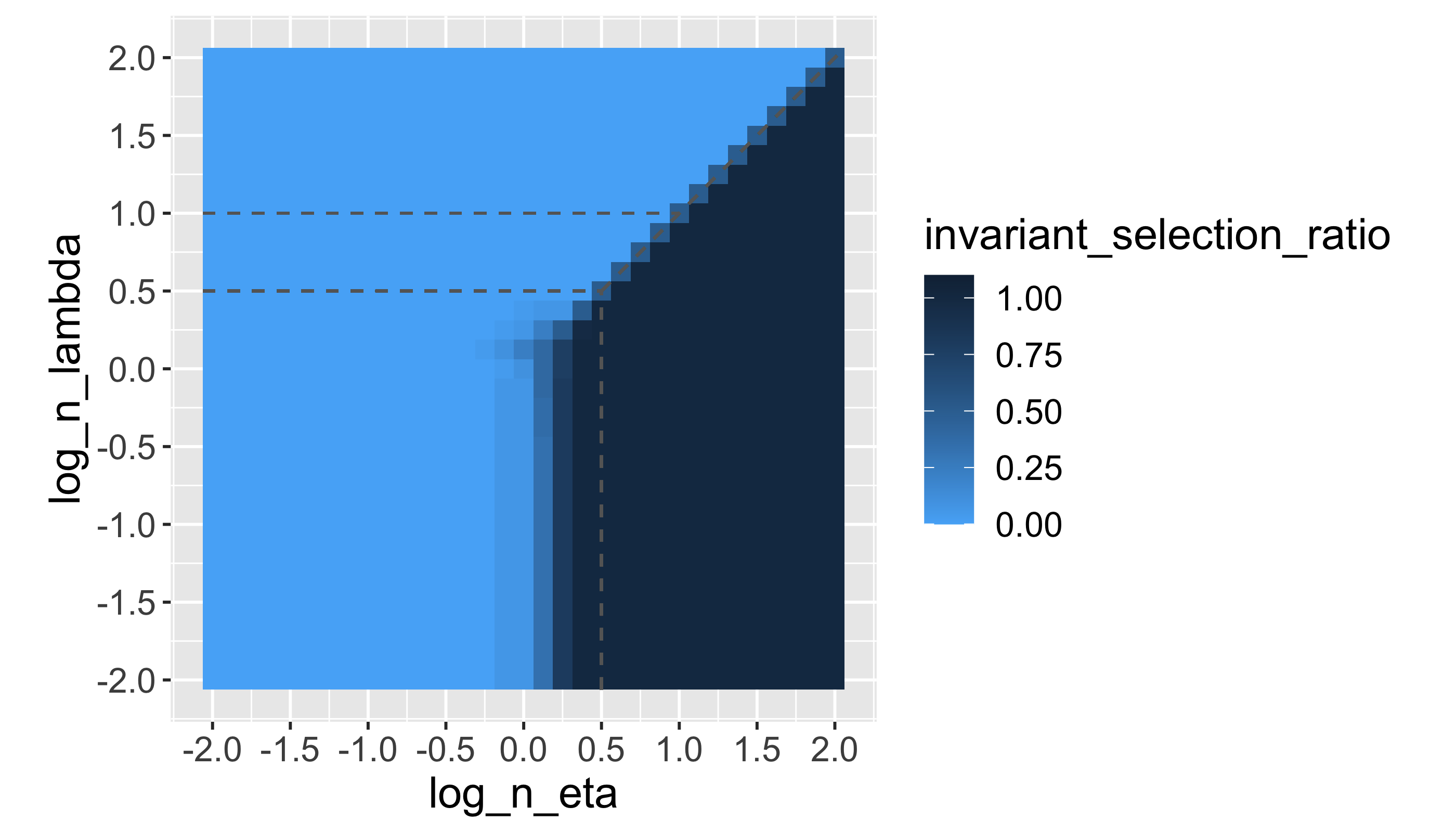}};
        \begin{scope}[x={(image.south east)},y={(image.north west)}]
            \node[text=gray] at (0.5, 0.38) {(i)};
            \node at (0.28, 0.38) {(ii)};
            \node at (0.28, 0.69) {(iii)};
        \end{scope}
    \end{tikzpicture}
    \begin{tikzpicture}
        \node[anchor=south west,inner sep=0] (image) at (0,0) {\includegraphics[width=6cm]{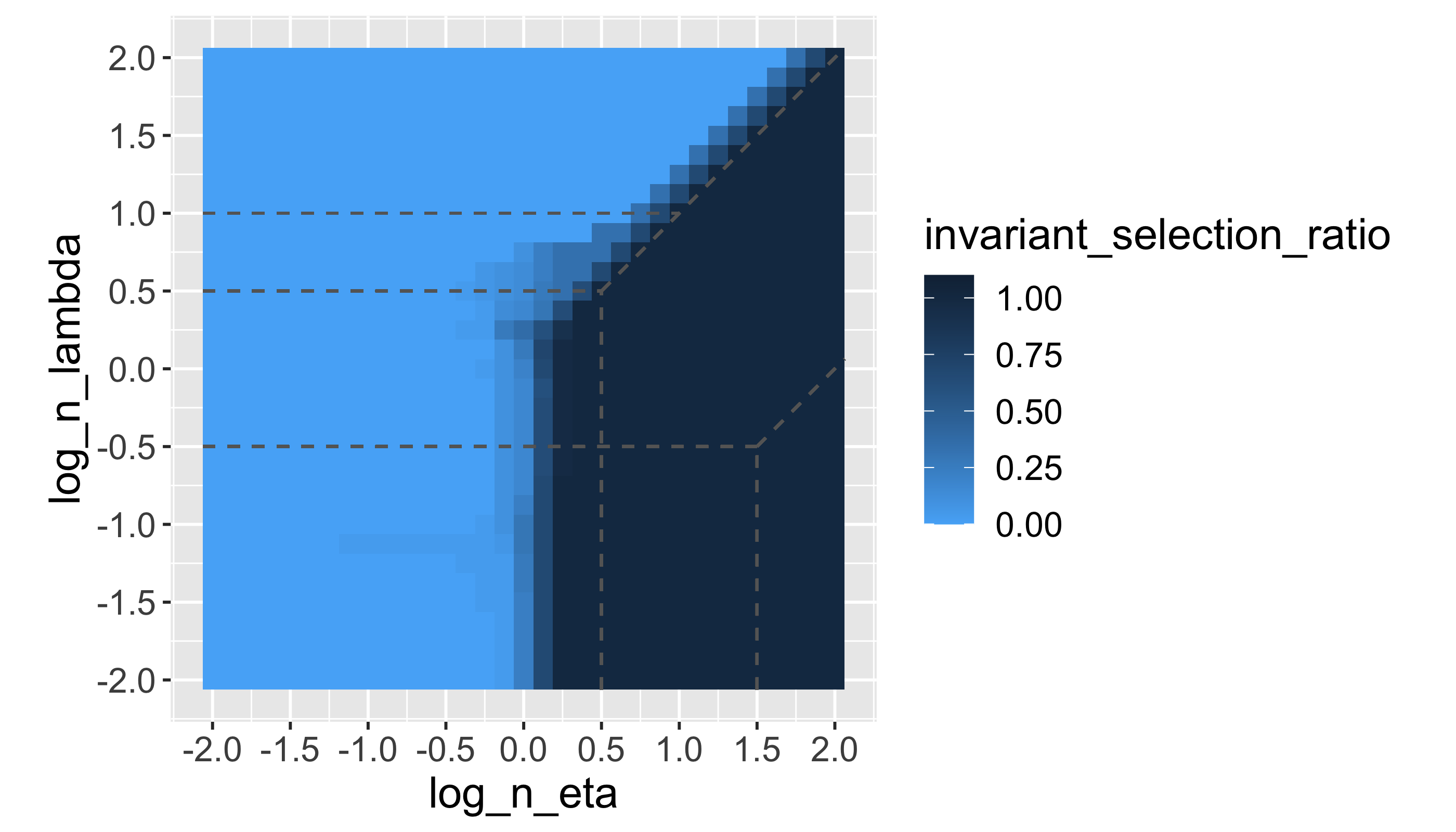}};
        \begin{scope}[x={(image.south east)},y={(image.north west)}]
            \node[text=gray] at (0.55, 0.3) {(i)};
            \node[text=gray] at (0.48, 0.54) {(ii)};
            \node at (0.28, 0.3) {(iii)};
            \node[text=gray] at (0.47, 0.3) {(iv)};
            \node at (0.28, 0.54) {(v)};
            \node at (0.28, 0.69) {(vi)};
        \end{scope}
    \end{tikzpicture}
  \caption{$\log$--$\log$ phase diagrams of convergence rate (top), active variable selection ratio (middle), and invariant variable selection ratio (bottom) for the Transfer Lasso (left) and the Adaptive Transfer Lasso (right). These empirical results confirm the theoretical results of Figure~\ref{fig:trlasso-consistent-region} and \ref{fig:adaptrlasso-consistent-selection-region}.}
  \label{fig:empirical-phase-diagram}
\end{figure*}
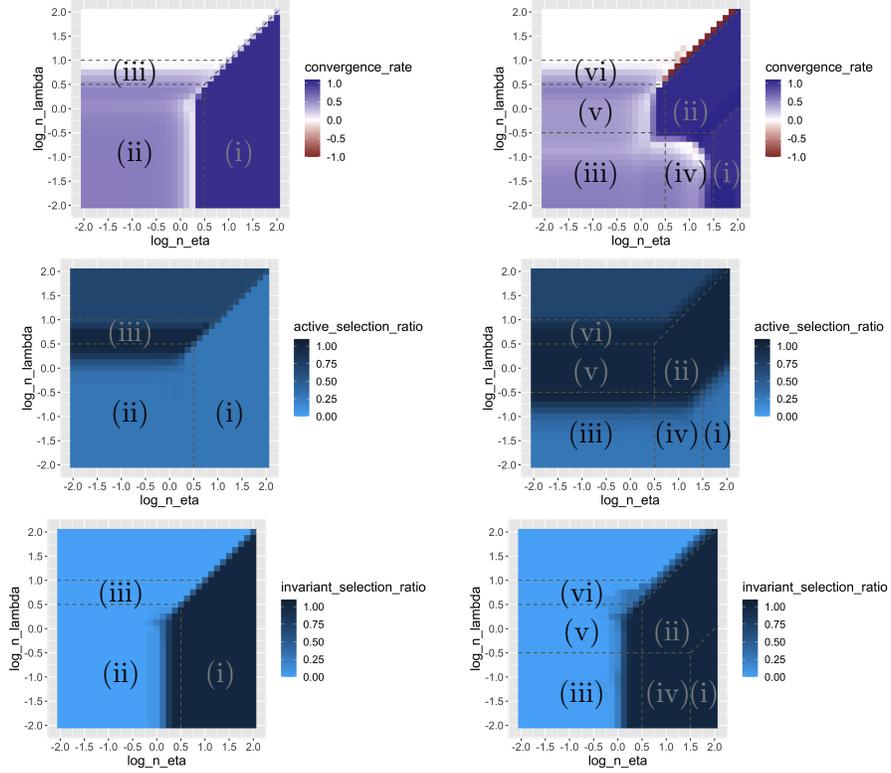

Figure~\ref{fig:empirical-phase-diagram} illustrates the empirical phase diagrams of $\log$--$\log$ scale for the Transfer Lasso and the Adaptive Transfer Lasso.
As Theorems~\ref{theorem:trlasso-consistency}--\ref{theorem:trlasso-invariant-selection-consistency} suggest, the Transfer Lasso achieves both $\sqrt{m}$-consistency and invariant variable selection consistency in the lower right region (i), but does not have active variable selection consistency.
Other regions also do not satisfy these properties simultaneously.
On the other hand, in the Adaptive Transfer Lasso, the upper right region (ii) satisfies the properties of $\sqrt{m}$-consistency and active/invariant variable selection consistency.
The empirical convergence rates and active/invariant variable selection ratios well reproduce Theorems~\ref{theorem:adaptrlasso-consistency}, \ref{theorem:adaptrlasso-active-selection-consistency}, and \ref{theorem:adaptrlasso-varying-selection-consistency} in other regions as well.
These empirical results confirm the theoretical results (Theorems~\ref{theorem:trlasso-consistency}--\ref{theorem:adaptrlasso-varying-selection-consistency}, Figures~\ref{fig:lasso_lambda}--\ref{fig:adaptrlasso-consistent-selection-region}).

\subsection{Empirical Comparison of Methods}

\label{sec:sim-1}

In this subsection, we compare the methods in various experimental settings based on hyperparameter determination by cross-validation.
The experimental settings include various source/target data sample sizes, number of dimensions, signal-to-noise ratios, and initial estimators.
We mainly considered two cases: one with a large amount of source data and the other with the same amount of source data as the target data.

First, we suppose that we have a large amount of source data and its sample size is $m = 10000$.
The simulation setting follows the previous subsections.
We used 
% $\betastar = [3, 1.5, 0, 0, 2, 0, 0, \dots, 0]^\top (\in \mathbb{R}^p)$;
$\sigma = 1, 3, 6, 10$;
$p=10, 20, 50, 100$; and
$n = 10, 20, 50, 100, 200, \dots, 5000, 10000$.

Initial estimators were obtained by the Lasso because the number of dimensions $p$ can be greater than sample size $n$ in this experiment.
We compared other initial estimators including Ridge, Ridgeless~\cite{belkin2020two,hastie2022surprises}, and Lassoless~\cite{mitra2019understanding,li2021minimum} in Appendix~\ref{sec:other-initial-estimators}.
The search spaces were $\gamma = 0.5, 1, 2$ for the Adaptive Lasso; $\alpha := \lambda_n / (\lambda_n + \eta_n) = 0.75, 0.5, 0.25$ for the Transfer Lasso; and $(\gamma_1, \gamma_2) = (0.5, 0.5), (1, 1), (2, 2)$, and $\alpha := \lambda_n / (\lambda_n + \eta_n) = 0.75, 0.5, 0.25$ for the Adaptive Transfer Lasso.
The hyperparameter $\lambda_n$ was determined by 10-fold cross validation with $\lambda_{\min} / \lambda_{\max} = 10^{-6}$ where $\lambda_{\max}$ is automatically determined by Theorem~4 in \cite{takada2020transfer}.
If $|\betatilde_j| \leq 10^{-3}$, then we set $|\betatilde_j| = 10^{-3}$ to avoid division by zero.

We evaluated the performance by two metrics:
$\ell_2$ norm for estimation error and F1 score for variable selection.
The F1 score is a harmonic average of precision and recall, where precision = (the number of correct selected variables) / (the number of selected variables) and recall = (the number of correct selected variables) / (the number of true active variables).
We used the F1 score because it allows us to evaluate the performance of variable selection even when there is an imbalance between the number of active and inactive variables.
We also evaluated other metrics in Appendix~\ref{sec:other-metrics}.
They included RMSE for prediction evaluation and sensitivity, specificity, positive predictive value, and the number of active variables for feature selection evaluation.

\begin{figure*}[t]
  \centering
  \includegraphics[width=10cm]{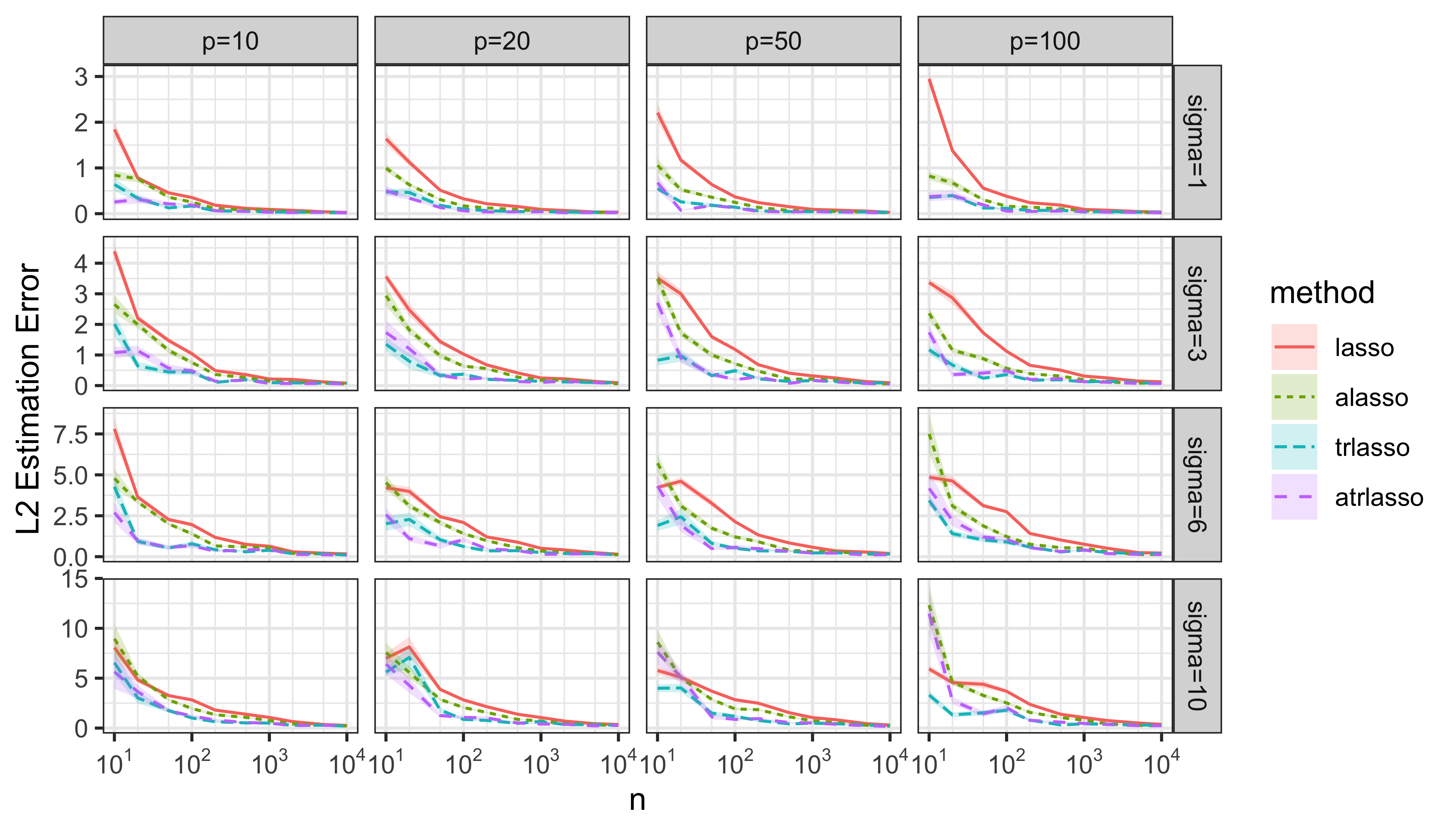}
  x\includegraphics[width=10cm]{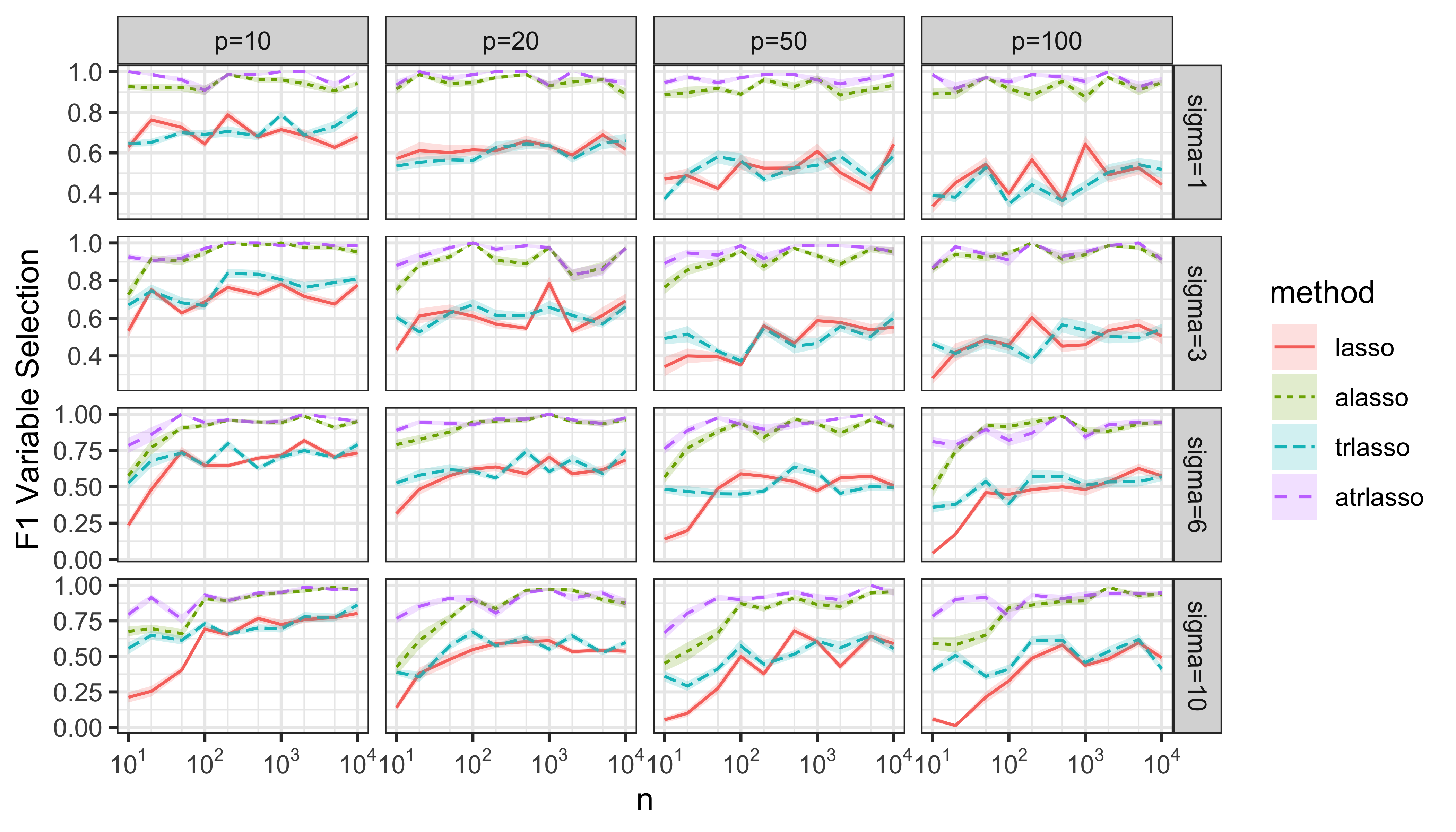}
  \caption{$\ell_2$ estimation errors (top) and variable selection F1-score (bottom) for a large amount of source data.}
  \label{fig:large-estimation}
\end{figure*}

The results are shown in Figure~\ref{fig:large-estimation}.
In terms of estimation errors, the Transfer Lasso and the Adaptive Transfer Lasso outperformed the other methods.
The Adaptive Lasso was superior to the Lasso, but it was inferior to the Transfer Lasso and the Adaptive Transfer Lasso.
In terms of variable selection, the Adaptive Lasso and the Adaptive Transfer Lasso outperformed the others, and the Adaptive Transfer Lasso was slightly superior to the Adaptive Lasso.
These results imply the superiority of the Adaptive Transfer Lasso with initial estimators using large amounts of source data.
The Adaptive Lasso, however, does not fully utilize the initial estimators in this setting.

\begin{figure*}[t]
  \centering
  \includegraphics[width=10cm]{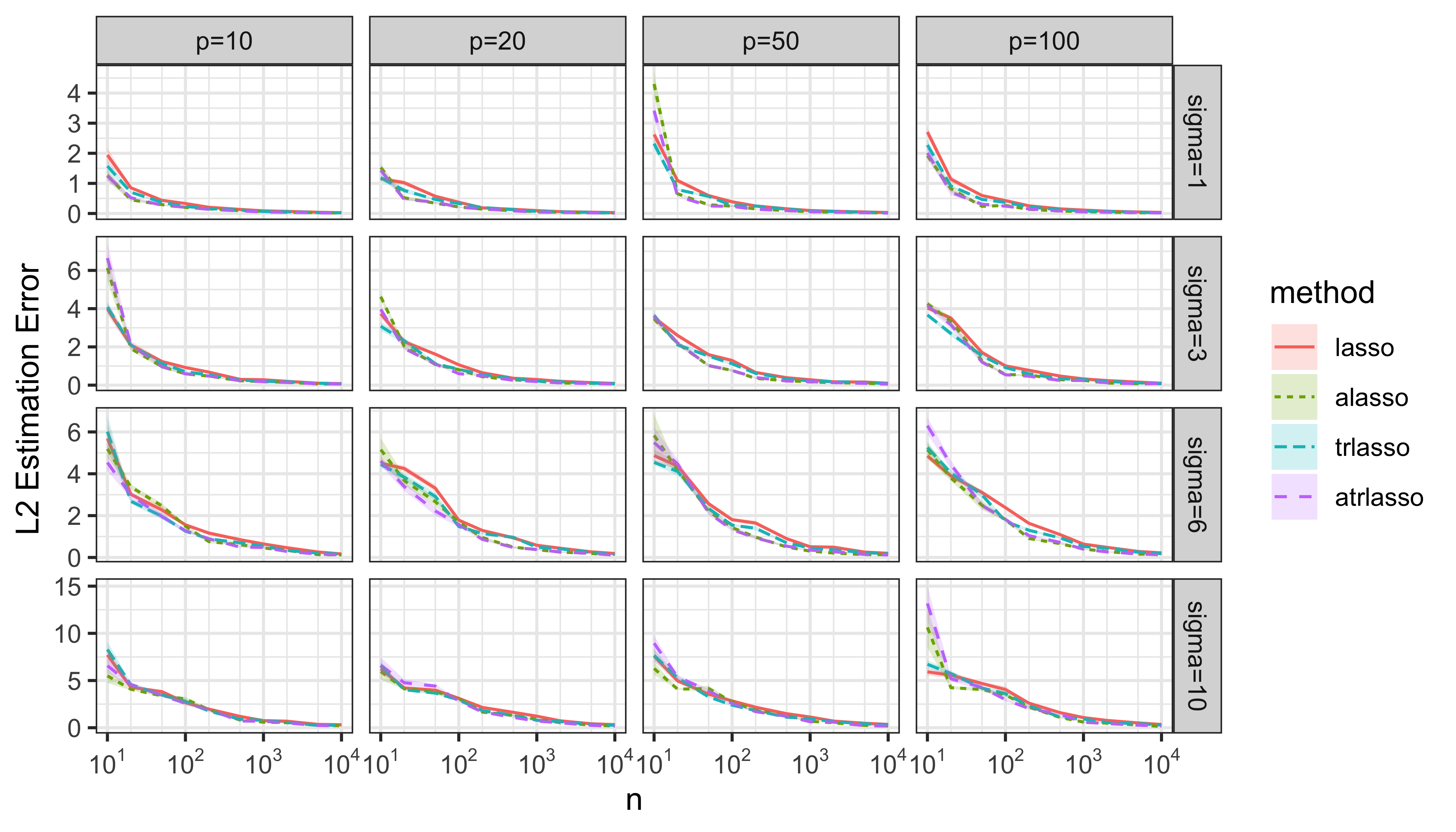}
  \includegraphics[width=10cm]{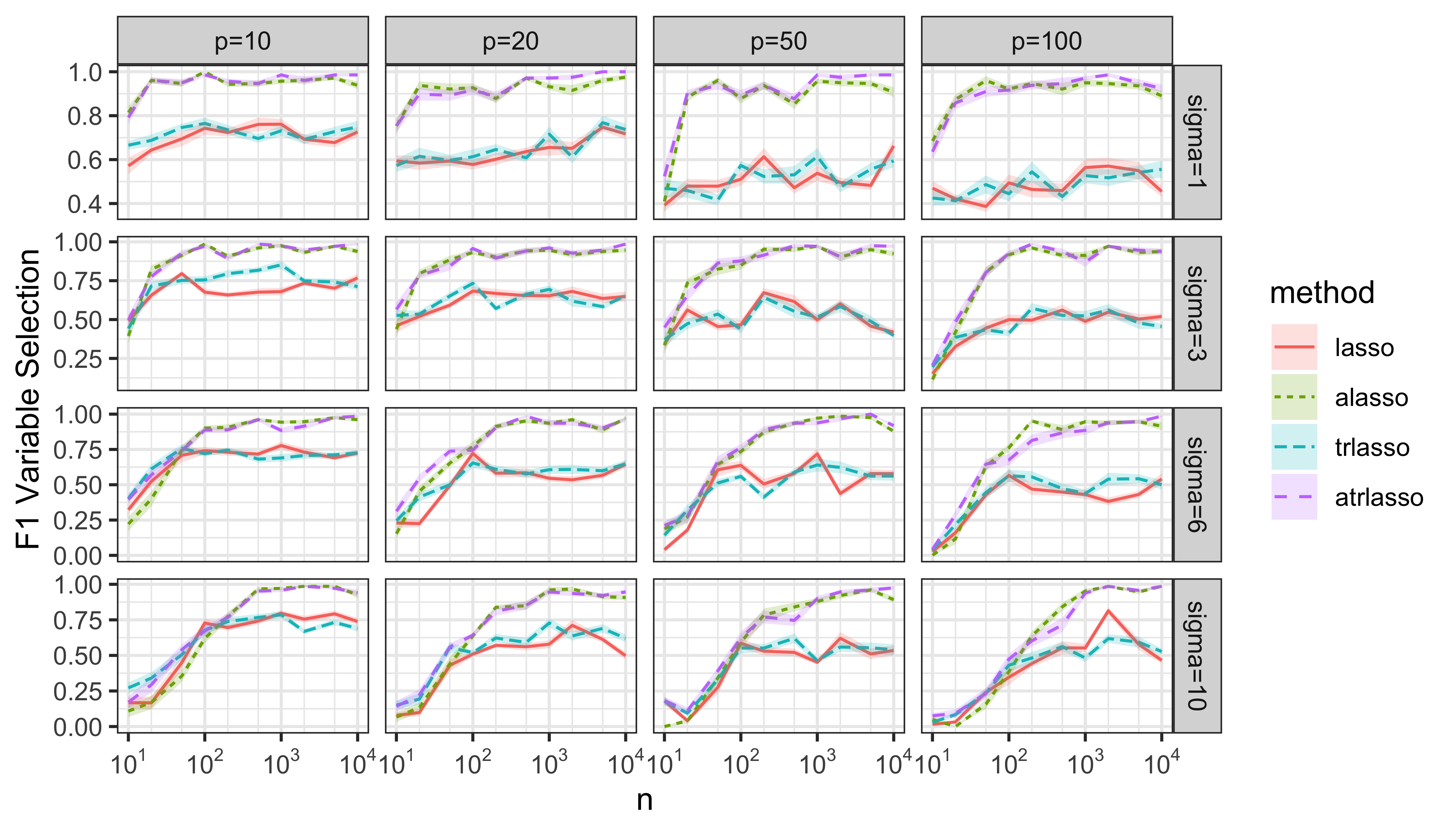}
  \caption{$\ell_2$ estimation errors (top) and variable selection F1-scores (bottom) for a medium amount of source data.}
  \label{fig:medium-estimation}
\end{figure*}

Next, we supposed that we have a medium amount of source data and its sample size is the same as the target data.
We used the same data generation process, comparison methods, and performance measurements as above.

The results are shown in Figure~\ref{fig:medium-estimation}.
All methods were comparable in terms of estimation errors, but in terms of variable selection, the Adaptive Lasso and the Adaptive Transfer Lasso were superior to those of the others.
The Adaptive Lasso and the Adaptive Transfer Lasso had similar performances for both estimation error and variable selection.
This is consistent with our theoretical analyses.

\section{Discussion}

We discuss additional comparisons among methods from two perspectives; regularization contours and prior distributions.
We also discuss future work.

\subsection{Regularization Contours}

The regularization contours help to intuitively capture the strength and pattern of regularization.
Figure~\ref{fig:contours} shows their contours with an initial estimator with a small value $\betatilde_1 = 0.5$ and a large value $\betatilde_2 = 2$.

The contours of the Adaptive Lasso are pointed at the coordinate axes (where some elements are zero) and are especially sharp where the initial estimator is small.
The contours of the Transfer Lasso are pointed at the points where some elements are zero or equal to the initial estimator, but they are not so sharp.
The contours for the Adaptive Transfer Lasso are pointed where some elements are zero or equal to the initial estimator, and the sharpness varies depending on the hyperparameters.
These observations indicate that the Adaptive Transfer Lasso flexibly changes the strength of regularization depending on the initial estimator.

\begin{figure*}[t]
  \centering
  \includegraphics[width=6cm]{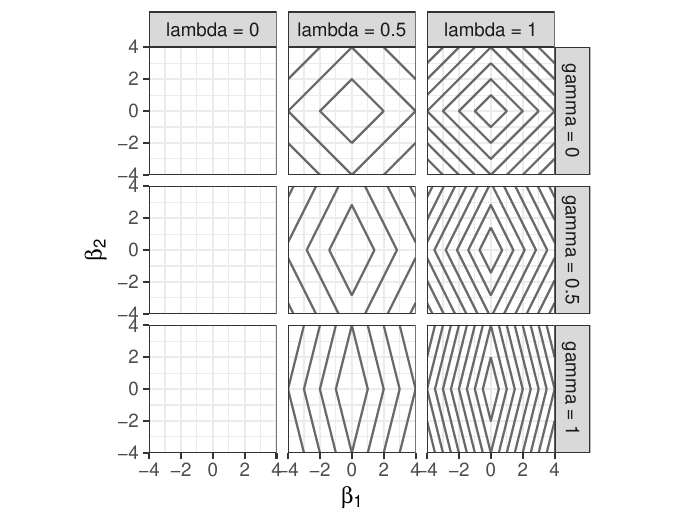}
  \includegraphics[width=6cm]{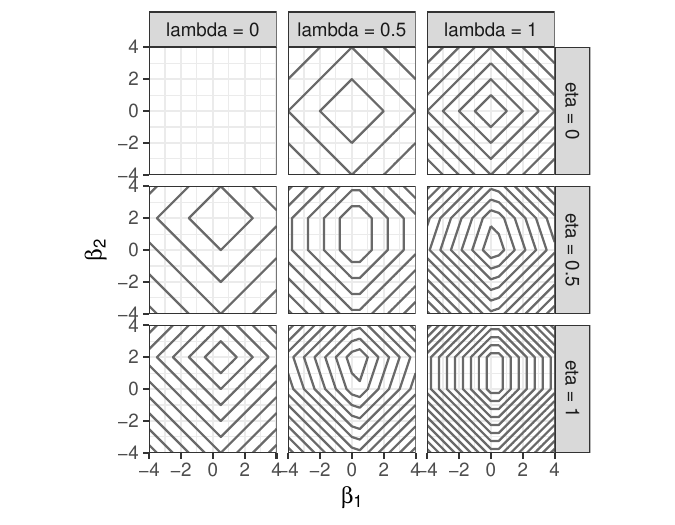}
  \includegraphics[width=6cm]{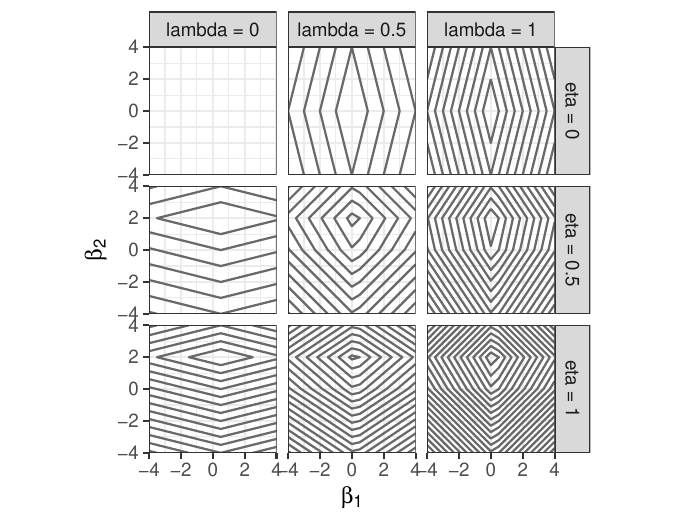}
  \caption{Regularization contours for the Adaptive Lasso~(top left), the Transfer Lasso~(top right), and the Adaptive Transfer Lasso~(bottom) with initial estimators $\betatilde = [0.5, 2]^\top$. Hyperparameters are $\lambda_n = 0, 0.5, 1$ (from left to right); $\eta_n = 0, 0.5, 1$ (from top to bottom); and $\gamma_1 = \gamma_2 = 1$ for the Adaptive Transfer Lasso.}
  \label{fig:contours}
\end{figure*}

\subsection{Prior Distribution}

In Bayesian perspectives, the Lasso regularization~\eqref{eq:lasso} can be seen as a negative log-likelihood of Laplace prior,
\begin{align}
    \lambda |\beta_j| = - \log P(\beta_j; \lambda) + const., ~
    P (z; \lambda) 
    := \frac{\lambda}{2} \exp \left( - \lambda |z| \right).
\end{align}
A similar view is possible for the Adaptive Lasso, the Transfer Lasso, and the Adaptive Transfer Lasso.
Most generally, the prior distribution of the Adaptive Transfer Lasso is given by
\begin{align}
    &\lambda v_j |\beta_j| + \eta w_j |\beta_j - \betatilde_j|
    = - \log P( \beta_j; \lambda, \eta, v_j, w_j, \betatilde_j ) + const.,\\
    &P( \beta_j; \lambda, \eta, v_j, w_j, \betatilde_j )
    := \frac{1}{Z} \exp \left( - \lambda v_j |\beta_j| - \eta w_j |\beta_j - \betatilde_j| \right),\\
    &Z := \frac{2 \lambda v_j}{\lambda^2 v_j^2 - \eta^2 w_j^2} \exp \left( - \eta w_j |\betatilde_j| \right) - \frac{2 \eta w_j}{\lambda^2 v_j^2 - \eta^2 w_j^2} \exp \left( - \lambda v_j |\betatilde_j| \right).
\end{align}

\begin{figure*}[t]
  \centering
  \includegraphics[width=6cm]{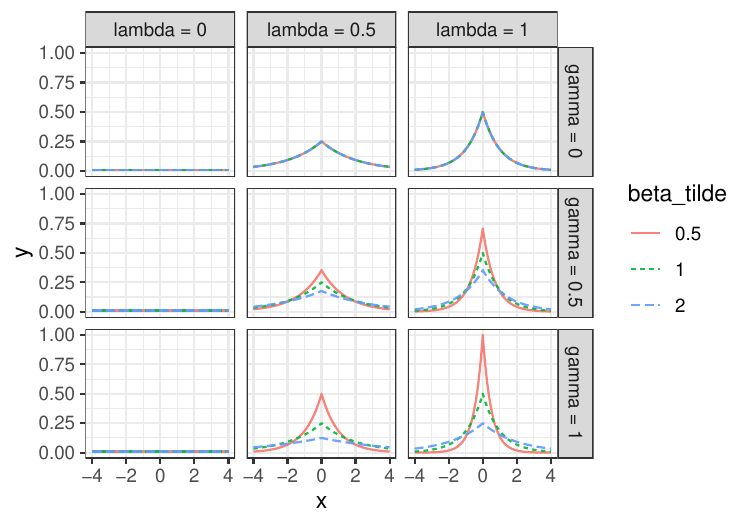}
  \includegraphics[width=6cm]{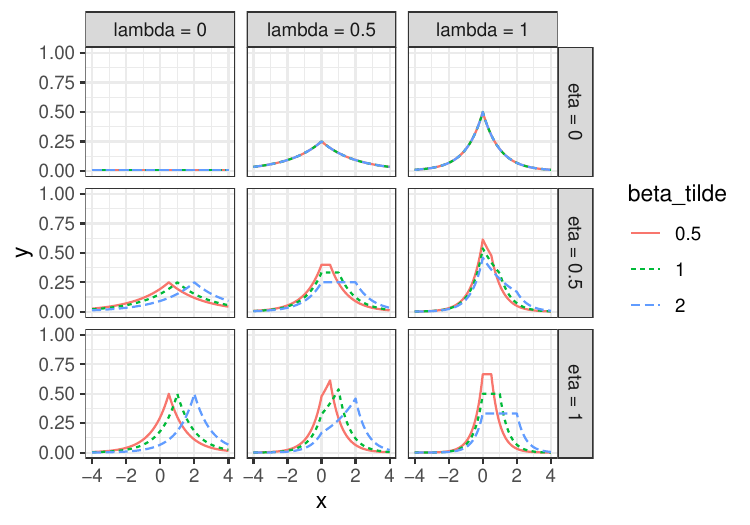}
  \includegraphics[width=6cm]{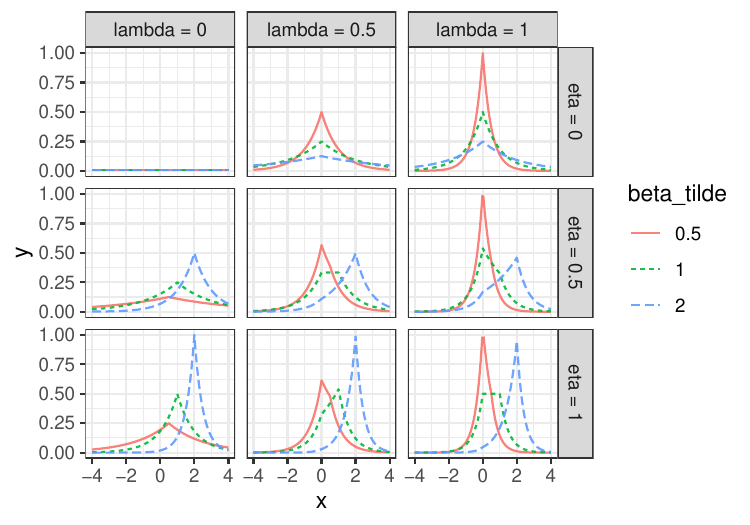}
  \caption{Prior distributions for the Adaptive Lasso (top left), the Transfer Lasso (top right), and the Adaptive Transfer Lasso (bottom) with various initial estimators. Hyperparameters are $\gamma_1 = \gamma_2 = 1$ for the Adaptive Transfer Lasso.}
  \label{fig:priors}
\end{figure*}

The prior distributions for the Adaptive Lasso, the Transfer Lasso, and the Adaptive Transfer Lasso are shown in Figures~\ref{fig:priors}.
The prior distributions for the Adaptive Lasso are all sharp at zero, and the distributions become steeper as the initial estimator decreases.
This means that the Adaptive Lasso controls the variance of the prior distribution based on how close to zero the initial estimator is.
The prior distribution for the Transfer Lasso is sharp at two points: zero and the initial estimator.
When the initial estimator is small, it is nearly the same as that for the Lasso, but when the initial estimator is large, the sharpness changes depending on the magnitudes of $\lambda$ and $\eta$.
The prior distribution for the Adaptive Transfer Lasso is somewhat different from that of the Transfer Lasso.
The prior distribution tends to peak at zero when the initial estimator is close to zero, whereas the prior distribution tends to peak at that initial estimator when the initial estimator is far from zero.
This suggests that the Adaptive Transfer Lasso can make full use of the information from the initial estimator and achieves accurate active and varying variable selection.

\subsection{Future Work}

In our asymptotic analysis, we considered the case where $p$ is fixed and $n$ is infinitely divergent.
The oracle property of Adaptive Lasso~\cite{zou2006adaptive} can be extended to the case for $p \gg n$ by high-dimensional asymptotic theory~\cite{huang2008adaptive}, under different kinds of assumptions.
As future research, it would be interesting to see whether this can be extended to the Transfer Lasso and the Adaptive Transfer Lasso.

In addition, we assumed that the initial estimator is consistent in our asymptotic analysis.
When the initial estimator is incorrectly specified, performance deteriorates significantly for the Adaptive Lasso, but not so much for the Transfer Lasso.
It would be interesting to theoretically verify this property.

\section{Conclusion}

The Adaptive Lasso and the Transfer Lasso are similar but have their advantages and disadvantages from the viewpoint of an asymptotic perspective.
We proposed the Adaptive Transfer Lasso, which has advantages over the Adaptive Lasso and the Transfer Lasso.
We confirmed it in numerical simulations.

\section*{Acknowledgment}
The research is the collaborative work of Toshiba Corporation and The Institute of Statistical Mathematics, based on funding from Toshiba Corporation.

\bibliographystyle{plain}
\bibliography{main}

\begin{thebibliography}{10}

\bibitem{bastani2021predicting}
Hamsa Bastani.
\newblock Predicting with proxies: Transfer learning in high dimension.
\newblock {\em Management Science}, 67(5):2964--2984, 2021.

\bibitem{belkin2020two}
Mikhail Belkin, Daniel Hsu, and Ji~Xu.
\newblock Two models of double descent for weak features.
\newblock {\em SIAM Journal on Mathematics of Data Science}, 2(4):1167--1180,
  2020.

\bibitem{chzhen2019lasso}
Evgenii Chzhen, Mohamed Hebiri, and Joseph Salmon.
\newblock On lasso refitting strategies.
\newblock {\em Bernoulli}, 25(4A):3175--3200, 2019.

\bibitem{fan2001variable}
Jianqing Fan and Runze Li.
\newblock Variable selection via nonconcave penalized likelihood and its oracle
  properties.
\newblock {\em Journal of the American statistical Association},
  96(456):1348--1360, 2001.

\bibitem{fu2000asymptotics}
Wenjiang Fu and Keith Knight.
\newblock Asymptotics for lasso-type estimators.
\newblock {\em The Annals of statistics}, 28(5):1356--1378, 2000.

\bibitem{geyer1994asymptotics}
Charles~J Geyer.
\newblock On the asymptotics of constrained m-estimation.
\newblock {\em The Annals of statistics}, pages 1993--2010, 1994.

\bibitem{geyer1996asymptotics}
Charles~J Geyer.
\newblock On the asymptotics of convex stochastic optimization.
\newblock {\em Unpublished manuscript}, 37, 1996.

\bibitem{hastie2022surprises}
Trevor Hastie, Andrea Montanari, Saharon Rosset, and Ryan~J Tibshirani.
\newblock Surprises in high-dimensional ridgeless least squares interpolation.
\newblock {\em The Annals of Statistics}, 50(2):949--986, 2022.

\bibitem{hastie2020best}
Trevor Hastie, Robert Tibshirani, and Ryan Tibshirani.
\newblock Best subset, forward stepwise or lasso? analysis and recommendations
  based on extensive comparisons.
\newblock {\em Statistical Science}, 35(4):579--592, 2020.

\bibitem{huang2008adaptive}
Jian Huang, Shuangge Ma, and Cun-Hui Zhang.
\newblock Adaptive lasso for sparse high-dimensional regression models.
\newblock {\em Statistica Sinica}, pages 1603--1618, 2008.

\bibitem{li2022transfer}
Sai Li, T~Tony Cai, Hongzhe Li, et~al.
\newblock Transfer learning for high-dimensional linear regression: Prediction,
  estimation and minimax optimality.
\newblock {\em Journal of the Royal Statistical Society Series B},
  84(1):149--173, 2022.

\bibitem{li2021minimum}
Yue Li and Yuting Wei.
\newblock Minimum $\ell_1$-norm interpolators: Precise asymptotics and multiple
  descent.
\newblock {\em arXiv preprint arXiv:2110.09502}, 2021.

\bibitem{meinshausen2007relaxed}
Nicolai Meinshausen.
\newblock Relaxed lasso.
\newblock {\em Computational Statistics \& Data Analysis}, 52(1):374--393,
  2007.

\bibitem{mitra2019understanding}
Partha~P Mitra.
\newblock Understanding overfitting peaks in generalization error: Analytical
  risk curves for $\ell_2 $ and $\ell_1$ penalized interpolation.
\newblock {\em arXiv preprint arXiv:1906.03667}, 2019.

\bibitem{pollard1991asymptotics}
David Pollard.
\newblock Asymptotics for least absolute deviation regression estimators.
\newblock {\em Econometric Theory}, 7(2):186--199, 1991.

\bibitem{takada2020transfer}
Masaaki Takada and Hironori Fujisawa.
\newblock Transfer learning via $\ell_1$ regularization.
\newblock {\em Advances in Neural Information Processing Systems},
  33:14266--14277, 2020.

\bibitem{tian2022transfer}
Ye~Tian and Yang Feng.
\newblock Transfer learning under high-dimensional generalized linear models.
\newblock {\em Journal of the American Statistical Association}, pages 1--14,
  2022.

\bibitem{tibshirani1996regression}
Robert Tibshirani.
\newblock Regression shrinkage and selection via the lasso.
\newblock {\em Journal of the Royal Statistical Society: Series B
  (Methodological)}, 58(1):267--288, 1996.

\bibitem{zhang2010nearly}
C-H ZHANG.
\newblock Nearly unbiased variable selection under minimax concave penalty.
\newblock {\em The Annals of Statistics}, 38:894--942, 2010.

\bibitem{zhao2006model}
Peng Zhao and Bin Yu.
\newblock On model selection consistency of lasso.
\newblock {\em The Journal of Machine Learning Research}, 7:2541--2563, 2006.

\bibitem{zou2006adaptive}
Hui Zou.
\newblock The adaptive lasso and its oracle properties.
\newblock {\em Journal of the American statistical association},
  101(476):1418--1429, 2006.

\end{thebibliography}

\newpage

\begin{appendix}

\section*{Appendices}

\section{Additional Asymptotic Properties}
\label{appendix:fixed-asymptotics}

In this section, we describe additional theoretical results that are not described in the main text.
In Sections \ref{sec:alasso-deterministic} and \ref{sec:trlasso-deterministic}, we describe the asymptotic properties of Adaptive Lasso and Transfer Lasso, respectively, when the source parameters are deterministic (fixed).
In Section \ref{sec:additional-tlasso}, we discuss additional results for Transfer Lasso when the hyperparameters are at boundary values.
% We discuss additional results for Transfer Lasso when the hyperparameters are at boundary values in Section \ref{sec:additional-tlasso}.

\subsection{Adaptive Lasso with deterministic Source Parameter}
\label{sec:alasso-deterministic}

To avoid zero division, we define Adaptive Lasso with a deterministic source parameter $\betatilde$ as
\begin{align}
\label{alasso-deteministic}
    \hat{\beta}_n^\mathcal{A} = \argmin_{\beta: \beta_j = 0 \text{~for~} \betatilde_j = 0 } \left\{
    Z_n^\mathcal{A}(\beta; \betatilde, \lambda_n, \gamma) := \frac{1}{n} \| y - X\beta \|_2^2 + \frac{\lambda_n}{n} \sum_{j: \betatilde_j \neq 0} w_j |\beta_j|
    \right\},
\end{align}
and $w_j := 1 / |\betatilde_j|^\gamma$ for $\betatilde_j \neq 0$.
By the definition \eqref{alasso-deteministic} and Lemma~\ref{lemma:adapt-lasso-oracle}, we can easily obtain Lemma~\ref{lemma:adapt-lasso-consistency-deterministic} and Corollary~\ref{lemma:adapt-lasso-oracle-deterministic}.

\begin{lemma}[$\sqrt{n}$-consistency for the Adaptive Lasso with Deterministic Source Parameter]
\label{lemma:adapt-lasso-consistency-deterministic}
Suppose that $\tilde{S} = S$.
If $\lambda_n/\sqrt{n} \rightarrow \lambda_0 \geq 0$, then
\begin{align}
    &\sqrt{n} (\betahat_{Sn}^\mathcal{A} - \betastar_S) \\
    \overset{d}{\rightarrow}&
    \argmin_u \left\{
    u_S^\top C_{SS} u_S -2 u_S^\top W_S
    + \lambda_0 \sum_{j: \betatilde_j \neq 0} w_j \left( u_j \sgn(\betastar_j) I(\betastar_j \neq 0) + |u_j| I(\betastar_j=0) \right)
    \right\},\nonumber
\end{align}
where $W \sim \mathcal{N}(0, \sigma^2 C)$, and $\betahat_{S^cn}^\mathcal{A} = 0$.
\end{lemma}

\begin{corollary}[Oracle Property for Adaptive Lasso with Deterministic Source Parameter]
\label{lemma:adapt-lasso-oracle-deterministic}
Suppose that $\tilde{S} = S$.
If $\lambda_n = o(\sqrt{n})$, then the Adaptive Lasso estimator satisfies the oracle property.
\end{corollary}

Based on Lemma~\ref{lemma:adapt-lasso-consistency-deterministic} and Corollary~\ref{lemma:adapt-lasso-oracle-deterministic}, the Adaptive Lasso can satisfy the oracle property even for deterministic source parameters.
However, the source parameter must satisfy exact support recovery, which is very restrictive.

\subsection{Transfer Lasso with deterministic Source Parameter}
\label{sec:trlasso-deterministic}

Now we consider the asymptotic properties of Transfer Lasso with deterministic source parameter.
Let $T := \{ j: \betatilde_j \neq \betastar_j \}$.
We can identify the range of its estimator~(Theorem~\ref{trlasso-variable-selection}) for a general condition.
The inequality in the probability is not necessarily tight.

\begin{theorem}[Estimation Range for Transfer Lasso]
\label{trlasso-variable-selection}
If $(\lambda_n + \eta_n) / \sqrt{n} \rightarrow \infty$ and $(\lambda_n - \eta_n) / \sqrt{n} \rightarrow 0$, then the transfer lasso estimates satisfy
\begin{align}
    \lim_{n\rightarrow \infty} P\left( \min \left\{ 0, \betatilde_j \right\} \leq \betahat_j \leq \max \left\{ 0, \betatilde_j \right\} \right) = 1.
\end{align}
\end{theorem}

\begin{proof}
The proof is given in \ref{proof:trlasso-variable-selection}.
\end{proof}

To obtain an asymptotic distribution and consistent variable selection, we need to impose the condition that $\sgn(\betatilde_j - \betastar_j) = \sgn(\betastar_j)$ for $\forall j \in S \cap T$.
This requires ``over-estimation'' for the initial estimator.
We obtain Theorem~\ref{trlasso-distribution-sign-consistency} under this condition.

\begin{theorem}[Asymptotic distribution and active/varying variable consistency with deterministic source parameters]
\label{trlasso-distribution-sign-consistency}
Suppose that $\sgn(\betatilde_j - \betastar_j) = \sgn(\betastar_j)$ for $\forall j \in S \cap T$.
If $(\lambda_n + \eta_n) / \sqrt{n} \rightarrow \infty$ and $(\lambda_n - \eta_n) / \sqrt{n} \rightarrow 0$,
then the Transfer Lasso estimator satisfies
\begin{align}
    \sqrt{n}(\betahat - \betastar) \overset{d}{\rightarrow}
    \argmin_{u\in \mathcal{U}} & \left\{ -2u^\top W + u^\top C u \right\},
    \label{eq-trlasso-distribution}
\end{align}
\begin{align}
    W \sim \mathcal{N} (0, \sigma^2 C), ~
    \mathcal{U} := \left\{ u\in \mathbb{R}^p \left|
    \begin{array}{l}
        u_j=0 ~\text{for}~ \forall j \in S^c \cap T^c\\
        u_j \betatilde_j \geq 0 ~\text{for} ~\forall j \in S^c \cap T\\
        u_j \betastar_j \leq 0 ~\text{for}~ \forall j \in S \cap T^c
    \end{array}
    \right. \right\},
    \label{eq-w}
\end{align}
and
\begin{align}
    \begin{cases}
        \lim_{n\rightarrow \infty} P \left(0 < \betahat_j < \betatilde_j ~\text{or}~ \betatilde_j < \betahat_j < 0 \right) = 1 ~ & \text{for} ~ \forall j \in (S \cap T)\\
        \lim_{n\rightarrow \infty} P \left(0 < \betahat_j \leq \betatilde_j ~\text{or}~ \betatilde_j \leq \betahat_j < 0 \right) = 1 ~ & \text{for} ~ \forall j \in (S \cap T^c)\\
        \lim_{n\rightarrow \infty} P \left(0 \leq \betahat_j < \betatilde_j ~\text{or}~ \betatilde_j < \betahat_j \leq 0 \right) = 1 ~ & \text{for} ~ \forall j \in (S^c \cap T)\\
        \lim_{n\rightarrow \infty} P \left( \betahat_j = 0 \right) = 1 ~ & \text{for} ~ \forall j \in (S^c \cap T^c).
    \end{cases}
\end{align}
\end{theorem}

\begin{proof}
The proof is given in \ref{proof:trlasso-distribution-sign-consistency}.
\end{proof}

Theorem~\ref{trlasso-distribution-sign-consistency} shows that the asymptotic distribution in \eqref{eq-trlasso-distribution} and \eqref{eq-w} is a truncated Gaussian-mixture distribution.
Each distribution in mixtures must satisfy 
1)~Gaussian distribution for $j \in S \cap T$;
2)~truncated Gaussian distribution truncated at zero, or delta distribution at zero for $j \in S^c \cap T$ and $j \in S \cap T^c$; and 
3)~delta distribution at zero for $j \in S^c \cap T^c$.

In addition, Theorem~\ref{trlasso-distribution-sign-consistency} indicates that the true active variables ($j \in S$) and the true varying variables ($j \in T$) can be recovered asymptotically if the source parameters of both active and varying variables have the same sign as the true variables and have larger absolute values than the true variables.

Asymptotic normality (instead of a truncated normal mixture distribution) can be obtained under a more restrictive condition~(Theorem~\ref{trlasso-oracle}).

\begin{theorem}[Oracle property (asymptotic normality and active/varying variable selection consistency)]
\label{trlasso-oracle}
Suppose that $S=T$, and $\sgn(\betatilde_j - \betastar_j) = \sgn(\betastar_j)$ for $\forall j \in S(=T)$.
If $(\lambda_n + \eta_n) / \sqrt{n} \rightarrow \infty$ and $(\lambda_n - \eta_n) / \sqrt{n} \rightarrow 0$, then the Transfer Lasso estimates satisfy the oracle property, that is,
\begin{align}
    \lim_{n\rightarrow \infty} P(\hat{S}_n = S) 
    = \lim_{n\rightarrow \infty} P(\hat{T}_n = T)
    = 1,
\end{align}
\begin{align}
    \sqrt{n} (\betahat_S - \betastar_S) \overset{d}{\rightarrow} \mathcal{N}(0, \sigma^2 C_{SS}^{-1}).
\end{align}
\end{theorem}

\begin{proof}
The proof is given in \ref{proof:trlasso-oracle}.    
\end{proof}

Theorem~\ref{trlasso-oracle} is the oracle property for the Transfer Lasso.
The Adaptive Lasso requires $\sqrt{n}$-consistency for the initial estimator.
In contrast, the Transfer Lasso requires variable consistency, sign consistency, and ``over-estimation'' for the source parameters.
This yields $\sqrt{n}$-consistency as well as variable selection consistency for both active and varying variables.

\subsection{Transfer Lasso with Initial Estimator in Boundary Region}
\label{sec:additional-tlasso}

We have an additional result on the asymptotic property for Transfer Lasso with an initial estimator when the hyperparameters are at boundary values.

\begin{theorem}
\label{theorem:tlasso-asymptotic-distribution-boundary}
Suppose that $\betatilde$ is a $\sqrt{m}$-consistent estimator and define $z := \sqrt{m} (\betatilde_m - \betastar)$.
Suppose that $n / m \to 0$. 
If $\lambda_n / \sqrt{n} \to \infty$, $\eta_n / \lambda_n \to 1$, and $(\lambda_n - \eta_n) / \sqrt{n} \to \delta_0$, then
\begin{align}
    \sqrt{n} \left( \betahat_n^\mathcal{T} - \betastar \right)
    \overset{d}{\to}
    \argmin_{u \in \mathcal{U}} \left\{ u_S^\top C_{SS} u_S - 2 u_S^\top W_S - \delta_0 \sum_{j \in S} |u_j| \right\},
\end{align}
\begin{align}
    \mathcal{U} := \left\{ u\in \mathbb{R}^p \left|
        \begin{array}{l}
            \betastar_j u_j \leq 0 ~\text{for}~ \forall j \in S ~\text{and} ~
            u_j = 0 ~\text{for}~ \forall j \in S^c
        \end{array}
    \right. \right\}.
\end{align}
In addition, $\betahat_n^\mathcal{T}$ results in inconsistent active variable selection.
\end{theorem}

\begin{proof}
    The proof is given in \ref{proof:tlasso-asymptotic-distribution-boundary}.
\end{proof}

Theorem~\ref{theorem:tlasso-asymptotic-distribution-boundary} implies $\sqrt{n}$-consistency  but inconsistent variable selection for the Transfer Lasso.

\section{Proofs}

\subsection{Proofs of Lasso and Adaptive Lasso}
\label{proof-lasso}

In these proofs, the superscripts of each method may be omitted when it is obvious.
For example, $\betahat_n^\mathcal{L}$, $\betahat_n^\mathcal{A}$, $\betahat_n^\mathcal{T}$, and $\betahat_n^\#$ may be simply written as $\betahat_n$.

\subsubsection{Proof of Lemma~\ref{lemma:lasso-consistency}}
\label{proof:lasso-consistency}
We have
\begin{align}
    Z_n^\mathcal{L}(\beta; \lambda_n) 
    :&= \frac{1}{n} \| y - X \beta \|_2^2 + \frac{\lambda_n}{n} \sum_j |\beta_j|\\
    &= \frac{1}{n} \| X\betastar + \varepsilon - X \beta \|_2^2 + \frac{\lambda_n}{n} \sum_j |\beta_j|\\
    &= (\beta - \betastar)^\top \left( \frac{1}{n} X^\top X \right) (\beta - \betastar) 
    - \frac{2}{n} (\beta - \betastar)^\top X^\top \varepsilon 
    + \frac{1}{n} \| \varepsilon \|_2^2
    + \frac{\lambda_n}{n} \sum_j |\beta_j|.
\end{align}
Let
\begin{align}
    V(\beta, \lambda_0) := (\beta - \betastar)^\top C (\beta - \betastar) + \lambda_0 \sum_j |\beta_j|.
\end{align}
Then, we have
\begin{align}
    &Z_n^\mathcal{L}(\beta; \lambda_n) - V(\beta; \lambda_0) - \sigma^2\\
    =& (\beta - \betastar)^\top \left( \frac{1}{n} X^\top X - C \right) (\beta - \betastar)
    - \frac{2}{n} (\beta - \betastar)^\top X^\top \varepsilon\\
    &+ \left( \frac{1}{n} \| \varepsilon \|_2^2 - \sigma^2 \right) + \left(\frac{\lambda_n}{n} - \lambda_0\right) \sum_j |\beta_j|\\
    \rightarrow& 0 ~ (n \rightarrow \infty; ~ \text{pointwise convergence)}
\end{align}
Since $Z_n^\mathcal{L}$ is convex, based on \cite{pollard1991asymptotics}, we have
\begin{align}
    \sup_{\beta \in K} |Z_n^\mathcal{L}(\beta; \lambda_n) - V(\beta; \lambda_0) - \sigma^2| \overset{p}{\rightarrow} 0
    ~ \text{for any compact set} ~ K.
\end{align}
Let $\beta_n^{(0)} := \argmin_\beta Z_n^\mathcal{L}(\beta; 0)$.
Because $\beta_n^{(0)}=O_P(1)$ and $\|\betahat^\mathcal{L}_n\|_1 \leq \|\betahat_n^{(0)}\|_1$, we have $\betahat_n^\mathcal{L} = O_P(1)$ and thus
\begin{align}
    \betahat_n^\mathcal{L} = \argmin_\beta Z_n^\mathcal{L}(\beta; \lambda_n) \overset{p}{\rightarrow} \argmin_\beta V(\beta; \lambda_0).
\end{align}

\subsubsection{Proof of Lemma~\ref{lemma:lasso-root-n-consistency}}
\label{proof:lasso-root-n-consistency}

Let $u := \sqrt{n} (\beta - \betastar)$.
We have
\begin{align}
    Z_n^\mathcal{L}(\beta; \lambda_n) 
    :=& \frac{1}{n} \| y - X \beta \|_2^2 + \frac{\lambda_n}{n} \sum_j |\beta_j|\\
    =& \frac{1}{n} \left\| \varepsilon - \frac{1}{\sqrt{n}} X u \right\|_2^2 + \frac{\lambda_n}{n} \sum_j \left| \betastar_j + \frac{u_j}{\sqrt{n}} \right|,
\end{align}
and
\begin{align}
    &Z_n^\mathcal{L}(\beta; \lambda_n) - Z_n^\mathcal{L}(\betastar; \lambda_n)\\
    = & \frac{1}{n} \left\| \varepsilon - \frac{1}{\sqrt{n}} X u \right\|_2^2 - \frac{1}{n} \| \varepsilon \|_2^2 + \frac{\lambda_n}{n} \sum_j \left( \left| \betastar_j + \frac{u_j}{\sqrt{n}} \right| - |\betastar_j| \right)\\
    = & \frac{1}{n} \underbrace{\left\{ u^\top \left( \frac{1}{n} X^\top X \right) u -2 u^\top \left(\frac{1}{\sqrt{n}}  X^\top \varepsilon \right) 
    + \frac{\lambda_n}{\sqrt{n}} \sum_j \left( \left| u_j + \sqrt{n} \betastar_j \right| - |\sqrt{n} \betastar_j| \right) \right\}}_{=: V_n(u)}.
    \label{lasso-V}
\end{align}
Assuming $\lambda_n / \sqrt{n} \rightarrow \lambda_0$, we obtain
\begin{align}
    V_n(u) \overset{d}{\rightarrow} & u^\top C u -2u^\top W
    + \lambda_0 \sum_j \left( u_j \sgn(\betastar_j) I(\betastar_j \neq 0) + |u_j| I(\betastar_j=0) \right) =: V(u).
\end{align}

Since $V_n(u)$ is convex and $V(u)$ has a unique minimum, based on \cite{geyer1996asymptotics}, we have
\begin{align}
    \argmin_u V_n(u) \overset{d}{\rightarrow} \argmin_u V(u).
\end{align}
On the other hand, we have
\begin{align}
    \betahat_n^\mathcal{L} =& \argmin_\beta Z_n^\mathcal{L}(\beta; \lambda_n) 
    = \argmin_\beta (Z_n^\mathcal{L}(\beta; \lambda_n) - Z_n^\mathcal{L}(\betastar; \lambda_n)) \\
    =& \argmin_\beta V_n(u) 
    = \betastar + \frac{1}{\sqrt{n}} \argmin_u V_n (u).
\end{align}
Therefore, we have $\argmin_u V_n(u) = \sqrt{n} (\betahat_n^\mathcal{L} - \betastar)$ and
\begin{align}
    \sqrt{n}(\betahat_n^\mathcal{L} - \betastar) \overset{d}{\rightarrow} \argmin_u V(u).
\end{align}

\subsubsection{Proof of Lemma~\ref{lemma:lasso-variable-selection}}
\label{proof:lasso-variable-selection}

Let $u^* := \argmin_u V(u)$ where $V(u)$ is defined in \eqref{lasso-V}.
Note that $\hat{S}_n = S$ implies $\betahat_j = 0 ~ \forall j \in S^c$ and $\sqrt{n} \betahat_{S^c} \overset{d}{\rightarrow} u^*_{S^c}$.
By the weak convergence result, we have
\begin{align}
    P(\hat{S}_n = S) 
    \leq P(\sqrt{n} \betahat_j = 0 ~ \forall j \in S^c),
\end{align}
and
\begin{align}
    \limsup_n P(\hat{S}_n = S) 
    \leq \limsup_n P(\sqrt{n} \betahat_j = 0 ~ \forall j \in S^c) \leq P(u_j^* = 0 ~ \forall j \in S^c) =: c.
\end{align}
If $\lambda_0 = 0$, then we have
\begin{align}
    u^* &= C^{-1}W \sim \mathcal{N}(0, \sigma^2 C^{-1}),
\end{align}
resulting in $c = 0$.
If $\lambda_0 > 0$, then the KKT condition yields
\begin{align}
    \begin{cases}
        -2 W_S + 2(Cu^*)_S + \lambda_0 \sgn(\betastar_S) = 0\\
        \left|-2 W_{S^c} + 2(Cu^*)_{S^c}\right| \leq \lambda_0.
    \end{cases}
        \end{align}
When $u^*_{S^c} = 0$, the above conditions indicate
\begin{align}
    \left|-2 W_{S^c} + C_{S^c S} C_{SS}^{-1} (2W_S - \lambda_0 \sgn (\betastar_S)\right| \leq \lambda_0,
\end{align}
thus we have
\begin{align}
    c \leq 
    P\left( \left| -2 W_{S^c} + C_{S^c S} C_{SS}^{-1} (2W_S - \lambda_0 \sgn (\betastar_S) \right| \leq \lambda_0 \right)
    <1.
\end{align}

\subsubsection{Proof of Lemma~\ref{lemma:lasso-slower-consistency}}
\label{proof:lasso-slower-consistency}

Let $u := (n / \lambda_n) (\beta - \betastar) $.
We have
\begin{align}
    Z_n^\mathcal{L}(\beta)
    :=& \frac{1}{n} \| y - X \beta \|_2^2 + \frac{\lambda_n}{n} \sum_j |\beta_j|\\
    =& \frac{1}{n} \left\| \varepsilon - \frac{\lambda_n}{n} X u \right\|_2^2 + \frac{\lambda_n}{n} \sum_j \left| \betastar_j + \frac{\lambda_n}{n} u_j \right|,
\end{align}
and
\begin{align}
    &Z_n^\mathcal{L}(\beta) - Z_n^\mathcal{L}(\betastar)\\
    =& \frac{1}{n} \left\| \varepsilon - \frac{\lambda_n}{n} X u \right\|_2^2 - \frac{1}{n} \left\| \varepsilon \right\|_2^2 + \frac{\lambda_n}{n} \sum_j \left( \left| \betastar_j + \frac{\lambda_n}{n} u_j \right| - \left| \betastar_j \right| \right)\\
    =& \frac{\lambda_n^2}{n^2} \underbrace{\left\{ u^\top \left( \frac{1}{n} X^\top X \right) u -2 \frac{\sqrt{n}}{\lambda_n} u^\top \left(\frac{1}{\sqrt{n}}  X^\top \varepsilon \right)
    + \sum_j \left( \left| u_j + \frac{n}{\lambda_n} \betastar_j \right| - \left| \frac{n}{\lambda_n} \betastar_j \right| \right)\right\}}_{=: V_n(u)}.
\end{align}
 Assuming $\lambda_n / n \rightarrow 0$ and $\lambda_n / \sqrt{n} \rightarrow \infty$, we obtain
\begin{align}
    V_n(u) \overset{d}{\rightarrow} & u^\top C u + \sum_{j=1}^p \left( u_j \sgn(\betastar_j) I(\betastar \neq 0) + |u_j| I(\betastar_j = 0) \right) =: V(u).
\end{align}
Since $V_n(u)$ is convex and $V(u)$ has a unique minimum, based on \cite{geyer1996asymptotics}, we have
\begin{align}
    \argmin_u V_n(u) \overset{d}{\rightarrow} \argmin_u V(u).
\end{align}
On the other hand, we have
\begin{align}
    \betahat_n^\mathcal{L} =& \argmin_\beta Z_n^\mathcal{L}(\beta; \lambda_n) 
    = \argmin_\beta (Z_n^\mathcal{L}(\beta; \lambda_n) - Z_n^\mathcal{L}(\betastar; \lambda_n)) \\
    =& \argmin_\beta V_n(u) 
    = \betastar + \frac{\lambda_n}{n} \argmin_u V_n (u).
\end{align}
Therefore, we have $\argmin_u V_n(u) = (n / \lambda_n) (\betahat_n^\mathcal{L} - \betastar)$ and
\begin{align}
    \frac{n}{\lambda_n} (\betahat_n^\mathcal{L} - \betastar) \overset{d}{\rightarrow} \argmin_u V(u).
\end{align}

\subsubsection{Proof of Lemma~\ref{lemma:adapt-lasso-oracle}}
\label{proof:adapt-lasso-oracle}

Asymptotic Normality Part:
Let $u := \sqrt{n} (\beta - \betastar)$.
We have
\begin{align}
    Z_n^\mathcal{A}(\beta) 
    :=& \frac{1}{n} \| y - X \beta \|_2^2 + \frac{\lambda_n}{n} \sum_j w_j |\beta_j|\\
    =& \frac{1}{n}\left\| \varepsilon - \frac{1}{\sqrt{n}} X u \right\|_2^2 + \frac{\lambda_n}{n} \sum_j w_j \left| \betastar_j + \frac{u_j}{\sqrt{n}} \right|
    \end{align}
and
\begin{align}
    &Z_n^\mathcal{A}(\beta) - Z_n^\mathcal{A}(\betastar)\\
    =& \frac{1}{n} \left\| \varepsilon - \frac{1}{\sqrt{n}} X u \right\|_2^2 - \frac{1}{n} \| \varepsilon \|_2^2 + \frac{\lambda_n}{n} \sum_j w_j \left( \left| \betastar_j + \frac{u_j}{\sqrt{n}} \right| - |\betastar_j| \right)\\
    =& \frac{1}{n} \underbrace{\left\{ u^\top \left( \frac{1}{n} X^\top X \right) u -2 u^\top \left(\frac{1}{\sqrt{n}} X^\top \varepsilon \right)
    + \frac{\lambda_n}{\sqrt{n}} \sum_j w_j \left( \left| u_j + \sqrt{n} \betastar_j \right| - \left| \sqrt{n} \betastar_j \right| \right) \right\}}_{=: V_n(u)}.\label{eq:alasso-Z-difference}
\end{align}
We know that 
\begin{align}
    -2 u^\top \left(\frac{1}{\sqrt{n}}  X^\top \varepsilon \right)
    \overset{d}{\rightarrow} -2u^\top W, ~ 
    W \sim \mathcal{N}(0, \sigma^2 C) \label{uw}
\end{align}
\begin{align}
    u^\top \left( \frac{1}{n} X^\top X \right) u
    \rightarrow u^\top C u. \label{ucu}
\end{align}
Now we consider the last term of \eqref{eq:alasso-Z-difference}.
If $\betastar_j \neq 0$, then $w_j \rightarrow 1/|\betastar_j|^\gamma$ and $ | u_j + \sqrt{n} \betastar_j | - | \sqrt{n} \betastar_j | \rightarrow u_j \sgn(\betastar_j)$, thus by Slutsky's theorem,
\begin{align}
    \frac{\lambda_n}{\sqrt{n}} w_j \left( \left| u_j + \sqrt{n} \betastar_j \right| - \left| \sqrt{n} \betastar_j \right| \right)
        \overset{p}{\rightarrow} 0,
\end{align}
as $\lambda_n / \sqrt{n} \rightarrow 0$.
If $\betastar_j = 0$, then $\lambda_n w_j / \sqrt{n} = \lambda_n n^{(\gamma - 1)/2} |\sqrt{n}\betatilde_j|^{-\gamma}$ and $| u_j + \sqrt{n} \betastar_j | - | \sqrt{n} \betastar_j | = |u_j|$, thus by Slutsky's theorem,
\begin{align}
    \label{eq-asymptotic-zero}
    \frac{\lambda_n}{\sqrt{n}} w_j \left( \left| u_j + \sqrt{n} \betastar_j \right| - \left| \sqrt{n} \betastar_j \right| \right)
        \overset{d}{\rightarrow}
    \begin{cases}
        0 & \text{if} ~ u_j=0\\
        \infty & \text{otherwise},
    \end{cases}
\end{align}
as $\lambda_n n^{(\gamma - 1)/2} \rightarrow \infty$.
Therefore, we have for every $u$,
\begin{align}
    V_n(u) \overset{d}{\rightarrow} V(u) 
    := 
    \begin{cases}
        -2u_S^\top W_S + u_S^\top C_{SS} u_S & \text{if} ~ u_j = 0 ~ \forall j \in S^c\\
        \infty & \text{otherwise}.
    \end{cases}
\end{align}
Since $V_n(u)$ is convex and $V(u)$ has a unique minimum of $(C_{SS}^{-1} W_S, 0)^\top$, based on \cite{geyer1994asymptotics}, we have
\begin{align}
    \argmin_u V_n(u) \overset{d}{\rightarrow} \argmin_u V(u)
    = (C_{SS}^{-1} W_S, 0)^\top.
\end{align}
On the other hand,
\begin{align}
    \betahat_n^\mathcal{L} =& \argmin_\beta Z_n(\beta; \lambda_n) 
    = \argmin_\beta (Z_n(\beta; \lambda_n) - Z_n(\betastar; \lambda_n)) \\
    =& \argmin_\beta V_n(u) 
    = \betastar + \frac{1}{\sqrt{n}} \argmin_u V_n (u).
\end{align}
Therefore, we have $\argmin_u V_n(u) = \sqrt{n} (\betahat_n - \betastar)$ and
\begin{align}
    \sqrt{n}(\betahat_S - \betastar_S) \overset{d}{\rightarrow} \mathcal{N}(0, \sigma^2 C_{SS}^{-1})^\top, ~
    \sqrt{n} \betahat_{S^c} \overset{d}{\rightarrow} 0.
\end{align}

Variable Selection Consistency Part:
Asymptotic normality indicates that $\betahat \overset{p}{\rightarrow} \betastar$, thus 
\begin{align}
    \forall j \in S, ~ P(j \in \supp(\betahat)) \rightarrow 1.
\end{align}
Now, we consider the event $j \in S^c$ and $j \in \supp (\betahat)$.
By the KKT conditions, we have
\begin{align}
    &2 \x_j^\top (y - X \betahat) + \lambda_n w_j \sgn(\betahat_j) = 0.
\end{align}
This yields
\begin{align}
    2 \left(\frac{1}{n} \x_j^\top X \right)  \sqrt{n} (\betastar - \betahat) + 2 \frac{1}{\sqrt{n}} \x_j^\top \varepsilon
    + \lambda_n n^{(\gamma - 1)/2} |\sqrt{n} \betatilde_j|^{-\gamma} \sgn(\betahat_j) = 0
    \label{eq-adaptive-kkt}
\end{align}
By Slutsky's theorem, the first and second terms on the left-hand side converge to some normal distribution, but the the last term on the left-hand side diverges to infinity.
Therefore, we have for $\forall j \in S^c$,
\begin{align}
    P\left( j \in \supp (\betahat) \right) 
    \leq P\left( 2 \x_j^\top (y - X \betahat) + \lambda_n w_j \sgn(\betahat_j) = 0 \right)
    \rightarrow 0.
\end{align}

\subsection{Proofs of Transfer Lasso}

\label{proof-trlasso}

\subsubsection{Proof of Lemma~\ref{lemma:adapt-lasso-oracle-large}}
\label{proof:adapt-lasso-oracle-large}

The proof is similar to that of Lemma~\ref{lemma:adapt-lasso-oracle}.
If $\betatilde$ is a $\sqrt{m}$-consistent initial estimator and $\lambda_n \sqrt{m^\gamma / n} \rightarrow \infty$, 
\eqref{eq-asymptotic-zero} reduces to
\begin{align}
    \lambda_n w_j \left( \left| \betastar_j + \frac{u_j}{\sqrt{n}} \right| - |\betastar_j| \right)
    = \lambda_n \sqrt{\frac{m^\gamma}{n}} \left| \sqrt{m} \betatilde_j \right|^{-\gamma} |u_j|
    \overset{d}{\rightarrow}
    \begin{cases}
        0 & \text{if} ~ u_j=0\\
        \infty & \text{otherwise}
    \end{cases}
\end{align}
and \eqref{eq-adaptive-kkt} reduces to
\begin{align}
    2 \left(\frac{1}{n} \x_j^\top X \right)  \sqrt{n} (\betastar - \betahat) + 2 \frac{1}{\sqrt{n}} \x_j^\top \varepsilon
    + \lambda_n \sqrt{\frac{m^\gamma}{n}} |\sqrt{m} \betatilde_j|^{-\gamma} \sgn(\betahat_j) = 0.
\end{align}
These modifications does not affect the remaining proofs.

\subsubsection{Proof of Theorem \ref{theorem:trlasso-consistency} and Theorem \ref{theorem:tlasso-asymptotic-distribution-boundary}}
\label{proof:trlasso-consistency}

Let $u := l (\beta - \betastar)$ where $l = l(n, m, \lambda_n)$ is a certain function as defined later. 
Let $z := \sqrt{m} (\betatilde - \betastar)$. 
Since $\betatilde$ is a $\sqrt{m}$-consistent estimator, $z$ follows some distribution.
Suppose that $n / m \to r_0 \geq 0$.
The objective function for Transfer Lasso is
\begin{align}
    Z_n^\mathcal{T}(\beta)
    :=& \frac{1}{n} \| y - X \beta \|_2^2 + \frac{\lambda_n}{n} \sum_j |\beta_j| + \frac{\eta_n}{n} \sum_j |\beta_j - \betatilde_j|\\
    =& \frac{1}{n}\left\| \varepsilon - \frac{1}{l} X u \right\|_2^2 
    + \frac{\lambda_n}{n} \sum_j \left| \betastar_j + \frac{u_j}{l} \right| 
    + \frac{\eta_n}{n} \sum_j \left| \frac{u_j}{l} - \frac{z_j}{\sqrt{m}} \right|\\
    =& \frac{1}{l^2} u^\top \left( \frac{1}{n} X^\top X \right) u 
    - \frac{2}{\sqrt{n} l} u^\top \left( \frac{1}{\sqrt{n}} X^\top \varepsilon \right) 
    + \frac{1}{n} \|\varepsilon\|_2^2 \\
    +& \frac{\lambda_n}{nl} \sum_j \left| u_j + l \betastar_j \right| 
    + \frac{\eta_n}{nl} \sum_j \left| u_j- \frac{l}{\sqrt{m}} z_j \right|,
\end{align}
and we have
\begin{align}
    \label{eq:trlasso-zbeta-zbetastar}
    Z_n^\mathcal{T}(\beta) - Z_n^\mathcal{T}(\betastar)
    =& \frac{1}{l^2} u^\top \left( \frac{1}{n} X^\top X \right) u 
    - \frac{2}{\sqrt{n}l} u^\top \left( \frac{1}{\sqrt{n}} X^\top \varepsilon \right) \\
    +& \frac{\lambda_n}{nl} \sum_j \left( \left| u_j + l \betastar_j \right| -  \left| l \betastar_j \right| \right)
    + \frac{\eta_n}{nl} \sum_j \left( \left| u_j- \frac{l}{\sqrt{m}} z_j \right| - \left|\frac{l}{\sqrt{m}} z_j \right| \right).\nonumber
\end{align}
We divide the case into three cases: $l = \sqrt{m}$ (Case I), $\sqrt{n}$ (Case II), and $n / \lambda_n$ (Case III).

Case I.
Let $l = \sqrt{m}$. 
Then, \eqref{eq:trlasso-zbeta-zbetastar} reduces to
\begin{align}
    Z_n^\mathcal{T}(\beta) - Z_n^\mathcal{T}(\betastar)
    =& \frac{1}{m} u^\top \left( \frac{1}{n} X^\top X \right) u 
    - \frac{2}{\sqrt{nm}} u^\top \left( \frac{1}{\sqrt{n}} X^\top \varepsilon \right) \\
    +& \frac{\lambda_n}{n\sqrt{m}} \sum_j \left( \left| u_j + \sqrt{m} \betastar_j \right| -  \left| \sqrt{m} \betastar_j \right| \right)
    + \frac{\eta_n}{n\sqrt{m}} \sum_j \left( \left| u_j - z_j\right| - \left| z_j \right| \right).
\end{align}
Let $V_n(u) := (n\sqrt{m} / \eta_n) (Z_n^\mathcal{T}(\beta) - Z_n^\mathcal{T}(\betastar)) + \sum_j |z_j|$, then we have
\begin{align}
    V_n (u)
    &= \frac{n}{\sqrt{m} \eta_n} u^\top \left(\frac{1}{n}X^\top X\right) u 
    - \frac{2\sqrt{n}}{\eta_n} u^\top \left(\frac{1}{\sqrt{n}} X^\top \varepsilon \right)\\
    &+ \frac{\lambda_n}{\eta_n} \sum_j \left( \left| u_j + \sqrt{m} \betastar_j \right| -  \left| \sqrt{m} \betastar_j \right| \right)
    + \sum_j \left| u_j - z_j \right|.
\end{align}
On the other hand, we have
\begin{align}
    \betahat_n^\mathcal{T} =& \argmin_\beta Z_n^\mathcal{T}(\beta) 
    = \argmin_\beta \left( Z_n^\mathcal{T}(\beta) - Z_n^\mathcal{T}(\betastar) \right)
        = \betastar + \frac{1}{\sqrt{m}} \argmin_u V_n (u),
\end{align}
and hence
\begin{align}
    \sqrt{m} (\betahat_n - \betastar) = \argmin_u V_n(u).
\end{align}

Consider the case (i) where 
$\eta_n / \sqrt{n} \to \infty$, $\lambda_n / \eta_n \to \rho_0$, and $0 \leq \rho_0 < 1$.
Let $V(u) := \lim_{n\to\infty} V_n(u)$, then we have
\begin{align}
    \label{eq:trlasso-vu-i}
    V(u) = \sum_j \left( \left( \rho_0 u_j \sgn(\betastar_j) + \left| 
u_j - z_j \right| \right) I(\betastar_j \neq 0) + \left( \rho_0 |u_j| + |u_j - z_j| \right) I(\betastar_j = 0) \right).
\end{align}
Because $V_n(u)$ is convex and $V(u)$ has a unique minimum, we obtain
\begin{align}
    \sqrt{m} (\betahat_n - \betastar) \to 
    \argmin_u V(u) 
    = z.
\end{align}

Case II.
Let $l = \sqrt{n}$.
Then, \eqref{eq:trlasso-zbeta-zbetastar} reduces to
\begin{align}
    &Z_n^\mathcal{T}(\beta) - Z_n^\mathcal{T}(\betastar)\\
    =& \frac{1}{n} u^\top \left( \frac{1}{n} X^\top X \right) u 
    - \frac{2}{n} u^\top \left( \frac{1}{\sqrt{n}} X^\top \varepsilon \right) \\
    +& \frac{\lambda_n}{n\sqrt{n}} \sum_j \left( \left| u_j + \sqrt{n} \betastar_j \right| -  \left| \sqrt{n} \betastar_j \right| \right)
    + \frac{\eta_n}{n\sqrt{n}} \sum_j \left( \left| u_j- \sqrt{\frac{n}{m}}z_j \right| - \left| \sqrt{\frac{n}{m}} z_j \right| \right).
\end{align}
Let $V_n(u) := n (Z_n^\mathcal{T}(\beta) - Z_n^\mathcal{T}(\betastar))$, then we have
\begin{align}
    \label{eq:trlasso-vnu-ii}
    V_n(u)
    &= u^\top \left(\frac{1}{n}X^\top X\right) u 
    - 2 u^\top \left(\frac{1}{\sqrt{n}} X^\top \varepsilon \right)\\
    &+ \frac{\lambda_n}{\sqrt{n}} \sum_j \left( \left| u_j + \sqrt{n} \betastar_j \right| -  \left| \sqrt{n} \betastar_j \right| \right)
    + \frac{\eta_n}{\sqrt{n}} \sum_j \left( \left| u_j - \sqrt{\frac{n}{m}} z_j \right| -  \left|\sqrt{\frac{n}{m}} z_j \right| \right).
\end{align}

Consider the case (ii) where 
$\lambda_n / \sqrt{n} \to \lambda_0 \geq 0$ and $\eta_n / \sqrt{n} \to \eta_0 \geq 0$.
Let $V(u) := \lim_{n\to\infty} V_n(u)$, then we have
\begin{align}
    \label{eq:trlasso-vu-ii}
    V(u) 
    =& u^\top C u - 2 u^\top W\\
    +& \lambda_0 \sum_j \left( u_j \sgn(\betastar_j) I(\betastar_j \neq 0) + \left| u_j \right| I(\betastar_j = 0) \right) 
    + \eta_0 \sum_j \left( \left| u_j- \sqrt{r_0} z_j \right| - \left| \sqrt{r_0} z_j \right| \right).
\end{align}
On the other hand,we have
\begin{align}
    \betahat_n^\mathcal{T} =& \argmin_\beta Z_n^\mathcal{T}(\beta) 
    = \argmin_\beta \left( Z_n^\mathcal{T}(\beta) - Z_n^\mathcal{T}(\betastar) \right)
        = \betastar + \frac{1}{\sqrt{n}} \argmin_u V_n (u),
\end{align}
and hence
\begin{align}
    \sqrt{n} (\betahat_n - \betastar) = \argmin_u V_n(u).
\end{align}
Because $V_n(u)$ is convex and $V(u)$ has a unique minimum, we obtain
\begin{align}
    &\sqrt{n} \left( \betahat_n - \betastar \right) \to
    \argmin_u \Biggl\{ u^\top C u - 2 u^\top W\\
    &+ \lambda_0 \sum_j \left( u_j \sgn(\betastar_j) I(\betastar_j \neq 0) + \left| u_j \right| I(\betastar_j = 0) \right) 
    + \eta_0 \sum_j \left( \left| u_j- \sqrt{r_0}z_j \right| - \left| \sqrt{r_0} z_j \right| \right)
    \Biggr\}.
\end{align}

Case III.
Let $l = n / \lambda_n (\to \infty)$.
Then, \eqref{eq:trlasso-zbeta-zbetastar} reduces to
\begin{align}
    &Z_n^\mathcal{T}(\beta) - Z_n^\mathcal{T}(\betastar)\\
    =& \frac{\lambda_n^2}{n^2} u^\top \left( \frac{1}{n} X^\top X \right) u 
    - \frac{\lambda_n}{n\sqrt{n}} u^\top \left( \frac{1}{\sqrt{n}} X^\top \varepsilon \right) \\
    +& \frac{\lambda_n^2}{n^2} \sum_j \left( \left| u_j + \frac{n}{\lambda_n} \betastar_j \right| -  \left| \frac{n}{\lambda_n} \betastar_j \right| \right)
    + \frac{\lambda_n \eta_n}{n^2} \sum_j \left( \left| u_j- \sqrt{\frac{n}{m}}\frac{\sqrt{n}}{\lambda_n} z_j \right| - \left| \sqrt{\frac{n}{m}}\frac{\sqrt{n}}{\lambda_n} z_j \right| \right)
\end{align}
Let $V_n(u) := (n^2 / \lambda_n^2) (Z_n^\mathcal{T}(\beta) - Z_n^\mathcal{T}(\betastar))$, then we have
\begin{align}
    V_n(u)
    &= u^\top \left(\frac{1}{n}X^\top X\right) u 
    - \frac{2\sqrt{n}}{\lambda_n} u^\top \left(\frac{1}{\sqrt{n}} X^\top \varepsilon \right)\\
    &+ \sum_j \left( \left| u_j + \frac{n}{\lambda_n} \betastar_j \right| -  \left| \frac{n}{\lambda_n} \betastar_j \right| \right)
    + \frac{\eta_n}{\lambda_n} \sum_j \left( \left| u_j - \sqrt{\frac{n}{m}}\frac{\sqrt{n}}{\lambda_n} z_j \right| -  \left|\sqrt{\frac{n}{m}}\frac{\sqrt{n}}{\lambda_n} z_j \right| \right).
\end{align}

Consider the case (iii) where 
$\lambda_n / \sqrt{n} \to \infty$, $\lambda_n / n \to 0$, and $\eta_n / \lambda_n \to \rho_0' \geq 0$.
Let $V(u) := \lim_{n\to\infty} V_n(u)$, then we have
\begin{align}
    \label{eq:trlasso-vu-iv}
    V(u)
    = u^\top C u + \sum_j \left( \left( u_j \sgn(\betastar_j) + \rho_0' \left| u_j \right| \right) I(\betastar_j \neq 0) + (1 + \rho_0') \left| u_j \right| I(\betastar_j = 0) \right).
\end{align}
On the other hand,we have
\begin{align}
    \betahat_n^\mathcal{T} =& \argmin_\beta Z_n^\mathcal{T}(\beta) 
    = \argmin_\beta \left( Z_n^\mathcal{T}(\beta) - Z_n^\mathcal{T}(\betastar) \right)
        = \betastar + \frac{\lambda_n}{n} \argmin_u V_n (u),
\end{align}
and hence
\begin{align}
    \frac{n}{\lambda_n} (\betahat_n - \betastar) = \argmin_u V_n(u).
\end{align}
Because $V_n(u)$ is convex and $V(u)$ has a unique minimum, we obtain
\begin{align}
    \label{eq:trlasso-iv-asymptotic-distribution}
    &\frac{n}{\lambda_n} (\betahat_n - \betastar) \\
    \to& 
    \argmin_u \left\{ u^\top C u + \sum_j \left( \left( u_j \sgn(\betastar_j) + \rho_0' \left| u_j \right| \right) I(\betastar_j \neq 0) + (1 + \rho_0') \left| u_j \right| I(\betastar_j = 0) \right) \right\}.\nonumber \\
\end{align}
In addition, if $\rho_0' \geq 1$, then the right-hand side of \eqref{eq:trlasso-iv-asymptotic-distribution} reduces to $0$.

\subsubsection{Proof of Theorem \ref{theorem:trlasso-active-selection-consistency}}
\label{proof:trlasso-active-selection-consistency}

By Theorem~\ref{theorem:trlasso-consistency} and Corollary~\ref{corollary:trlasso-convergence}, the Transfer Lasso estimator satisfies $\betahat_n^\mathcal{T} \overset{p}{\to} \betastar$, thus
\begin{align}
    \forall j \in S, ~ \limsup_{n \to \infty} P(j \in \supp(\betahat_n^\mathcal{T})) = 1.
\end{align}
Let $u^* := \argmin_u V(u)$ where $V(u)$ is the asymptotic objective function \eqref{eq:trlasso-vu-i} in the case (i) and \eqref{eq:trlasso-vu-ii} in the case (ii).
By the weak convergence result, we have
\begin{align}
    \limsup_{n \to \infty} P(\hat{S}_n^\mathcal{T} = S) 
    =& \limsup_{n \to \infty} P(\betahat_j^\mathcal{T} \neq 0 ~ \forall j \in S ~\text{and}~ \betahat_j^\mathcal{T} = 0 ~ \forall j \in S^c) \\
    \leq& \limsup_{n \to \infty} P(l \betahat_j^\mathcal{T} = 0 ~ \forall j \in S^c) \\
    \leq& P(u_j^* = 0 ~ \forall j \in S^c),
\end{align}
where $l$ is $\sqrt{m}$ in the case (i) and $\sqrt{n}$ in the case (ii).
We evaluate the probability of $u_{S^c}^* = 0$ in each case.

Consider the case (i) where $\eta_n / \sqrt{n} \to \infty$ and $\lambda_n / \eta_n \to \rho_0$ with $0 \leq \rho_0 < 1$.
The asymptotic distribution of the Transfer Lasso is equal to that of the initial estimator $z$.
Because we suppose that $z$ is inconsistent in terms of variable selection, 
we obtain
\begin{align}
    P(u_j^* = 0 ~ \forall j \in S^c)
    = P(z_j = 0 ~ \forall j \in S^c)
    \leq c < 1,
\end{align}
hence $\betahat_n^\mathcal{T}$ is inconsistent in terms of variable selection.

Consider the case (ii) where $\lambda_n / \sqrt{n} \to \lambda_0 \geq 0$ and $\eta_n / \sqrt{n} \to \eta_0 \geq 0$.
Suppose that $u_{S^c}^* = 0$.
Let 
$S_1 := \{ j : j \in S ~\text{and}~ u_j^* \neq \sqrt{r_0} z_j \}$, 
$S_2 :=  \{ j : j \in S ~\text{and}~ u_j^* = \sqrt{r_0} z_j \}$, 
$S_1^c := \{ j : j \in S^c ~\text{and}~ u_j^* \neq \sqrt{r_0} z_j \}$, and 
$S_2^c :=  \{ j : j \in S^c ~\text{and}~ u_j^* = \sqrt{r_0} z_j \}$.
By the KKT conditions of $\argmin_u V(u)$, we have
\begin{align}
\label{eq:trlasso-kkt-ii-active}
\begin{cases}
    2 C_{S_1 S_1} u_{S_1}^* + 2 \sqrt{r_0} C_{S_1 S_2} z_{S_2} - 2 W_{S_1} + \lambda_0 \sgn(\betastar_{S_1}) + \eta_0 \sgn(u_{S_1}^* - \sqrt{r_0}z_{S_1}) = 0\\
    \left| 2 C_{S_2 S_1} u_{S_1}^* + 2 \sqrt{r_0} C_{S_2 S_2} z_{S_2} - 2 W_{S_2} + \lambda_0 \sgn(\betastar_{S_2}) \right| \leq \eta_0\\
    \left| 2 C_{S^c S_1} u_{S_1}^* + 2 \sqrt{r_0} C_{S_1^c S_2} z_{S_2} - 2 W_{S_1^c} + \eta_0 \sgn(- \sqrt{r_0}z_{S_1^c}) \right| \leq \lambda_0\\
    \left| 2 C_{S^c S_1} u_{S_1}^* + 2 \sqrt{r_0} C_{S_2^c S_2} z_{S_2} - 2 W_{S_2^c} \right| \leq \lambda_0 + \eta_0.
\end{cases}
\end{align}
If $S_1 \neq \emptyset$, the first equation yields
\begin{align}
    u_{S_1}^* = C_{S_1 S_1}^{-1} \left( W_{S_1} - \sqrt{r_0} C_{S_1 S_2} z_{S_2} - \frac{\lambda_0}{2} \sgn(\betastar_{S_1}) - \frac{\eta_0}{2} \sgn(u_{S_1}^* - \sqrt{r_0} z_{S_1}) \right).
\end{align}
Because $W$ follows a Gaussian and $z$ follows some distribution, 
the probability holding the second, third, and fourth inequalities in \eqref{eq:trlasso-kkt-ii-active} is less than $1$.
This indicates inconsistent variable selection.
If $S_1 = \emptyset$, by the KKT conditions, we have
\begin{align}
\begin{cases}
    \left| 2 \sqrt{r_0} C_{S_2 S_2} z_{S_2} - 2 W_{S_2} + \lambda_0 \sgn(\betastar_{S_2}) \right| \leq \eta_0\\
    \left| 2 \sqrt{r_0} C_{S_1^c S_2} z_{S_2} - 2 W_{S_1^c} + \eta_0 \sgn(- \sqrt{r_0}z_{S_1^c}) \right| \leq \lambda_0\\
    \left| 2 \sqrt{r_0} C_{S_2^c S_2} z_{S_2} - 2 W_{S_2^c} \right| \leq \lambda_0 + \eta_0.
\end{cases}
\end{align}
The probability holding these inequalities is less than $1$.
This indicates inconsistent variable selection.

\subsubsection{Proof of Theorem \ref{theorem:trlasso-varying-selection-consistency}}
\label{proof:trlasso-varying-selection-consistency}

Consider the case (i) in Theorem~\ref{theorem:trlasso-consistency}.
Consider the event where $\betahat_j \neq \betatilde_j$ for some $j \in S$.
By the KKT conditions, 
\begin{align}
    \label{eq:trlasso-i-kkt}
    \begin{dcases}
    2 \left( \frac{1}{n} \x_j^\top X \right) \sqrt{\frac{n}{m}} \sqrt{m} (\betahat - \betastar)
    - \frac{2}{\sqrt{n}} \x_j^\top \varepsilon 
    + \frac{\lambda_n}{\sqrt{n}} \sgn(\betahat_j)  
    + \frac{\eta_n}{\sqrt{n}} \sgn( \betahat_j - \betatilde_j ) 
    = 0, \\
    ~~~~~~~~~~~~~~~~~~~~~~~~~~~~~~~~~~~~~~~~~~~~~~~~~~~~~~~~~~~~~~~~~~~~~~~~~~~~~~~~~~~~~~~~~~~~~~~~~~~~~~~~~~~~~~~~~~~~~~\text{for}~ \betahat_j \neq 0,\\
    \left| 2 \left( \frac{1}{n} \x_j^\top X \right) \sqrt{\frac{n}{m}} \sqrt{m} (\betahat - \betastar)
    - \frac{2}{\sqrt{n}} \x_j^\top \varepsilon 
    + \frac{\eta_n}{\sqrt{n}} \sgn(\betahat_j - \betatilde_j) \right|
    \leq \frac{\lambda_n}{\sqrt{n}}, 
    ~\text{for}~ \betahat_j = 0.\\
    \end{dcases}
\end{align}
The term including $\eta_n / \sqrt{n}$ in \eqref{eq:trlasso-i-kkt} in both $\betahat_j \neq 0$ and $\betahat_j = 0$ cases diverge to infinity faster compared to the rest terms.
Therefore, we have
\begin{align}
    \forall j \in S, ~ \lim_{n \to \infty} P(\betahat_j \neq \betatilde_j) = 0.
\end{align}
This concludes
\begin{align}
    \lim_{n \to \infty} P(\betahat_S = \betatilde_S) = 1.
\end{align}

\subsubsection{Proof of Theorem \ref{theorem:trlasso-invariant-selection-consistency}}
\label{proof:trlasso-invariant-selection-consistency}

By Theorem~\ref{theorem:trlasso-consistency} and Corollary~\ref{corollary:trlasso-convergence}, the Transfer Lasso estimator satisfies $\betahat_n^\mathcal{T} \overset{p}{\to} \betastar$, thus
\begin{align}
    \forall j \in S, ~ P(j \in \supp(\betahat_n^\mathcal{T})) \to 1.
\end{align}
Consider the case (ii) where $\lambda_n / \sqrt{n} \to \lambda_0 \geq 0$ and $\eta_n / \sqrt{n} \to \eta_0 \geq 0$.
Let $u^* := \argmin_u V(u)$ where $V(u)$ is the asymptotic objective function \eqref{eq:trlasso-vu-ii} for the case (ii).
By the weak convergence result, we have
\begin{align}
    \limsup_{n \to \infty} P(\betahat_S^\mathcal{T} = \betatilde_S)
    \leq P \left( u_j^* = \sqrt{r_0} z_j ~ \forall j \in S \right).
\end{align}
Suppose that $u_S^* = \sqrt{r_0} z_S$.
Let $S_1 := \{ j : j \in S ~\text{and}~ u_j^* \neq 0 \}$, 
$S_2 :=  \{ j : j \in S ~\text{and}~ u_j^* = 0 \}$, 
$S_1^c := \{ j : j \in S^c ~\text{and}~ u_j^* \neq 0 ~\text{and}~ u_j^* \neq \sqrt{r_0} z_j \}$, 
$S_2^c :=  \{ j : j \in S^c ~\text{and}~ u_j^* \neq 0 ~\text{and}~ u_j^* = \sqrt{r_0} z_j \}$,
$S_3^c := \{ j : j \in S^c ~\text{and}~ u_j^* = 0 ~\text{and}~ u_j^* \neq \sqrt{r_0} z_j \}$, and
$S_4^c :=  \{ j : j \in S^c ~\text{and}~ u_j^* = 0 ~\text{and}~ u_j^* = \sqrt{r_0} z_j \}$.
By the KKT conditions of $\argmin_u V(u)$, we have
\begin{align}
\label{eq:trlasso-kkt-ii-invariant}
\begin{cases}
    \left| 2 \sqrt{r_0} C_{S S} z_S + 2 C_{S S^c} u_{S^c}^* - 2 W_{S} + \lambda_0 \sgn(\betastar_{S}) \right| \leq \eta_0\\
    2 \sqrt{r_0} C_{S_1^c S} z_S + 2 C_{S_1^c S^c} u_{S^c}^* - 2 W_{S_1^c} + \lambda_0 \sgn(u_{S_1^c}) + \eta_0 \sgn(u_{S_1^c}^* - \sqrt{r_0}z_{S_1^c}) = 0\\
    \left| 2 \sqrt{r_0} C_{S_2^c S} z_S + 2 C_{S_2^c S^c} u_{S^c}^* - 2 W_{S_2^c} + \lambda_0 \sgn(u_{S_2^c}) \right| \leq \eta_0\\
    \left| 2 \sqrt{r_0} C_{S_3^c S} z_S + 2 C_{S_3^c S^c} u_{S^c}^* - 2 W_{S_3^c} + \eta_0 \sgn(u_{S_3^c}^* - \sqrt{r_0}z_{S_3^c}) \right| \leq \lambda_0\\
    \left| 2 \sqrt{r_0} C_{S_4^c S} z_S + 2 C_{S_4^c S^c} u_{S^c} - 2 W_{S_4^c} \right| \leq \lambda_0 + \eta_0.
\end{cases}
\end{align}
Note that $u_{S_2^c} = \sqrt{r_0} z_{S_2^c}$, $u_{S_2^c} = 0$, and $u_{S_4^c} = \sqrt{r_0} z_{S_4^c} = 0$.
If $S_1 \neq \emptyset$, the second equation yields
\begin{align}
    u_{S_1^c}^* = C_{S_1 S_1}^{-1} \biggl(& W_{S_1^c} - \sqrt{r_0} C_{S_1^c S} z_S - \sqrt{r_0} C_{S_1^c S_2^c} z_{S_2^c} \\
    &- \frac{\lambda_0}{2} \sgn(u_{S_1^c}^*) - \frac{\eta_0}{2} \sgn(u_{S_1^c}^* - \sqrt{r_0} z_{S_1^c}) \biggr).
\end{align}
Hence, we have
\begin{align}
\begin{cases}
    \left| 2 \sqrt{r_0} C_{S S} z_S + 2 C_{S S_1^c} u_{S_1^c}^* + 2 \sqrt{r_0} C_{S S_2^c} z_{S_2^c} - 2 W_{S} + \lambda_0 \sgn(\betastar_{S}) \right| \leq \eta_0\\
    \left| 2 \sqrt{r_0} C_{S_2^c S} z_S + 2 C_{S_2^c S_1^c} u_{S_1^c}^* + 2 \sqrt{r_0} C_{S_2^c S_2^c} z_{S_2^c} - 2 W_{S_2^c} + \lambda_0 \sgn(u_{S_2^c}) \right| \leq \eta_0\\
    \left| 2 \sqrt{r_0} C_{S_3^c S} z_S + 2 C_{S_3^c S_1^c} u_{S_1^c}^* + 2 \sqrt{r_0} C_{S_3^c S_2^c} z_{S_2^c} - 2 W_{S_3^c} + \eta_0 \sgn(u_{S_3^c} - \sqrt{r_0}z_{S_3^c}) \right| \leq \lambda_0\\
    \left| 2 \sqrt{r_0} C_{S_4^c S} z_S + 2 C_{S_4^c S_1^c} u_{S_1^c}^* + 2 \sqrt{r_0} C_{S_4^c S_2^c} z_{S_2^c} - 2 W_{S_4^c} \right| \leq \lambda_0 + \eta_0.
\end{cases}
\end{align}
Because $W$ follows a Gaussian distribution and $u_{S_1^c}^*$ and $z$ follow some distribution, the probability holding these inequalities is less than $1$.
This indicates inconsistent invariant variable selection.
If $S_1 = \emptyset$, by the KKT conditions, we have
\begin{align}
\begin{cases}
    \left| 2 \sqrt{r_0} C_{S S} z_S + 2 \sqrt{r_0} C_{S S_2^c} z_{S_2^c} - 2 W_{S} + \lambda_0 \sgn(\betastar_{S}) \right| \leq \eta_0\\
    \left| 2 \sqrt{r_0} C_{S_2^c S} z_S + 2 \sqrt{r_0} C_{S_2^c S_2^c} z_{S_2^c} - 2 W_{S_2^c} + \lambda_0 \sgn(u_{S_2^c}) \right| \leq \eta_0\\
    \left| 2 \sqrt{r_0} C_{S_3^c S} z_S + 2 \sqrt{r_0} C_{S_3^c S_2^c} z_{S_2^c} - 2 W_{S_3^c} + \eta_0 \sgn(u_{S_3^c} - \sqrt{r_0}z_{S_3^c}) \right| \leq \lambda_0\\
    \left| 2 \sqrt{r_0} C_{S_4^c S} z_S + 2 \sqrt{r_0} C_{S_4^c S_2^c} z_{S_2^c} - 2 W_{S_4^c} \right| \leq \lambda_0 + \eta_0.
\end{cases}
\end{align}
Because $W$ follows a Gaussian distribution and $z$ follows some distribution, the probability holding these inequalities is less than $1$.
This indicates inconsistent variable selection.

\subsection{Proofs of Adaptive Transfer Lasso}

\subsubsection{Proof of Theorem \ref{theorem:adaptrlasso-consistency}}
\label{proof:adaptrlasso-consistency}

Let $u := l (\beta - \betastar)$ where $l = l(n, m, \lambda_n)$ is a certain function as defined later. 
Let $z := \sqrt{m} (\betatilde - \betastar)$. 
Since $\betatilde$ is a $\sqrt{m}$-consistent estimator, $z$ follows some distribution.
We suppose that $n / m \to r_0 \geq 0$.
The objective function for the Adaptive Transfer Lasso is
\begin{align}
    Z_n^\# (\beta)
    :=& \frac{1}{n} \| y - X \beta \|_2^2 + \frac{\lambda_n}{n} \sum_j \frac{|\beta_j|}{|\betatilde_j|^{\gamma_1}} + \frac{\eta_n}{n} \sum_j |\betatilde_j|^{\gamma_2} |\beta_j - \betatilde_j|\\
    =& \frac{1}{n}\left\| \varepsilon - \frac{1}{l} X u \right\|_2^2 
    + \frac{\lambda_n}{n} \sum_j \frac{\left| \betastar_j + \frac{u_j}{l} \right|}{\left| \frac{z_j}{\sqrt{m}} + \betastar_j \right|^{\gamma_1}} 
    + \frac{\eta_n}{n} \sum_j \left| \frac{z_j}{\sqrt{m}} + \betastar_j \right|^{\gamma_2} \left| \frac{u_j}{l} - \frac{z_j}{\sqrt{m}} \right|\\
    =& \frac{1}{l^2} u^\top \left( \frac{1}{n} X^\top X \right) u 
    - \frac{2}{\sqrt{n} l} u^\top \left( \frac{1}{\sqrt{n}} X^\top \varepsilon \right) 
    + \frac{1}{n} \|\varepsilon\|_2^2 \\
    +& \frac{\sqrt{m^{\gamma_1}}\lambda_n}{nl} \sum_j \frac{1}{\left| z_j + \sqrt{m} \betastar_j \right|^{\gamma_1}} \left| u_j + l \betastar_j \right| \\
    +& \frac{\eta_n}{n\sqrt{m^{\gamma_2}}l} \sum_j \left| z_j + \sqrt{m} \betastar_j \right|^{\gamma_2} \left| u_j- \frac{l}{\sqrt{m}} z_j \right|,
\end{align}
and we have
\begin{align}
\label{eq:adaptrlasso-zbeta-zbetastar}
    Z_n^\# (\beta) - Z_n^\# (\betastar)
    =& \frac{1}{l^2} u^\top \left( \frac{1}{n} X^\top X \right) u 
    - \frac{2}{\sqrt{n}l} u^\top \left( \frac{1}{\sqrt{n}} X^\top \varepsilon \right) \\
    +& \frac{\sqrt{m^{\gamma_1}}\lambda_n}{nl} \sum_j \frac{1}{\left| z_j + \sqrt{m} \betastar_j \right|^{\gamma_1}} \left( \left| u_j + l \betastar_j \right| - \left| l \betastar_j \right| \right) \nonumber \\
    +& \frac{\eta_n}{n\sqrt{m^{\gamma_2}}l} \sum_j \left| z_j + \sqrt{m} \betastar_j \right|^{\gamma_2} \left( \left| u_j- \frac{l}{\sqrt{m}} z_j \right| -  \left| \frac{l}{\sqrt{m}} z_j \right| \right). \nonumber
\end{align}

We divide the case into three cases: $l = \sqrt{m}$ (Case I), $l = \sqrt{n}$ (Case II), and $l = n / \lambda_n$ (Case III).

Case I.
Let $l = \sqrt{m}$.
Then, \eqref{eq:adaptrlasso-zbeta-zbetastar} reduces to
\begin{align}
    &Z_n^\# (\beta) - Z_n^\# (\betastar)\\
    =& \frac{1}{m} u^\top \left( \frac{1}{n} X^\top X \right) u 
    - \frac{2}{\sqrt{nm}} u^\top \left( \frac{1}{\sqrt{n}} X^\top \varepsilon \right) \\
    +& \frac{\sqrt{m^{\gamma_1 - 1}}\lambda_n}{n} \sum_j \frac{1}{\left| z_j + \sqrt{m} \betastar_j \right|^{\gamma_1}} \left( \left| u_j + \sqrt{m} \betastar_j \right| - \left| \sqrt{m} \betastar_j \right| \right) \nonumber \\
    +& \frac{\eta_n}{n\sqrt{m^{\gamma_2 + 1}}} \sum_j \left| z_j + \sqrt{m} \betastar_j \right|^{\gamma_2} \left( \left| u_j- z_j \right| -  \left| z_j \right| \right) \nonumber \\
    =& \frac{1}{m} u^\top \left( \frac{1}{n} X^\top X \right) u 
    - \frac{2}{\sqrt{nm}} u^\top \left( \frac{1}{\sqrt{n}} X^\top \varepsilon \right) \\
    +& \frac{\lambda_n}{n\sqrt{m}} \sum_j \left( \frac{\left| u_j + \sqrt{m} \betastar_j \right| - \left| \sqrt{m} \betastar_j \right|}{\left| \betastar_j + \frac{z_j}{\sqrt{m}} \right|^{\gamma_1}} I(\betastar_j \neq 0) + \frac{\sqrt{m^{\gamma_1}} |u_j|}{\left| z_j \right|^{\gamma_1}} I(\betastar_j = 0) \right) \nonumber \\
    +& \frac{\eta_n}{n\sqrt{m}} \sum_j \left( \left| \betastar_j + \frac{z_j}{\sqrt{m}} \right|^{\gamma_2} I(\betastar_j \neq 0) + \frac{\left| z_j \right|^{\gamma_2}}{\sqrt{m^{\gamma_2}}} I(\betastar_j = 0) \right) \left( \left| u_j- z_j \right| -  \left| z_j \right| \right). \nonumber
\end{align}

Consider the case (i) where 
\begin{align}
    \frac{\eta_n}{n \sqrt{m^{\gamma_2+1}}} \gg \frac{1}{\sqrt{nm}}
    ~\text{and}~
    \frac{\eta_n}{n\sqrt{m^{\gamma_2 + 1}}} \gg \frac{\sqrt{m^{\gamma_1 - 1}}\lambda_n}{n},
\end{align}
that is,
\begin{align}
    \frac{\eta_n}{\sqrt{n m^{\gamma_2}}} \to \infty
    ~\text{and}~
        \frac{\sqrt{m^{\gamma_1 + \gamma_2}} \lambda_n}{\eta_n} \to 0.
\end{align}
Let $V_{n}(u) := (n \sqrt{m^{\gamma_2 + 1}} / \eta_n) (Z_n^\#(\beta) - Z_n^\# (\betastar)) + (\text{term irrelevant to}~ \beta)$, then we have
\begin{align}
    V_n(u)
    =&\frac{\sqrt{n m^{\gamma_2}}}{\eta_n} \sqrt{\frac{n}{m}} u^\top \left( \frac{1}{n} X^\top X \right) u 
    - \frac{2\sqrt{n m^{\gamma_2}}}{\eta_n} u^\top \left( \frac{1}{\sqrt{n}} X^\top \varepsilon \right) \\
    +& \frac{\sqrt{m^{\gamma_1 + \gamma_2}} \lambda_n}{\eta_n} \sum_j \left( \frac{\left| u_j + \sqrt{m} \betastar_j \right| - \left| \sqrt{m} \betastar_j \right|}{\left| z_j + \sqrt{m} \betastar_j \right|^{\gamma_1}} I(\betastar_j \neq 0) + \frac{|u_j|}{\left| z_j \right|^{\gamma_1}} I(\betastar_j = 0) \right) \nonumber \\
    +& \sum_j \left( \left| z_j + \sqrt{m} \betastar_j \right|^{\gamma_2} I(\betastar_j \neq 0) + \left| z_j \right|^{\gamma_2} I(\betastar_j = 0) \right) \left| u_j- z_j \right|. \nonumber
\end{align}
Let $V(u) := \lim_{n \to \infty} V_n(u)$, then we have
\begin{align}
    V(u) 
    =& \begin{dcases}
                \sum_j \left| z_j \right|^{\gamma_2} \left| u_j- z_j \right| I(\betastar_j = 0) &\text{for}~ u_S = z_S, \\
        \infty &\text{otherwise}.
    \end{dcases}
\end{align}
Therefore, we obtain
\begin{align}
    \sqrt{m} (\betahat_n^\# - \betastar) = \argmin_u V_n(u) \overset{d}{\to} \argmin_u V(u) = z.
\end{align}

Consider the case (ii) where
\begin{align}
    \frac{\sqrt{m^{\gamma_1 - 1}} \lambda_n}{n} \gg \frac{1}{\sqrt{nm}},
    \frac{\sqrt{m^{\gamma_1 - 1}} \lambda_n}{n} \gg \frac{\eta_n}{n \sqrt{m^{\gamma_2 + 1}}},
\end{align}
\begin{align}
    \frac{\eta_n}{n \sqrt{m}} \gg \frac{1}{\sqrt{nm}},
    \frac{\eta_n}{n \sqrt{m}} \gg \frac{\lambda_n}{n \sqrt{m}},
\end{align}
that is,
\begin{align}
    \frac{\sqrt{m^{\gamma_1}} \lambda_n}{\sqrt{n}} \to \infty,
    \frac{\eta_n}{\sqrt{n}} \to \infty,
    \frac{\eta_n}{\lambda_n} \to \infty,
    \frac{\eta_n}{\sqrt{m^{\gamma_1 + \gamma_2}}\lambda_n} \to 0.
\end{align}
We divide into three cases:
\begin{align}
    &\text{(ii-a)} ~ \frac{\eta_n}{n \sqrt{m}} \gg \frac{\sqrt{m^{\gamma_1 - 1}} \lambda_n}{n}, ~~~
    \text{(ii-b)} ~ \frac{\eta_n}{n \sqrt{m}} \ll \frac{\sqrt{m^{\gamma_1 - 1}} \lambda_n}{n}, \\
    &\text{(ii-c)} ~ \frac{\eta_n}{n \sqrt{m}} \asymp \frac{\sqrt{m^{\gamma_1 - 1}} \lambda_n}{n},
\end{align}
that is,
\begin{align}
    \text{(ii-a)} ~ \frac{\sqrt{m^{\gamma_1}}\lambda_n}{\eta_n} \to 0, ~~~
    \text{(ii-b)} ~ \frac{\sqrt{m^{\gamma_1}}\lambda_n}{\eta_n} \to \infty, ~~~
    \text{(ii-c)} ~ \frac{\sqrt{m^{\gamma_1}}\lambda_n}{\eta_n} \to \rho_0 > 0.
\end{align}

In the case (ii-a), let $V_n(u) := (n / \sqrt{m^{\gamma_1 - 1}} \lambda_n) (Z_n^\#(\beta) - Z_n^\# (\betastar))  + (\text{term irrelevant to}~ \beta)$, then we have
\begin{align}
    V_n(u)
    =& \frac{\sqrt{n}}{\sqrt{m^{\gamma_1}} \lambda_n} \sqrt{\frac{n}{m}} u^\top \left( \frac{1}{n} X^\top X \right) u 
    - \frac{2\sqrt{n}}{\sqrt{m^{\gamma_1}} \lambda_n} u^\top \left( \frac{1}{\sqrt{n}} X^\top \varepsilon \right) \\
    +& \sum_j \left( \frac{\left| u_j + \sqrt{m} \betastar_j \right| - \left| \sqrt{m} \betastar_j \right|}{\left| z_j + \sqrt{m} \betastar_j \right|^{\gamma_1}} I(\betastar_j \neq 0) + \frac{|u_j|}{\left| z_j \right|^{\gamma_1}} I(\betastar_j = 0) \right) \nonumber \\
    +& \frac{\eta_n}{\sqrt{m^{\gamma_1}}\lambda_n} \sum_j \left( \left| \betastar_j + \frac{z_j}{\sqrt{m}} \right|^{\gamma_2} I(\betastar_j \neq 0) + \frac{\left| z_j \right|^{\gamma_2}}{\sqrt{m^{\gamma_2}}} I(\betastar_j = 0) \right) \left| u_j- z_j \right|. \nonumber
\end{align}
Let $V(u) := \lim_{n \to \infty} V_n(u)$, then we have
\begin{align}
    V(u) 
    =& \begin{dcases}
        \sum_j \frac{|u_j|}{|z_j|^{\gamma_1}} I(\betastar_j = 0) &\text{for}~ u_S = z_S, \\
        \infty &\text{otherwise}.
    \end{dcases}
\end{align}
Therefore, we obtain
\begin{align}
    \sqrt{m} (\betahat_n^\# - \betastar) = \argmin_u V_n(u) \overset{d}{\to} \argmin_u V(u) =
    \begin{dcases}
        0 & \text{for}~ j \in S^c,\\
        z_j & \text{for}~ j \in S.
    \end{dcases}
\end{align}

In the case (ii-b), let $V_n(u) := (n \sqrt{m} / \eta_n) (Z_n^\#(\beta) - Z_n^\# (\betastar)) + (\text{term irrelevant to}~ \beta)$,
then we have
\begin{align}
    V_n(u)
    =& \frac{\sqrt{n}}{\eta_n} \sqrt{\frac{n}{m}} u^\top \left( \frac{1}{n} X^\top X \right) u 
    - \frac{2\sqrt{n}}{\eta_n} u^\top \left( \frac{1}{\sqrt{n}} X^\top \varepsilon \right) \\
    +& \frac{\sqrt{m^{\gamma_1}} \lambda_n}{\eta_n} \sum_j \left( \frac{\left| u_j + \sqrt{m} \betastar_j \right| - \left| \sqrt{m} \betastar_j \right|}{\left| z_j + \sqrt{m} \betastar_j \right|^{\gamma_1}} I(\betastar_j \neq 0) + \frac{|u_j|}{\left| z_j \right|^{\gamma_1}} I(\betastar_j = 0) \right) \nonumber \\
    +& \sum_j \left( \left| \betastar_j + \frac{z_j}{\sqrt{m}} \right|^{\gamma_2} I(\betastar_j \neq 0) + \frac{\left| z_j \right|^{\gamma_2}}{\sqrt{m^{\gamma_2}}} I(\betastar_j = 0) \right) \left| u_j- z_j \right|. \nonumber
\end{align}
Let $V(u) := \lim_{n \to \infty} V_n(u)$, then we have
\begin{align}
    V(u) 
    =& \begin{dcases}
        \sum_j |\betastar_j|^{\gamma_2} \left| u_j - z_j \right| I(\betastar_j \neq 0) &\text{for}~ u_{S^c} = 0, \\
        \infty &\text{otherwise}.
    \end{dcases}
\end{align}
Therefore, we obtain
\begin{align}
    \sqrt{m} (\betahat_n^\# - \betastar) = \argmin_u V_n(u) \overset{d}{\to} \argmin_u V(u) = 
    \begin{dcases}
        0 & \text{for}~ j \in S^c,\\
        z_j & \text{for}~ j \in S.
    \end{dcases}
\end{align}

In the case (ii-c), let $V_n(u) := (n \sqrt{m} / \eta_n) (Z_n^\#(\beta) - Z_n^\# (\betastar)) + (\text{term irrelevant to}~ \beta)$,
then we have
\begin{align}
    V_n(u)
    =& \frac{\sqrt{n}}{\eta_n} \sqrt{\frac{n}{m}} u^\top \left( \frac{1}{n} X^\top X \right) u 
    - \frac{2\sqrt{n}}{\eta_n} u^\top \left( \frac{1}{\sqrt{n}} X^\top \varepsilon \right) \\
    +& \frac{\sqrt{m^{\gamma_1}} \lambda_n}{\eta_n} \sum_j \left( \frac{\left| u_j + \sqrt{m} \betastar_j \right| - \left| \sqrt{m} \betastar_j \right|}{\left| z_j + \sqrt{m} \betastar_j \right|^{\gamma_1}} I(\betastar_j \neq 0) + \frac{|u_j|}{\left| z_j \right|^{\gamma_1}} I(\betastar_j = 0) \right) \nonumber \\
    +& \sum_j \left( \left| \betastar_j + \frac{z_j}{\sqrt{m}} \right|^{\gamma_2} I(\betastar_j \neq 0) + \frac{\left| z_j \right|^{\gamma_2}}{\sqrt{m^{\gamma_2}}} I(\betastar_j = 0) \right) \left| u_j- z_j \right|. \nonumber
\end{align}
Let $V(u) := \lim_{n \to \infty} V_n(u)$, then we have
\begin{align}
    V(u) 
    = \sum_j \left( |\betastar_j|^{\gamma_2} \left| u_j - z_j \right| I(\betastar_j \neq 0)
    + \frac{\rho_0 |u_j|}{|z_j|^{\gamma_1}} I(\betastar_j = 0) \right).
\end{align}
Therefore, we obtain
\begin{align}
    \sqrt{m} (\betahat_n^\# - \betastar) = \argmin_u V_n(u) \overset{d}{\to} \argmin_u V(u) = 
    \begin{dcases}
        0 & \text{for}~ j \in S^c,\\
        z_j & \text{for}~ j \in S.
    \end{dcases}
\end{align}

Case II.
Let $l = \sqrt{n}$.
Then, \eqref{eq:adaptrlasso-zbeta-zbetastar} reduces to
\begin{align}
    &Z_n^\# (\beta) - Z_n^\# (\betastar)\\
    =& \frac{1}{n} u^\top \left( \frac{1}{n} X^\top X \right) u 
    - \frac{2}{n} u^\top \left( \frac{1}{\sqrt{n}} X^\top \varepsilon \right) \\
    +& \frac{\sqrt{m^{\gamma_1}}\lambda_n}{n\sqrt{n}} \sum_j \frac{1}{\left| z_j + \sqrt{m} \betastar_j \right|^{\gamma_1}} \left( \left| u_j + \sqrt{n} \betastar_j \right| - \left| \sqrt{n} \betastar_j \right| \right) \nonumber \\
    +& \frac{\eta_n}{n\sqrt{n m^{\gamma_2}}} \sum_j \left| z_j + \sqrt{m} \betastar_j \right|^{\gamma_2} \left( \left| u_j - \sqrt{\frac{n}{m}} z_j \right| -  \left| \sqrt{\frac{n}{m}} z_j \right| \right) \nonumber \\
    =& \frac{1}{n} u^\top \left( \frac{1}{n} X^\top X \right) u 
    - \frac{2}{n} u^\top \left( \frac{1}{\sqrt{n}} X^\top \varepsilon \right) \\
    +& \frac{\lambda_n}{n\sqrt{n}} \sum_j \left( \frac{ \left| u_j + \sqrt{n} \betastar_j \right| - \left| \sqrt{n} \betastar_j \right|}{\left| \betastar_j + \frac{z_j}{\sqrt{m}} \right|^{\gamma_1}} I(\betastar_j \neq 0) + \frac{\sqrt{m^{\gamma_1}} |u_j|}{\left| z_j \right|^{\gamma_1}} I(\betastar_j = 0) \right) \nonumber \\
    +& \frac{\eta_n}{n\sqrt{n}} \sum_j \left( \left| \betastar_j + \frac{z_j}{\sqrt{m}} \right|^{\gamma_2} I(\betastar_j \neq 0) + \frac{\left| z_j \right|^{\gamma_2}}{\sqrt{m^{\gamma_2}}} I(\betastar_j = 0) \right) \left( \left| u_j - \sqrt{\frac{n}{m}} z_j \right| - \left| \sqrt{\frac{n}{m}} z_j \right| \right). \nonumber
\end{align}

Consider the case (iii) where 
\begin{align}
    \frac{1}{n} \gtrsim \frac{\sqrt{m^{\gamma_1}}\lambda_n}{n\sqrt{n}}
    ~\text{and}~
    \frac{1}{n} \gtrsim \frac{\eta_n}{n\sqrt{n}},
\end{align}
that is,
\begin{align}
    \frac{\sqrt{m^{\gamma_1}}\lambda_n}{\sqrt{n}} \to \lambda_1 \geq 0
    ~\text{and}~
    \frac{\eta_n}{\sqrt{n}} \to \eta_0 \geq 0.
\end{align}
Let $V_n(u) := n (Z_n^\#(\beta) - Z_n^\# (\betastar))$, then we have
\begin{align}
    \label{eq:adaptrlasso-vn-n}
    V_n(u)
    =& u^\top \left( \frac{1}{n} X^\top X \right) u 
    - 2 u^\top \left( \frac{1}{\sqrt{n}} X^\top \varepsilon \right) \\
    +& \frac{\lambda_n}{\sqrt{n}} \sum_j \left( \frac{ \left| u_j + \sqrt{n} \betastar_j \right| - \left| \sqrt{n} \betastar_j \right|}{\left| \betastar_j + \frac{z_j}{\sqrt{m}} \right|^{\gamma_1}} I(\betastar_j \neq 0) + \frac{\sqrt{m^{\gamma_1}} |u_j|}{\left| z_j \right|^{\gamma_1}} I(\betastar_j = 0) \right) \nonumber \\
    +& \frac{\eta_n}{\sqrt{n}} \sum_j \left( \left| \betastar_j + \frac{z_j}{\sqrt{m}} \right|^{\gamma_2} I(\betastar_j \neq 0) + \frac{\left| z_j \right|^{\gamma_2}}{\sqrt{m^{\gamma_2}}} I(\betastar_j = 0) \right) \left( \left| u_j - \sqrt{\frac{n}{m}} z_j \right| - \left| \sqrt{\frac{n}{m}} z_j \right| \right). \nonumber
\end{align}
Let $V(u) := \lim_{n \to \infty} V_n(u)$, then we have
\begin{align}
    V(u) =& u^\top C u - 2 u^\top W \\
    +& \sum_j \left( \frac{\lambda_1 | u_j |}{|z_j|^{\gamma_1}} I(\betastar_j = 0)
    + \eta_0 \left| \betastar_j \right|^{\gamma_2} \left( \left| u_j - \sqrt{r_0} z_j \right| - \left| \sqrt{r_0} z_j \right| \right) I(\betastar_j \neq 0) \right).
\end{align}
Therefore, we obtain
\begin{align}
    &\sqrt{n} (\betahat_n^\# - \betastar) = \argmin_u V_n(u) \overset{d}{\to} \argmin_u V(u) \\
    =&\argmin_{u} \left\{
    u^\top C u - 2 u^\top W
    + \sum_j \left( \frac{\lambda_1 | u_j |}{|z_j|^{\gamma_1}} I(\betastar_j = 0)
    + \eta_0 \left| \betastar_j \right|^{\gamma_2} \left| u_j - \sqrt{r_0} z_j \right| I(\betastar_j \neq 0) \right) \right\}.
\end{align}

Consider the case (iv) where 
\begin{align}
    \frac{1}{n} \gtrsim \frac{\sqrt{m^{\gamma_1}}\lambda_n}{n\sqrt{n}}
    ~\text{and}~
    \frac{\eta_n}{n\sqrt{n}} \gg \frac{1}{n} \gtrsim \frac{\eta_n}{n\sqrt{n m^{\gamma_2}}},
\end{align}
that is,
\begin{align}
    \frac{\sqrt{m^{\gamma_1}}\lambda_n}{\sqrt{n}} \to \lambda_1 \geq 0,~
    \frac{\eta_n}{\sqrt{n}} \to \infty,
    ~\text{and}~
    \frac{\eta_n}{\sqrt{n m^{\gamma_2}}} \to \eta_1 \geq 0.
\end{align}
Let $V_n(u) := n (Z_n^\#(\beta) - Z_n^\# (\betastar)) + (\text{term irrelevant to}~ \beta)$
and $V(u) := \lim_{n \to \infty} V_n(u)$.
Then, we have \eqref{eq:adaptrlasso-vn-n} and
\begin{align}
    V(u) = 
    \begin{dcases}
        u^\top C u - 2 u^\top W + \sum_j \left( \frac{\lambda_1}{|z_j|^{\gamma_1}} \left| u_j \right| + \eta_1 |z_j|^{\gamma_2} |u_j - \sqrt{r_0} z_j| \right) I(\betastar_j = 0) & \text{for}~ u_S = r_0 z_S,\\
        \infty & \text{otherwise}.
    \end{dcases}
\end{align}
Therefore, we obtain
\begin{align}
    &\sqrt{n} (\betahat_n^\# - \betastar) = \argmin_u V_n(u) \\
    \overset{d}{\to}& \argmin_u V(u) =
    \argmin_{u \in \mathcal{U}} \left\{ u^\top C u - 2 u^\top W + \sum_j \left( \frac{\lambda_1}{|z_j|^{\gamma_1}} \left| u_j \right| + \eta_1 |z_j|^{\gamma_2} |u_j - \sqrt{r_0} z_j| \right) \right\},
\end{align}
\begin{align}
    \mathcal{U} := \left\{ u ~\middle|~ u_S = r_0 z_S \right\}.
\end{align}

Consider the case (v) where 
\begin{align}
    \frac{\sqrt{m^{\gamma_1}}\lambda_n}{n\sqrt{n}} \gg \frac{1}{n} \gtrsim \frac{\lambda_n}{n\sqrt{n}}
    ~\text{and}~
    \frac{1}{n} \gtrsim \frac{\eta_n}{n\sqrt{n}},
\end{align}
that is,
\begin{align}
    \frac{\sqrt{m^{\gamma_1}}\lambda_n}{\sqrt{n}} \to \infty,
    \frac{\lambda_n}{\sqrt{n}} \to \lambda_0 \geq 0,
    ~\text{and}~
    \frac{\eta_n}{\sqrt{n}} \to \eta_0 \geq 0.
\end{align}
Let $V_n(u) := n (Z_n^\#(\beta) - Z_n^\# (\betastar))$
and $V(u) := \lim_{n \to \infty} V_n(u)$.
Then, we have \eqref{eq:adaptrlasso-vn-n} and
\begin{align}
    V(u) = 
    \begin{dcases}
        u^\top C u - 2 u^\top W + \sum_j \left( \lambda_0 u_j \sgn(\betastar_j) / |\betastar_j|^{\gamma_1} + \eta_0 \left| \betastar_j \right|^{\gamma_2} \left| u_j - \sqrt{r_0} z_j \right| \right) & \text{for}~ u_{S^c} = 0,\\
        \infty & \text{otherwise}.
    \end{dcases}
\end{align}
Therefore, we obtain
\begin{align}
    &\sqrt{n} (\betahat_n^\# - \betastar)  = \argmin_u V_n(u) \\
    \overset{d}{\to}& \argmin_u V(u) =
    \argmin_{u \in \mathcal{U}} \left\{ u^\top C u - 2 u^\top W + \sum_j \left( \frac{\lambda_0 u_j \sgn(\betastar_j)}{|\betastar_j|^{\gamma_1}} + \eta_0 \left| \betastar_j \right|^{\gamma_2} \left| u_j - \sqrt{r_0} z_j \right| \right) \right\},
\end{align}
\begin{align}
    \mathcal{U} := \left\{ u ~\middle|~ u_{S^c} = 0 \right\}.
\end{align}

Case III.
Let $l = n / \lambda_n$.
Suppose that $n / \lambda_n \to \infty$
Then, \eqref{eq:adaptrlasso-zbeta-zbetastar} reduces to
\begin{align}
    &Z_n^\# (\beta) - Z_n^\# (\betastar)\\
    =& \frac{\lambda_n^2}{n^2} u^\top \left( \frac{1}{n} X^\top X \right) u 
    - \frac{2\lambda_n}{n\sqrt{n}} u^\top \left( \frac{1}{\sqrt{n}} X^\top \varepsilon \right) \\
    +& \frac{\sqrt{m^{\gamma_1}}\lambda_n^2}{n^2} \sum_j \frac{1}{\left| z_j + \sqrt{m} \betastar_j \right|^{\gamma_1}} \left( \left| u_j + \frac{n}{\lambda_n} \betastar_j \right| - \left| \frac{n}{\lambda_n} \betastar_j \right| \right) \nonumber \\
    +& \frac{\lambda_n \eta_n}{n^2 \sqrt{m^{\gamma_2}}} \sum_j \left| z_j + \sqrt{m} \betastar_j \right|^{\gamma_2} \left( \left| u_j- \frac{n}{\sqrt{m}\lambda_n} z_j \right| -  \left| \frac{n}{\sqrt{m}\lambda_n} z_j \right| \right) \nonumber \\
    =& \frac{\lambda_n^2}{n^2} u^\top \left( \frac{1}{n} X^\top X \right) u 
    - \frac{2\lambda_n}{n\sqrt{n}} u^\top \left( \frac{1}{\sqrt{n}} X^\top \varepsilon \right) \\
    +& \frac{\lambda_n^2}{n^2} \sum_j \left( \frac{\left| u_j + \frac{n}{\lambda_n} \betastar_j \right| - \left| \frac{n}{\lambda_n} \betastar_j \right|}{\left| \betastar_j + \frac{z_j}{\sqrt{m}} \right|^{\gamma_1}} I(\betastar_j \neq 0) + \frac{\sqrt{m^{\gamma_1}} |u_j|}{\left| z_j \right|^{\gamma_1}} I(\betastar_j = 0) \right) \nonumber \\
    +& \frac{\lambda_n \eta_n}{n^2} \sum_j \left( \left| \betastar_j + \frac{z_j}{\sqrt{m}} \right|^{\gamma_2} I(\betastar_j \neq 0) + \frac{\left| z_j \right|^{\gamma_2}}{\sqrt{m^{\gamma_2}}} I(\betastar_j = 0) \right) \left( \left| u_j - \frac{n}{\sqrt{m}\lambda_n} z_j \right| -  \left| \frac{n}{\sqrt{m}\lambda_n} z_j \right| \right). \nonumber
\end{align}

Consider the case (vi) where
\begin{align}
    \frac{\lambda_n^2}{n^2} \gg \frac{\lambda_n}{n\sqrt{n}}
    ~\text{and}~
    \frac{\lambda_n^2}{n^2} \gg \frac{\lambda_n \eta_n}{n^2},
\end{align}
that is,
\begin{align}
    \frac{\lambda_n}{\sqrt{n}} \to \infty,~
    \frac{\lambda_n}{n} \to 0,
    ~\text{and}~
    \frac{\lambda_n}{\eta_n} \to \infty.
\end{align}
Let $V_n(u) := (n^2 / \lambda_n^2) (Z_n^\#(\beta) - Z_n^\# (\betastar))$, then we have
\begin{align}
    V_n(u)
    =& u^\top \left( \frac{1}{n} X^\top X \right) u 
    - \frac{2\sqrt{n}}{\lambda_n} u^\top \left( \frac{1}{\sqrt{n}} X^\top \varepsilon \right) \\
    +& \sum_j \left( \frac{\left| u_j + \frac{n}{\lambda_n} \betastar_j \right| - \left| \frac{n}{\lambda_n} \betastar_j \right|}{\left| \betastar_j + \frac{z_j}{\sqrt{m}} \right|^{\gamma_1}} I(\betastar_j \neq 0) + \frac{\sqrt{m^{\gamma_1}} |u_j|}{\left| z_j \right|^{\gamma_1}} I(\betastar_j = 0) \right) \nonumber \\
    +& \frac{\eta_n}{\lambda_n} \sum_j \left( \left| \betastar_j + \frac{z_j}{\sqrt{m}} \right|^{\gamma_2} I(\betastar_j \neq 0) + \frac{\left| z_j \right|^{\gamma_2}}{\sqrt{m^{\gamma_2}}} I(\betastar_j = 0) \right) \left( \left| u_j - \frac{n}{\sqrt{m}\lambda_n} z_j \right| -  \left| \frac{n}{\sqrt{m}\lambda_n} z_j \right| \right). \nonumber
\end{align}
Let $V(u) := \lim_{n \to \infty} V_n(u)$, then we have
\begin{align}
    V(u) = 
    \begin{dcases}
        u^\top C u + \sum_j \frac{u_j \sgn(\betastar_j) }{|\betastar_j|^{\gamma_1}} I(\betastar_j \neq 0) & \text{for}~ u_{S^c} = 0,\\
        \infty & \text{otherwise}.
    \end{dcases}
\end{align}
Therefore, we obtain
\begin{align}
    \frac{n}{\lambda_n} (\betahat_n^\# - \betastar)  = \argmin_u V_n(u) \overset{d}{\to} \argmin_u V(u) =
    \argmin_{u \in \mathcal{U}} \left\{ u^\top C u + \sum_j \frac{\sgn(\betastar_j)}{|\betastar_j|^{\gamma_1}} u_j \right\},
\end{align}
\begin{align}
    \mathcal{U} := \left\{ u ~\middle|~ u_{S^c} = 0 \right\}.
\end{align}

\subsubsection{Proof of Theorem \ref{theorem:adaptrlasso-active-selection-consistency}}
\label{proof:adaptrlasso-active-selection-consistency}

Suppose that $\betatilde$ is a $\sqrt{m}$-consistent estimator.
Suppose that $\betahat$ is $l$-consistent and let $\hat{u} := l (\betahat - \betastar)$ where $l = l(n, m, \lambda_n)$ is a certain function as defined later.
Then, we have
\begin{align}
    \forall j \in S, ~ P(j \in \supp(\betahat)) \to 1.
\end{align}
Consider the event $\betastar_j \in S^c$ and $\betahat_j \neq 0$.
By the KKT conditions, for $\betastar_j = 0$ and $\betahat_j \neq \betatilde_j$,
\begin{align}
    2 \left( \frac{1}{n} \x_j^\top X \right) \frac{\sqrt{n}}{l} \hat{u}
    - \frac{2}{\sqrt{n}} \x_j^\top \varepsilon 
    + \frac{\sqrt{m^{\gamma_1}}\lambda_n}{\sqrt{n}} \frac{\sgn(\betahat_j)}{|\sqrt{m} \betatilde_j|^{\gamma_1}}  
    + \frac{\eta_n}{\sqrt{n m^{\gamma_2}}} \left| \sqrt{m} \betatilde_j \right|^{\gamma_2} \sgn( \betahat_j - \betatilde_j ) = 0,
\end{align}
and for $\betastar_j = 0$ and $\betahat_j = \betatilde_j$,
\begin{align}
    \left| 
    2 \left( \frac{1}{n} \x_j^\top X \right) \frac{\sqrt{n}}{l} \hat{u}
    - \frac{2}{\sqrt{n}} \x_j^\top \varepsilon 
    + \frac{\sqrt{m^{\gamma_1}}\lambda_n}{\sqrt{n}} \frac{\sgn(\betahat_j) }{|\sqrt{m} \betatilde_j|^{\gamma_1}} 
    \right|
    \leq
    \frac{\eta_n}{\sqrt{n m^{\gamma_2}}} \left| \sqrt{m} \betatilde_j \right|^{\gamma_2}.
\end{align}
Suppose that $\betahat$ is $l$-consistent and that
$\sqrt{m^{\gamma_1}}\lambda_n / \sqrt{n} \gg 1$, 
$\sqrt{m^{\gamma_1}}\lambda_n / \sqrt{n} \gg \sqrt{n} / l$, 
and
$\sqrt{m^{\gamma_1}}\lambda_n / \sqrt{n} \gg \eta_n / \sqrt{n m^{\gamma_2}}$.
For $l = \sqrt{m}, \sqrt{n}, n/\lambda_n$, these conditions reduce to the conditions where
$\betahat$ is $l$-consistent, 
$\sqrt{m^{\gamma_1}} \lambda_n / \sqrt{n} \to \infty$,
and
$\sqrt{m^{\gamma_1 + \gamma_2}}\lambda_n / \eta_n \to \infty$.
Then, in both $\betahat_j \neq \betatilde_j$ and $\betahat_j = \betatilde_j$ cases, the term including $\sqrt{m^{\gamma_1}/n} \lambda_n$ diverges to infinity faster compared to the rest terms.
Therefore, we have
\begin{align}
    \forall j \in S^c, ~ \lim_{n \to \infty} P(j \in \supp(\betahat)) = 0.
\end{align}
This concludes
\begin{align}
    \lim_{n \to \infty} P(\hat{S}_n = S) = 1.
\end{align}

\subsubsection{Proof of Theorem \ref{theorem:adaptrlasso-varying-selection-consistency}}
\label{proof:adaptrlasso-varying-selection-consistency}

Suppose that $\betatilde$ is a $\sqrt{m}$-consistent estimator.
Consider the event where $\betastar_j \neq 0$ and $\betahat_j \neq \betatilde_j$.
Suppose that $\betahat$ is $l$-consistent and let $\hat{u} := l (\betahat - \betastar)$ where $l = l(n, m, \lambda_n)$ is a certain function as defined later.
By the KKT conditions, for $\betastar_j \neq 0$ and $\betahat_j \neq 0$,
\begin{align}
    2 \left( \frac{1}{n} \x_j^\top X \right) \frac{\sqrt{n}}{l} \hat{u}
    - \frac{2}{\sqrt{n}} \x_j^\top \varepsilon 
    + \frac{\lambda_n}{\sqrt{n}} \frac{\sgn(\betahat_j)}{|\betatilde_j|^{\gamma_1}}  
    + \frac{\eta_n}{\sqrt{n}} \left| \betatilde_j \right|^{\gamma_2} \sgn( \betahat_j - \betatilde_j ) = 0,
\end{align}
and for $\betastar_j \neq 0$ and $\betahat_j = 0$,
\begin{align}
    \left| 
    2 \left( \frac{1}{n} \x_j^\top X \right) \frac{\sqrt{n}}{l} \hat{u}
    - \frac{2}{\sqrt{n}} \x_j^\top \varepsilon 
    + \frac{\eta_n}{\sqrt{n}} \left| \betatilde_j \right|^{\gamma_2} \sgn( \betahat_j - \betatilde_j )
    \right|
    \leq
    \frac{\lambda_n}{\sqrt{n}} \frac{1}{|\betatilde_j|^{\gamma_1}}.
\end{align}
Suppose that $\betahat$ is $l$-consistent and that
$\eta_n / \sqrt{n} \gg 1$,
$\eta_n / \sqrt{n} \gg \sqrt{n} / l$,
and
$\eta_n / \sqrt{n} \gg \lambda_n / \sqrt{n}$.
For $l = \sqrt{m}, \sqrt{n}$, these conditions reduce to the conditions where
$\betahat$ is $l$-consistent,
$\eta_n / \sqrt{n} \to \infty$,
and
$\eta_n / \lambda_n \to \infty$.
Then, in both $\betahat_j \neq 0$ and $\betahat_j = 0$ cases, the term including $\eta_n / \sqrt{n}$ diverges to infinity faster compared to the rest terms.
Therefore, we have
\begin{align}
    \forall j \in S, ~ \lim_{n \to \infty} P(\betahat_j \neq \betatilde_j) = 0.
\end{align}
This concludes
\begin{align}
    \forall j \in S, ~ \lim_{n \to \infty} P(\betahat_j = \betatilde_j) = 1.
\end{align}

\subsection{Proofs for Transfer Lasso with Deterministic Source Parameters}

\subsubsection{Proof of Theorem~\ref{trlasso-variable-selection}}
\label{proof:trlasso-variable-selection}

For $\betatilde_j = 0$, we consider the event $\betahat_j \neq 0$.
Then, we have by KKT conditions
\begin{align}
    2 \left(\frac{1}{n} \x_j^\top X \right)  \sqrt{n} (\betastar - \betahat) + 2 \frac{1}{\sqrt{n}} \x_j^\top \varepsilon
    + \frac{\lambda_n + \eta_n}{\sqrt{n}} \sgn(\betahat_j) = 0
\end{align}
The first and second terms converge to some truncated Gaussian-mixture distribution, and the third term diverges to infinity as $(\lambda_n + \eta_n) / \sqrt{n} \rightarrow \infty$, which is an contradiction.
Hence, we have for $\betatilde_j = 0$
\begin{align}
    \lim_{n\rightarrow \infty} P\left( \betahat_j = 0 \right) = 1
\end{align}

For $\betatilde_j \neq 0$, we consider the event $\betahat_j \neq 0$ and $\betahat_j \neq \betatilde_j$.
Then, we have by KKT conditions,
\begin{align}
    2 \left(\frac{1}{n} \x_j^\top X \right)  \sqrt{n} (\betastar - \betahat) + 2 \frac{1}{\sqrt{n}} \x_j^\top \varepsilon
    + \frac{\lambda_n}{\sqrt{n}} \sgn(\betahat_j) + \frac{\eta_n}{\sqrt{n}} \sgn(\betahat_j - \betatilde_j) = 0
\end{align}
The third and fourth terms diverge to infinity as $(\lambda_n + \eta_n) / \sqrt{n} \rightarrow \infty$ if $\sgn(\betahat_j) = \sgn(\betahat_j - \betatilde_j)$, which induces an contradiction.
Hence, we have for $\betatilde_j \neq 0$
\begin{align}
    \lim_{n\rightarrow \infty} P\left( \min \left\{ 0, \betatilde_j \right\} \leq \betahat_j \leq \max \left\{ 0, \betatilde_j \right\} \right) = 1
\end{align}

\subsubsection{Proof of Theorem \ref{trlasso-distribution-sign-consistency}}
\label{proof:trlasso-distribution-sign-consistency}

Asymptotic distribution:
Let $u := \sqrt{n} (\beta - \betastar)$. 
We have
\begin{align}
    Z_n^\mathcal{T}(\beta; \betatilde, \lambda_n, \eta_n) 
    :&= \frac{1}{n} \| y - X \beta \|_2^2 + \frac{\lambda_n}{n} \sum_j |\beta_j| + \frac{\eta_n}{n} |\beta_j - \betatilde_j|\\
    &= \frac{1}{n}\left\| \varepsilon - \frac{1}{\sqrt{n}} X u \right\|_2^2 + \frac{\lambda_n}{n} \sum_j \left| \betastar_j + \frac{u_j}{\sqrt{n}} \right| + \frac{\eta_n}{n} \sum_j \left| \betastar_j - \betatilde_j + \frac{u_j}{\sqrt{n}} \right|
\end{align}
and
\begin{align}
    &Z_n^\mathcal{T}(\beta; \betatilde, \lambda_n, \eta_n) - Z_n^\mathcal{T}(\betastar; \betatilde, \lambda_n, \eta_n)\\
    =& \left\| \varepsilon - \frac{1}{\sqrt{n}} X u \right\|_2^2 - \| \varepsilon \|_2^2 \\
    &+ \lambda_n \sum_j \left( \left| \betastar_j + \frac{u_j}{\sqrt{n}} \right| - \left|\betastar_j\right| \right) + \eta_n \sum_j \left( \left| \betastar_j - \betatilde_j + \frac{u_j}{\sqrt{n}} \right| - \left|\betastar_j - \betatilde_j\right| \right)\\
    = & u^\top \left( \frac{1}{n} X^\top X \right) u -2 u^\top \left(\frac{1}{\sqrt{n}}  X^\top \varepsilon \right)\\
    &+ \lambda_n \sum_j \left( \left| \betastar_j + \frac{u_j}{\sqrt{n}} \right| - \left|\betastar_j\right| \right) + \eta_n \sum_j \left( \left| \betastar_j - \betatilde_j + \frac{u_j}{\sqrt{n}} \right| - \left|\betastar_j - \betatilde_j\right| \right)
\end{align}
The first and second terms are the same as \eqref{uw} and \eqref{ucu}.
We consider the third and fourth terms.
\begin{align}
    \lambda_n \left( \left| \betastar_j + \frac{u_j}{\sqrt{n}} \right| - \left|\betastar_j\right| \right)
    = \begin{cases}
        \frac{\lambda_n}{\sqrt{n}} |u_j| & \text{if} ~ \betastar_j = 0\\
        \frac{\lambda_n}{\sqrt{n}} u_j \sgn(\betastar_j) & \text{if} ~ \betastar_j \neq 0 ~ \text{and} ~ \betastar_j (\betastar_j + \frac{u_j}{\sqrt{n}}) \geq 0\\
        -\frac{\lambda_n}{\sqrt{n}} u_j \sgn(\betastar_j) - 2 \lambda_n |\betastar_j| & \text{if} ~ \betastar_j \neq 0 ~ \text{and} ~ \betastar_j (\betastar_j + \frac{u_j}{\sqrt{n}}) < 0
    \end{cases}
    \label{V-lambda}
\end{align}
\begin{align}
    &\eta_n \left( \left| \betastar_j - \betatilde_j + \frac{u_j}{\sqrt{n}} \right| - \left|\betastar_j - \betatilde_j\right| \right)\\
    =& \begin{cases}
        \frac{\eta_n}{\sqrt{n}} |u_j| & \text{if} ~ \betastar_j = \betatilde_j\\
        \frac{\eta_n}{\sqrt{n}} u_j \sgn(\betastar_j - \betatilde_j) & \text{if} ~ \betastar_j \neq \betatilde_j  ~ \text{and} ~ (\betastar_j - \betatilde_j) (\betastar_j  - \betatilde_j + \frac{u_j}{\sqrt{n}}) \geq 0\\
        -\frac{\eta_n}{\sqrt{n}} u_j \sgn(\betastar_j - \betatilde_j) - 2 \eta_n |\betastar_j - \betatilde_j| & \text{if} ~ \betastar_j \neq \betatilde_j ~ \text{and} ~ (\betastar_j - \betatilde_j) (\betastar_j  - \betatilde_j + \frac{u_j}{\sqrt{n}}) < 0
    \end{cases}
    \label{V-gamma}
\end{align}
For large $n$ such that ($n > u_j^2 / \beta_j^{*2} $ for $\forall j \in S$) and ($n > u_j^2 / (\betastar_j - \betatilde_j)^2$ for $\forall j$ such that $\betastar_j \neq \betatilde_j$), we have
\begin{align}
    \mathcal{R}_j :=
    &\lambda_n \left( \left| \betastar_j + \frac{u_j}{\sqrt{n}} \right| - \left|\betastar_j\right| \right)
    + \eta_n \left( \left| \betastar_j - \betatilde_j + \frac{u_j}{\sqrt{n}} \right| - \left|\betastar_j - \betatilde_j\right| \right)\\
    =& \begin{cases}
        \frac{\lambda_n + \eta_n}{\sqrt{n}} |u_j| & 
        \text{if} ~ \betastar_j = 0, \betatilde_j = 0\\
        \frac{\lambda_n}{\sqrt{n}} |u_j| - \frac{\eta_n}{\sqrt{n}} u_j \sgn(\betatilde_j)  &
        \text{if} ~ \betastar_j = 0, \betatilde_j \neq 0\\
        \frac{\lambda_n}{\sqrt{n}} u_j \sgn(\betastar_j) + \frac{\eta_n}{\sqrt{n}} |u_j| & \text{if} ~ \betastar_j \neq 0, \betatilde_j = \betastar_j\\
        \frac{\lambda_n}{\sqrt{n}} u_j \sgn(\betastar_j) + \frac{\eta_n}{\sqrt{n}} u_j \sgn(\betastar_j - \betatilde_j) & \text{if} ~ \betastar_j \neq 0, \betatilde_j \neq \betastar_j
    \end{cases}
\end{align}
Assume $(\lambda_n + \eta_n) / \sqrt{n} \rightarrow \delta_1$ and $(\lambda_n - \eta_n) / \sqrt{n} \rightarrow \delta_2$.
For $\betastar_j = 0$ and $\betatilde_j = 0$,
\begin{align}
    \mathcal{R}_j \rightarrow \begin{cases}
        0 & \text{if} ~ u_j = 0\\
        \delta_1 |u_j| & \text{if} ~ u_j \neq 0
    \end{cases}
\end{align}
For $\betastar_j = 0, \betatilde_j \neq 0$,
\begin{align}
    \mathcal{R}_j \rightarrow \begin{cases}
        0 & \text{if} ~ u_j = 0\\
        \delta_2 & \text{if} ~ u_j \betatilde_j > 0\\
        \delta_1 |u_j| & \text{if} ~  u_j \betatilde_j < 0
    \end{cases}
\end{align}
For $\betastar_j \neq 0, \betatilde_j = \betastar_j$,
\begin{align}
    \mathcal{R}_j \rightarrow \begin{cases}
        0 & \text{if} ~ u_j = 0\\
        \delta_1 & \text{if} ~ u_j \betastar_j > 0\\
        -\delta_2 |u_j| & \text{if} ~  u_j \betastar_j < 0
    \end{cases}
\end{align}
For $\betastar_j \neq 0, \betatilde_j \neq \betastar_j$,
\begin{align}
    \mathcal{R}_j \rightarrow \begin{cases}
        0 & \text{if} ~ u_j = 0\\
        \delta_2 |u_j| & \text{if} ~ \sgn(\betastar_j) \neq \sgn(\betastar_j - \betatilde_j) ~\text{and}~ u_j \sgn(\betastar_j) > 0\\
        -\delta_2 |u_j| & \text{if} ~ \sgn(\betastar_j) \neq \sgn(\betastar_j - \betatilde_j) ~\text{and}~ u_j \sgn(\betastar_j) < 0\\
        \delta_1 |u_j| & \text{if} ~  \sgn(\betastar_j) = \sgn(\betastar_j - \betatilde_j) ~ \text{and} ~ u_j \sgn(\betastar_j) > 0\\
        -\delta_1 |u_j| & \text{if} ~ \sgn(\betastar_j) = \sgn(\betastar_j - \betatilde_j) ~ \text{and} ~ u_j \sgn(\betastar_j) < 0
    \end{cases}
\end{align}
Furthermore, we assume $\delta_1 = \infty$ and $\delta_2 = 0$.
Then, for $\betastar_j = 0$ and $\betatilde_j = 0$,
\begin{align}
    \mathcal{R}_j \rightarrow \begin{cases}
        0 & \text{if} ~ u_j = 0\\
        \infty & \text{if} ~ u_j \neq 0
    \end{cases}
\end{align}
For $\betastar_j = 0, \betatilde_j \neq 0$,
\begin{align}
    \mathcal{R}_j \rightarrow \begin{cases}
        0 & \text{if} ~ u_j \betatilde_j \geq 0\\
        \infty & \text{if} ~  u_j \betatilde_j < 0
    \end{cases}
\end{align}
For $\betastar_j \neq 0, \betatilde_j = \betastar_j$,
\begin{align}
    \mathcal{R}_j \rightarrow \begin{cases}
        0 & \text{if} ~ u_j \betastar_j \leq 0\\
        \infty & \text{if} ~  u_j \betastar_j > 0
    \end{cases}
\end{align}
For $\betastar_j \neq 0, \betatilde_j \neq \betastar_j$,
\begin{align}
    \mathcal{R}_j \rightarrow \begin{cases}
        0 & \text{if} ~ \sgn(\betastar_j) \neq \sgn(\betastar_j - \betatilde_j) ~ \text{or} ~ u_j = 0\\
        \infty & \text{if} ~  \sgn(\betastar_j) = \sgn(\betastar_j - \betatilde_j) ~ \text{and} ~ u_j \sgn(\betastar_j) > 0\\
        -\infty & \text{if} ~ \sgn(\betastar_j) = \sgn(\betastar_j - \betatilde_j) ~ \text{and} ~ u_j \sgn(\betastar_j) < 0
    \end{cases}
\end{align}
Overall, we have
\begin{align}
    \mathcal{R}_j \rightarrow \begin{cases}
        0 & \text{if} ~ (\betastar_j = 0 ~\text{and}~ \betatilde_j=0 ~\text{and}~ u_j=0)\\
        & \text{or} ~ (\betastar_j=0 ~\text{and}~ \betatilde_j\neq0 ~\text{and}~ u_j \betatilde_j \geq 0)\\
        & \text{or} ~ (\betastar_j\neq0 ~\text{and}~ \betatilde_j=\betastar_j ~\text{and}~ u_j \betastar_j \leq 0)\\
        & \text{or} ~ (\betastar_j\neq0 ~\text{and}~ \betatilde_j\neq\betastar_j ~\text{and}~ (\sgn(\betastar_j) \neq \sgn(\betastar_j - \betatilde_j) ~\text{or}~ u_j = 0))\\
        \infty & \text{if} ~ (\betastar_j = 0 ~\text{and}~ \betatilde_j=0 ~\text{and}~ u_j\neq0)\\
        & \text{or} ~ (\betastar_j=0 ~\text{and}~ \betatilde_j\neq0 ~\text{and}~ u_j \betatilde_j < 0)\\
        & \text{or} ~ (\betastar_j\neq0 ~\text{and}~ \betatilde_j=\betastar_j ~\text{and}~ u_j \betastar_j > 0)\\
        & \text{or} ~ (\betastar_j\neq0 ~\text{and}~ \betatilde_j\neq\betastar_j ~\text{and}~ \sgn(\betastar_j) = \sgn(\betastar_j - \betatilde_j) ~\text{and}~ u_j \sgn(\betastar_j) > 0)\\
        - \infty & \text{if} ~ (\betastar_j\neq0 ~\text{and}~ \betatilde_j\neq\betastar_j ~\text{and}~ \sgn(\betastar_j) = \sgn(\betastar_j - \betatilde_j) ~ \text{and} ~ u_j \sgn(\betastar_j) < 0)
    \end{cases}
\end{align}
If we assume $\sgn(\betastar_j) \neq \sgn(\betastar_j - \betatilde_j)$ for $\forall j$ such that $\betastar_j \neq 0$ and $\betatilde_j \neq \betastar_j$, we exclude the case of $\mathcal{R} \rightarrow - \infty$, and
\begin{align}
    V_n(u) \rightarrow
    V(u) := \begin{cases}
        -2u^\top W + u^\top C u & \text{if} ~ (u_j=0 ~\text{for}~ \forall j ~\text{s.t.}~ \betastar_j = 0 ~\text{and}~ \betatilde_j=0)\\
        & \text{and} ~ (u_j \betatilde_j \geq 0 ~\text{for} ~\forall j ~\text{s.t.}~ \betastar_j=0 ~\text{and}~ \betatilde_j\neq0)\\
        & \text{and} ~ (u_j \betastar_j \leq 0 ~\text{for}~ \forall j ~\text{s.t.}~ \betastar_j\neq0 ~\text{and}~ \betatilde_j=\betastar_j)\\
        \infty & \text{otherwise}.
    \end{cases}
\end{align}
Since $V_n(u)$ is convex and $V(u)$ has a unique minimum, in the same manner of the proof of Adaptive Lasso, we have
\begin{align}
    \sqrt{n}(\betahat - \betastar) \overset{d}{\rightarrow}
    \argmin_{u\in \mathcal{U}} & -2u^\top W + u^\top C u\\
    \mathcal{U} :=& \left\{ u\in \mathbb{R}^p \left|
    \begin{array}{l}
        u_j=0 ~\text{for}~ \forall j ~\text{s.t.}~ \betastar_j = 0 ~\text{and}~ \betatilde_j=0\\
        u_j \betatilde_j \geq 0 ~\text{for} ~\forall j ~\text{s.t.}~ \betastar_j=0 ~\text{and}~ \betatilde_j\neq0\\
        u_j \betastar_j \leq 0 ~\text{for}~ \forall j ~\text{s.t.}~ \betastar_j\neq0 ~\text{and}~ \betatilde_j=\betastar_j
    \end{array}
    \right. \right\}
\end{align}

Next, we show active/varying variable consistency.
We consider the cases of (i)~$j \in S \cap T$, (ii)~$j \in S \cap T^c$, (iii)~$j \in S^c \cap T$, and (iv)~$j \in S^c \cap T^c$.

(i)~For $j \in S \cap T$, we have $\betatilde_j \neq 0$ because of $\sgn(\betatilde_j - \betastar_j) = \sgn(\betastar_j)$.
By KKT condition, we have
\begin{align}
    2 \left(\frac{1}{n} \x_j^\top X \right)  \sqrt{n} (\betastar - \betahat) + 2 \frac{1}{\sqrt{n}} \x_j^\top \varepsilon
    + \frac{\lambda_n}{\sqrt{n}} \sgn(\betahat_j) + \frac{\eta_n}{\sqrt{n}} \sgn(\betahat_j - \betatilde_j) = 0
\end{align}
and 
\begin{align}
    \lim_{n\rightarrow \infty} P\left( \min \left\{ 0, \betatilde_j \right\} \leq \betahat_j \leq \max \left\{ 0, \betatilde_j \right\} \right) = 1.
\end{align}
On the other hand, $\sqrt{n} (\betahat_j - \betastar_j)$ converges to some Gaussian-mixture distribution for $j \in S \cap T$.
Hence, we have $\betahat_j \rightarrow \betastar_j$ and
\begin{align}
    \lim_{n\rightarrow \infty} P\left( \betahat_j = 0 ~\text{or}~ \betahat_j = \betatilde_j \right) = 0
\end{align}
This concludes
\begin{align}
    \lim_{n\rightarrow \infty} P \left(0 < \betahat_j < \betatilde_j ~ \text{or}~ \betatilde_j < \betahat_j < 0 \right) = 1 ~ \text{for} ~ \forall j \in (S \cap T)
\end{align}

(ii)~For $j \in S \cap T^c$, we have $\betastar_j \neq 0$ and $\betastar_j = \betatilde_j$.
By KKT condition, we have
\begin{align}
    2 \left(\frac{1}{n} \x_j^\top X \right)  \sqrt{n} (\betastar - \betahat) + 2 \frac{1}{\sqrt{n}} \x_j^\top \varepsilon
    + \frac{\lambda_n}{\sqrt{n}} \sgn(\betahat_j) + \frac{\eta_n}{\sqrt{n}} \sgn(\betahat_j - \betastar_j) = 0
\end{align}
and 
\begin{align}
    \lim_{n\rightarrow \infty} P\left( \min \left\{ 0, \betastar_j \right\} \leq \betahat_j \leq \max \left\{ 0, \betastar_j \right\} \right) = 1
\end{align}
On the other hand, $\sqrt{n} (\betahat_j - \betastar_j)$ converges to some mixture distribution of truncated Gaussian distribution truncated at zero and delta distribution at zero for $j \in S \cap T^c$.
Hence, we have $\betahat_j \rightarrow \betastar_j \neq 0$ and 
\begin{align}
    \lim_{n\rightarrow \infty} P\left( \betahat_j = 0 \right) = 0
\end{align}
This concludes
\begin{align}
    \lim_{n\rightarrow \infty} P \left(0 < \betahat_j \leq \betatilde_j ~\text{or}~ \betatilde_j \leq \betahat_j < 0 \right) = 1 ~ & \text{for} ~ \forall j \in (S \cap T^c)
\end{align}

(iii)~For $j \in S^c \cap T$, we have $\betastar_j = 0$, $\betastar_j \neq \betatilde_j$, and $\betatilde_j \neq 0$.
By KKT condition, we have
\begin{align}
    2 \left(\frac{1}{n} \x_j^\top X \right)  \sqrt{n} (\betastar - \betahat) + 2 \frac{1}{\sqrt{n}} \x_j^\top \varepsilon
    + \frac{\lambda_n}{\sqrt{n}} \sgn(\betahat_j) + \frac{\eta_n}{\sqrt{n}} \sgn(\betahat_j - \betatilde_j) = 0
\end{align}
and 
\begin{align}
    \lim_{n\rightarrow \infty} P\left( \min \left\{ 0, \betastar_j \right\} \leq \betahat_j \leq \max \left\{ 0, \betastar_j \right\} \right) = 1
\end{align}
On the other hand, $\sqrt{n} \betahat_j$ converges to some mixture distribution of truncated Gaussian distribution truncated at zero and delta distribution at zero for $j \in S \cap T^c$.
Hence, we have $\betahat_j \rightarrow 0 \neq \betatilde_j$ and 
\begin{align}
    \lim_{n\rightarrow \infty} P\left( \betahat_j = \betatilde_j \right) = 0
\end{align}
This concludes
\begin{align}
    \lim_{n\rightarrow \infty} P \left(0 \leq \betahat_j < \betatilde_j ~\text{or}~ \betatilde_j < \betahat_j \leq 0 \right) = 1 ~ & \text{for} ~ \forall j \in (S^c \cap T)
\end{align}

(iv)~For $j \in S^c \cap T^c$, we have $\betatilde_j = \betastar_j = 0$ and
\begin{align}
    2 \left(\frac{1}{n} \x_j^\top X \right)  \sqrt{n} (\betastar - \betahat) + 2 \frac{1}{\sqrt{n}} \x_j^\top \varepsilon
    + \frac{\lambda_n + \eta_n}{\sqrt{n}} \sgn(\betahat_j) = 0
\end{align}
This implies that
\begin{align}
    \lim_{n\rightarrow \infty} P\left( \betahat_j = 0 \right) = 1
\end{align}

\subsubsection{Proof of Theorem \ref{trlasso-oracle}}
\label{proof:trlasso-oracle}

We assume that $\betatilde_j = 0$ for $\forall j$ such that $\betastar_j = 0$, and $\betatilde_j \neq \betastar_j$ for $\forall j$ such that $\betastar_j \neq 0$.
Then, $\mathcal{U} = \{ u ~ | ~ u_{S^c}=0\}$, and we have asymptotic normality, i.e.,
\begin{align}
    \sqrt{n}(\betahat_S - \betastar_S) \overset{d}{\rightarrow} \mathcal{N}(0, \sigma^2 C_{SS}^{-1})^\top, ~
    \sqrt{n} \betahat_{S^c} \overset{d}{\rightarrow} 0.
\end{align}

On the other hand, this indicates that
\begin{align}
    \forall j \in S, ~ P(j \in \supp(\betahat)) \rightarrow 1.
\end{align}
Now, we consider the event $j \in S^c$ and $j \in \supp (\betahat)$.
By the KKT conditions,
\begin{align}
    &2 \x_j^\top (y - X \betahat) + \lambda_n \sgn(\betahat_j) + \eta_n \sgn(\betahat_j - \betatilde_j) = 0\\
    \Rightarrow & \frac{2}{\sqrt{n}} \x_j^\top (y - X\betahat)
    + \frac{1}{\sqrt{n}} \lambda_n \sgn(\betahat_j) + \frac{\eta_n}{\sqrt{n}} \sgn(\betahat_j - \betatilde_j) = 0\\
    \Rightarrow & 2 \left(\frac{1}{n} \x_j^\top X \right)  \sqrt{n} (\betastar - \betahat) + 2 \frac{1}{\sqrt{n}} \x_j^\top \varepsilon
    + \frac{\lambda_n}{\sqrt{n}} \sgn(\betahat_j) + \frac{\eta_n}{\sqrt{n}} \sgn(\betahat_j - \betatilde_j) = 0
\end{align}
Since we assume $\betatilde_j = 0$ for $j \in S^c$, we have
\begin{align}
    2 \left(\frac{1}{n} \x_j^\top X \right)  \sqrt{n} (\betastar - \betahat) + 2 \frac{1}{\sqrt{n}} \x_j^\top \varepsilon
    + \frac{\lambda_n + \eta_n}{\sqrt{n}} \sgn(\betahat_j) = 0
\end{align}
The first and second terms on the left-hand side converge to some normal distribution, but the the last term on the left-hand side diverges to infinity if $(\lambda_n + \eta_n) / \sqrt{n} \rightarrow \infty$.
Hence, we have for $\forall j \in S^c$,
\begin{align}
    \lim_{n\rightarrow \infty} P\left( j \in \supp (\betahat) \right) 
        = 0.
\end{align}
This concludes 
\begin{align}
    \lim_{n\rightarrow \infty} P(\hat{S}_n = S) 
        = 1
\end{align}

On the other hand, 
\eqref{trlasso} is symmetric in terms of $0$ and $\betatilde$.
If we reparameterize parameters as
\begin{align}
    {\beta^*}' := \betastar - \betatilde,~
    \betatilde' := - \betatilde,~
    \betahat_n' := \betahat_n - \betatilde,~
    y' := y - X \betatilde,
\end{align}
\begin{align}
    \hat{S}_n' := \{ j : \betahat_j' \neq 0 \}, ~
    S' := \{ j : {\betastar_j}' \neq 0 \},
\end{align}
\begin{align}
    \lambda_n' := \eta_n,~
    \eta_n' := \lambda_n,
\end{align}
we have
\begin{align}
    \betahat_n' &= \betahat_n - \betatilde\\
    &= \argmin_{\beta'} \frac{1}{n} \left\|y - X (\beta'+\betatilde)\right\|_2^2 + \frac{\lambda_n}{n} \left\|(\beta'+\betatilde)\right\|_1 + \frac{\eta_n}{n} \left\|(\beta'+\betatilde) - \betatilde'\right\|_1\\
    &= \argmin_{\beta'} \frac{1}{n} \left\|y' - X \beta'\right\|_2^2 + \frac{\lambda_n'}{n} \left\|\beta'\right\|_1 + \frac{\eta_n'}{n} \left\|\beta' - \betatilde'\right\|_1
\end{align}
Hence, the property for $\{\betahat'_n, {\betastar}', \hat{S}_n', S'\}$ is the same as that for $\{\betahat_n, \betastar, \hat{S}_n, S\}$ in Theorem \ref{trlasso-oracle}.
In addition, we have
\begin{align}
    \hat{S}_n' 
    =\{ j : \betahat_j' \neq 0 \} 
    = \{ j : \betahat_j \neq \betatilde_j \}
    = \hat{T}_n, ~
    S'
    =\{ j : {\beta^*_j}' \neq 0 \}
    = \{ j : \betastar_j \neq \betatilde_j \}
    = T.
\end{align}
This concludes
\begin{align}
    \lim_{n\rightarrow \infty} P(\hat{T}_n = T)
    = \lim_{n\rightarrow \infty} P(\hat{S}_n' = S') 
    = 1
\end{align}

\subsection{Proofs of Transfer Lasso with Initial Estimator in Boundary Region}

\subsubsection{Proofs of Theorem~\ref{theorem:tlasso-asymptotic-distribution-boundary}}
\label{proof:tlasso-asymptotic-distribution-boundary}

Let $u := \sqrt{n} (\beta - \betastar)$, $V_n(u) := n (Z_n^\mathcal{T}(\beta) - Z_n^\mathcal{T}(\betastar))$, and $V(u) := \lim_{n\to\infty} V_n(u)$.
Consider the case where 
$\lambda_n / \sqrt{n} \to \infty$, $\lambda_n / \eta_n \to 1$, and
$(\lambda_n - \eta_n) / \sqrt{n} \to \delta_0$.
Suppose that $n / m \to 0$.
In the same way as in Case II of Proof~\ref{proof:trlasso-consistency} (Theorem~\ref{theorem:trlasso-consistency}), we obtain $V_n(u)$ as \eqref{eq:trlasso-vnu-ii}.
Because
\begin{align}
    \left( \left| u_j + \sqrt{m} \betastar_j \right| - \left| \sqrt{m} \betastar_j \right| \right) 
    \to u_j \sgn(\betastar_j) I(\betastar_j \neq 0) + |u_j| I(\betastar_j = 0),
\end{align}
\begin{align}
    \left( \left| u_j - \sqrt{\frac{n}{m}} z_j \right| - \left| \sqrt{\frac{n}{m}} z_j \right| \right) 
    \to |u_j|,
\end{align}
and
\begin{align}
    &\frac{\lambda_n}{\sqrt{n}} \left( u_j \sgn(\betastar_j) I(\betastar_j \neq 0) + |u_j| I(\betastar_j = 0) \right) + \frac{\eta_n}{\sqrt{n}} |u_j|\\
    =& \left( \frac{\lambda_n}{\sqrt{n}} u_j \sgn(\betastar_j) + \frac{\eta_n}{\sqrt{n}} |u_j| \right) I(\betastar_j \neq 0) + \frac{\lambda_n + \eta_n}{\sqrt{n}} |u_j| I(\betastar_j = 0),
\end{align}
we have
\begin{align}
    \label{eq:trlasso-vu-iii}
    V(u)
    = \begin{dcases}
        u^\top C u - 2 u^\top W - \delta_0 \sum_j |u_j| I(\betastar_j \neq 0), \\
        ~~~~~~~\text{if} ~ \betastar_j u_j \leq 0 ~\text{for}~ \forall j \in S ~\text{and}~ u_j = 0 ~\text{for}~ \forall j \in S^c,\\
        \infty, ~ \text{otherwise}.
    \end{dcases}
\end{align}
Since $V_n(u)$ is convex and $V(u)$ has a unique minimum, we obtain
\begin{align}
    \sqrt{n} \left( \betahat_n - \betastar \right)
    \to
    \argmin_{u \in \mathcal{U}} \left\{ u^\top C u - 2 u^\top W - \delta_0 \sum_j |u_j| \right\},
\end{align}
\begin{align}
    \mathcal{U} := \left\{ u\in \mathbb{R}^p \left|
        \begin{array}{l}
            \betastar_j u_j \leq 0 ~\text{for}~ \forall j \in S,\\
            u_j = 0 ~\text{for}~ \forall j \in S^c
        \end{array}
    \right. \right\}.
\end{align}

On the other hand, the Lagrangian function of $\argmin_{u \in \mathcal{U}} V(u)$ is
\begin{align}
    u_S^\top C_{SS} u_S - 2 u_S^\top W_S - \delta_0 \sum_{j \in S} |u_j| + \sum_{j \in S} \mu_j \betastar_j u_j + \sum_{j \in S^c} \mu_j u_j,
\end{align}
where $\mu \in \mathbb{R}^p$ is the Lagrangian multiplier.
Suppose that $u^*_{S^c} = 0$.
Let 
$S_1 := \{ j : j \in S ~\text{and}~ u_j^* \neq 0 \}$ and
$S_2 :=  \{ j : j \in S ~\text{and}~ u_j^* = 0 \}$.
By the KKT conditions of $\argmin_{u \in \mathcal{U}} V(u)$, we have
\begin{align}
\label{eq:trlasso-kkt-iii}
    \begin{dcases}
        2 C_{S_1 S_1} u_{S_1}^* - 2 W_{S_1} - \delta_0 \sgn(u_{S_1}^*) + \mu_{S_1} \betastar_{S_1} = 0\\
        \left| 2 C_{S_2 S_1} u_{S_1}^* - 2 W_{S_2} + \mu_{S_2} \betastar_{S_2} \right| \leq \delta_0\\
        \betastar_S u_S^* \leq 0\\
        \mu_S \geq 0\\
        \mu_S \betastar_S u_S^* = 0\\
        \mu_{S^c} = u_{S^c}^* = 0
    \end{dcases}
\end{align}
If $S_1 \neq \emptyset$, then we have
\begin{align}
    u_{S_1}^* = C_{S_1 S_1}^{-1} \left( W_{S_1} + \delta_0 \sgn(u_{S_1}^*) \right).
\end{align}
The probability holding the third inequality in \eqref{eq:trlasso-kkt-iii} is less than $1$.
This indicates inconsistent variable selection.
If $S_1 = \emptyset$, then we have
\begin{align}
    \left| - 2 W_{S_2} + \mu_{S_2} \betastar_{S_2} \right| \leq \delta_0.
\end{align}
The probability holding this inequality is less than $1$.
This indicates inconsistent variable selection.

\section{Additional Empirical Results}
\label{sec:additional-experiments}
We describe additional empirical results.

\subsection{Inconsistent Initial Estimator}

\begin{figure*}[t]
  \centering
  \includegraphics[width=10cm]{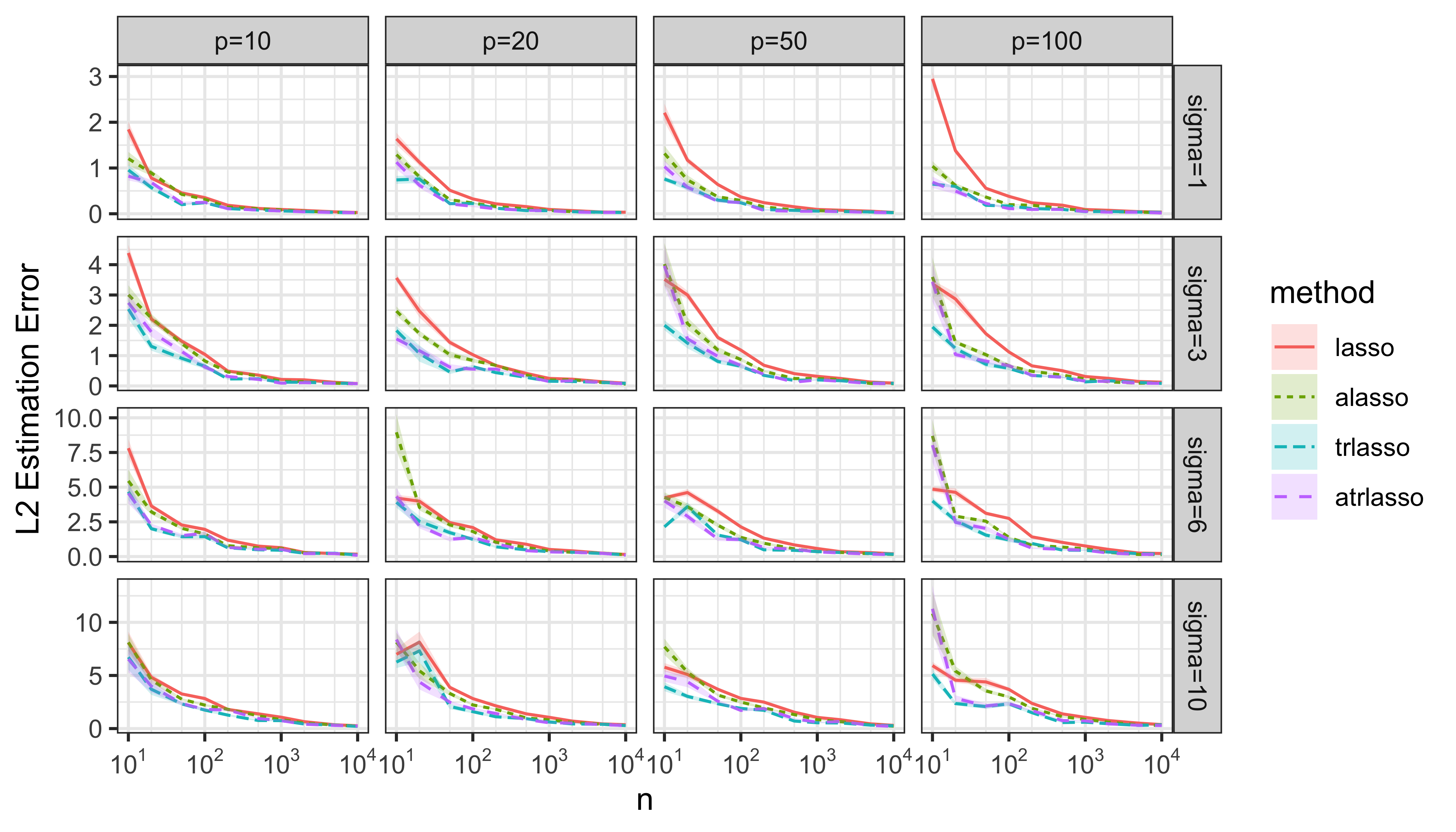}
  \caption{$\ell_2$ estimation errors for inconsistent source data (Case A).}
  \label{fig:addition-estimation}
  \vspace{0.5cm}
  \includegraphics[width=10cm]{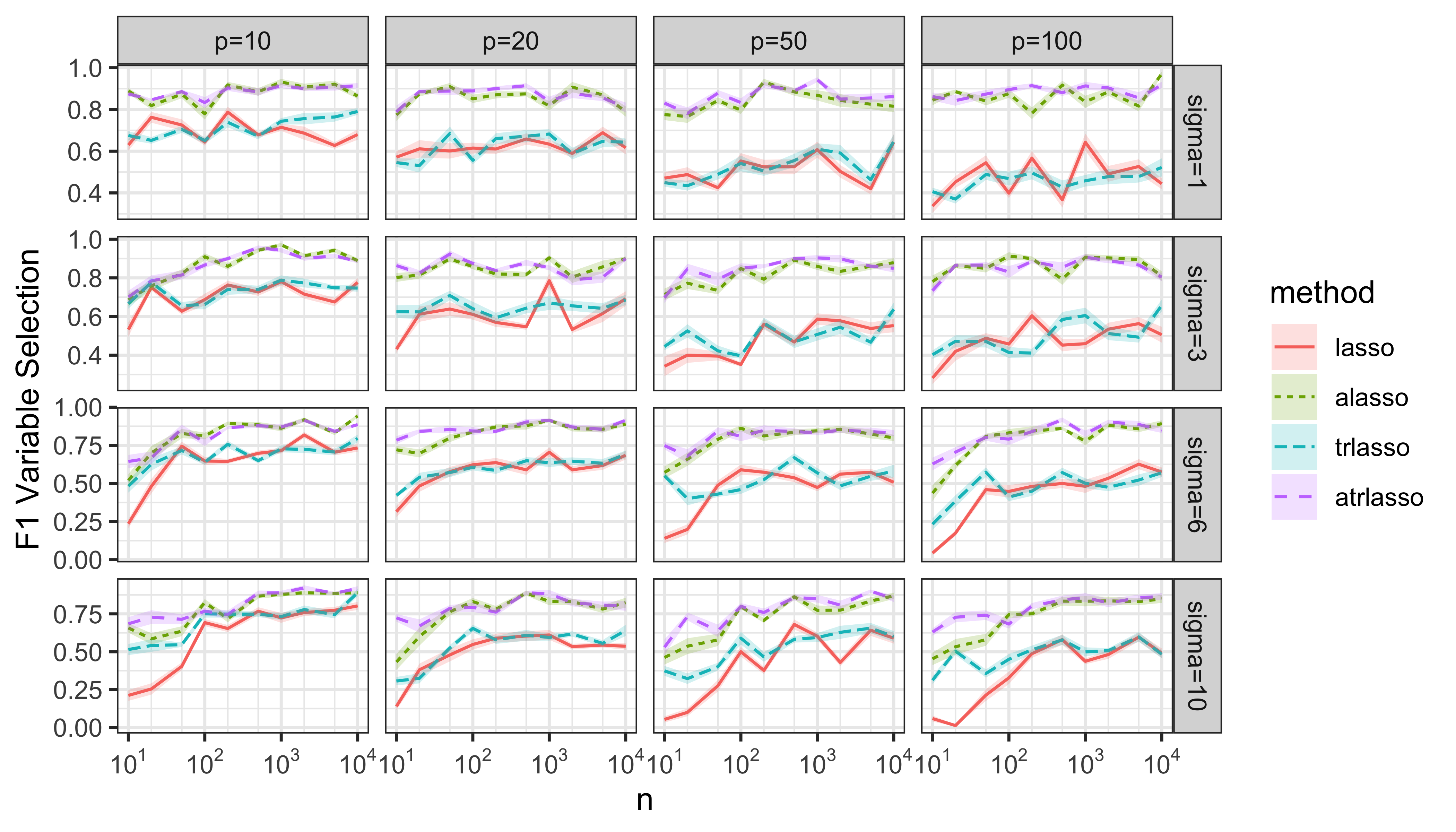}
  \caption{Variable selection F1-scores for inconsistent source data (Case A).}
  \label{fig:addition-selection}
\end{figure*}

\begin{figure*}[t]
  \centering
  \includegraphics[width=10cm]{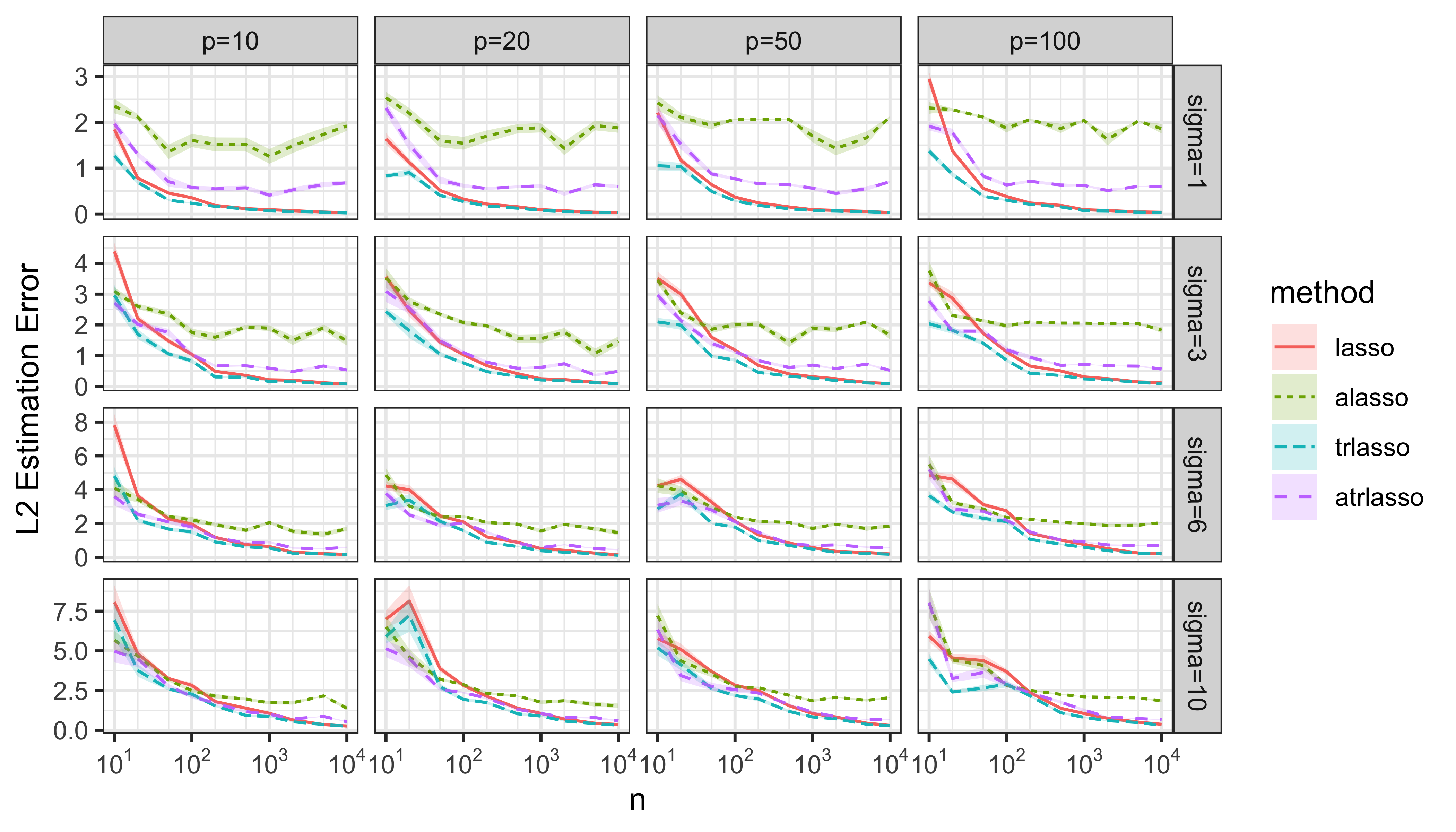}
  \caption{$\ell_2$ estimation errors for inconsistent source data (Case B).}
  \label{fig:removal-estimation}
  \vspace{0.5cm}
  \includegraphics[width=10cm]{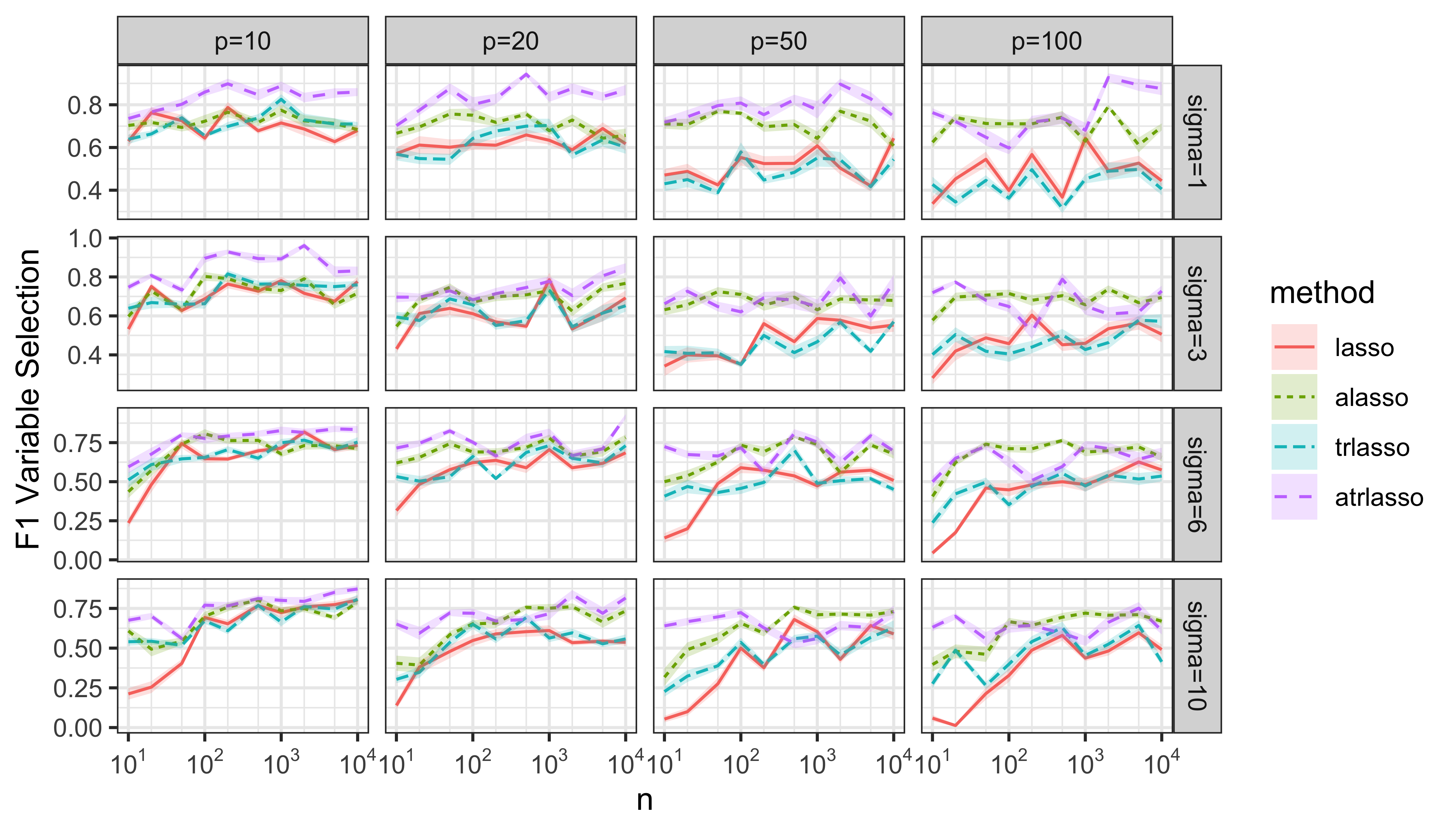}
  \caption{Variable selection F1-scores for inconsistent source data (Case B).}
  \label{fig:removal-selection}
\end{figure*}

In this simulation, we simulated inconsistent initial estimators.
We generated target data by $\betastar_{target} = [3, 1.5, 0, 0, 2, 0, 0, \dots, 0]^\top$.
Then, we generated the source data with parameters different from those of the target data.
We considered the following two cases:
\begin{description}
    \item[Case A] $\betastar_{source} = [3, 1.5, 0, 0, 2, 2, 0, \dots, 0]^\top$ (non-zero $\rightarrow$ zero change for $j=6$)
    \item[Case B] $\betastar_{source} = [3, 1.5, 0, 0, 0, 0, 0, \dots, 0]^\top$ (zero $\rightarrow$ non-zero change for $j=5$)
\end{description}
Other simulation settings were the same as those in Simulation~I.

The results are shown in Figures~\ref{fig:addition-estimation} and \ref{fig:addition-selection} for Case~A, and Figures~\ref{fig:removal-estimation} and \ref{fig:removal-selection} for Case~B.
In Case~A, the Transfer Lasso and the Adaptive Transfer Lasso were superior in estimation error, and the Adaptive Lasso was slightly inferior to the Transfer Lasso and the Adaptive Transfer Lasso.
Moreover, the Adaptive Lasso and the Adaptive Transfer Lasso were still superior in variable selection.
In Case B, the Adaptive Lasso and the Adaptive Transfer Lasso performed worse in terms of estimation error.
This may be because the initial estimator was incorrectly estimated close to zero, and regularization was strongly applied to it.
The performance degradation is particularly significant for the Adaptive Lasso, but not so significant for Adaptive Transfer Lasso.
The superiority of the Adaptive Lasso and the Adaptive Transfer Lasso also diminished in variable selection.
Overall, the Adaptive Transfer Lasso performed comparatively well in estimation error and variable selection, but the inconsistent initial estimators reduced its performance.

\subsection{Other Initial Estimators}

\label{sec:other-initial-estimators}

We compared other initial estimators for simulations of a large amount of target data.
The initial estimators included
\begin{itemize}
    \item Ridge (Figure~\ref{fig:large-l2-ridge}, \ref{fig:large-f1-ridge})
    \item Ridgeless~\cite{belkin2020two,hastie2022surprises}: minimum $\ell_2$-norm solution of least squares (Figure~\ref{fig:large-l2-ridgeless}, \ref{fig:large-f1-ridgeless})
    \item Lassoless~\cite{mitra2019understanding,li2021minimum}: minimum $\ell_1$-norm solution of least squares (Figure~\ref{fig:large-l2-lassoless}, \ref{fig:large-f1-lassoless})
\end{itemize}

The results of Ridge initial estimators were similar to those of Lasso initial estimators.
The results of Ridgeless and Lassoless initial estimators showed double descent phenomena, but they did not perform as well as Lasso and Ridge.

\begin{figure*}[t]
  \centering
  \includegraphics[width=10cm]{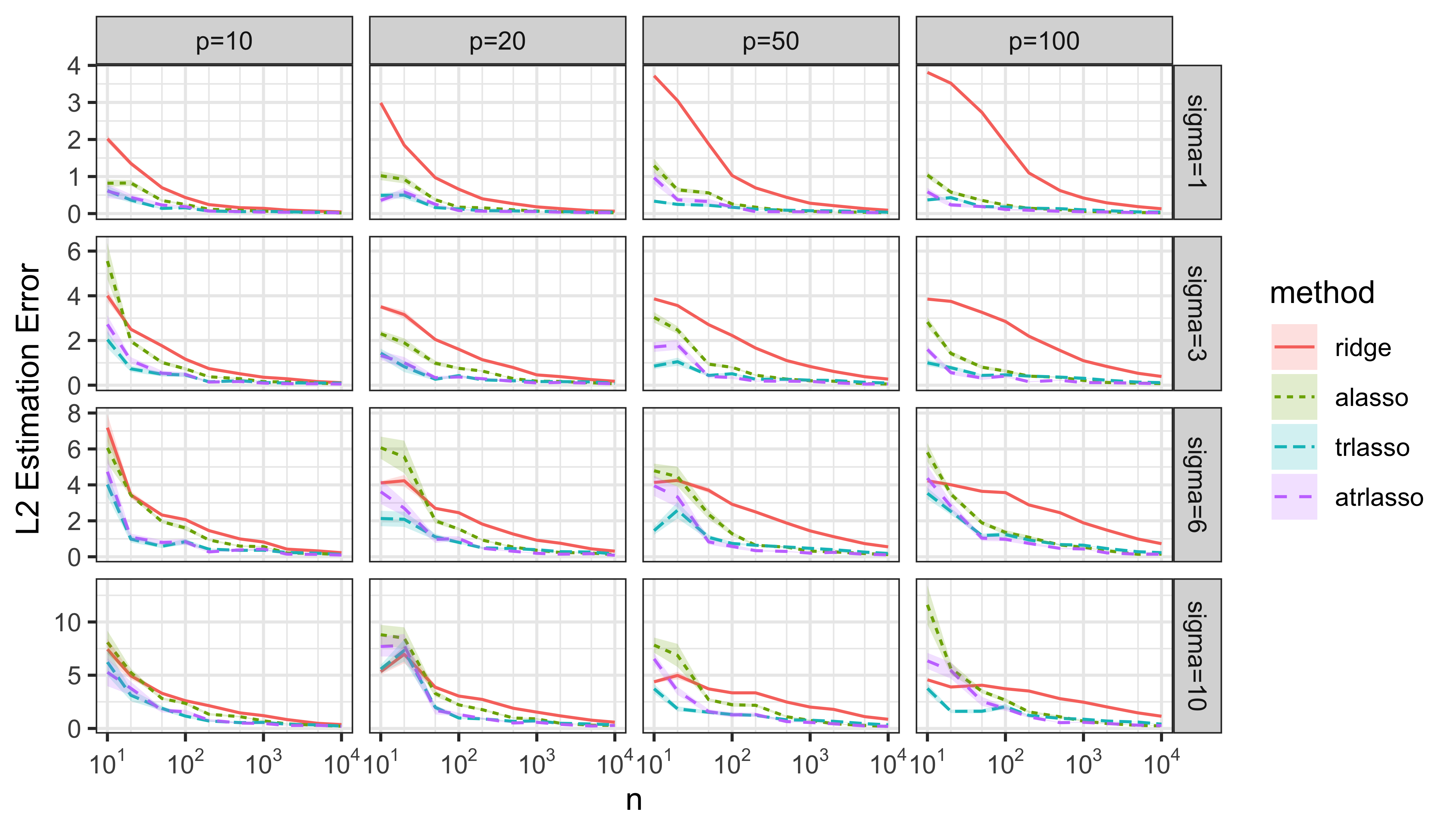}
  \caption{$\ell_2$ estimation errors for a large amount of source data and Ridge initial estimators.}
  \label{fig:large-l2-ridge}
  \vspace{0.5cm}
  \includegraphics[width=10cm]{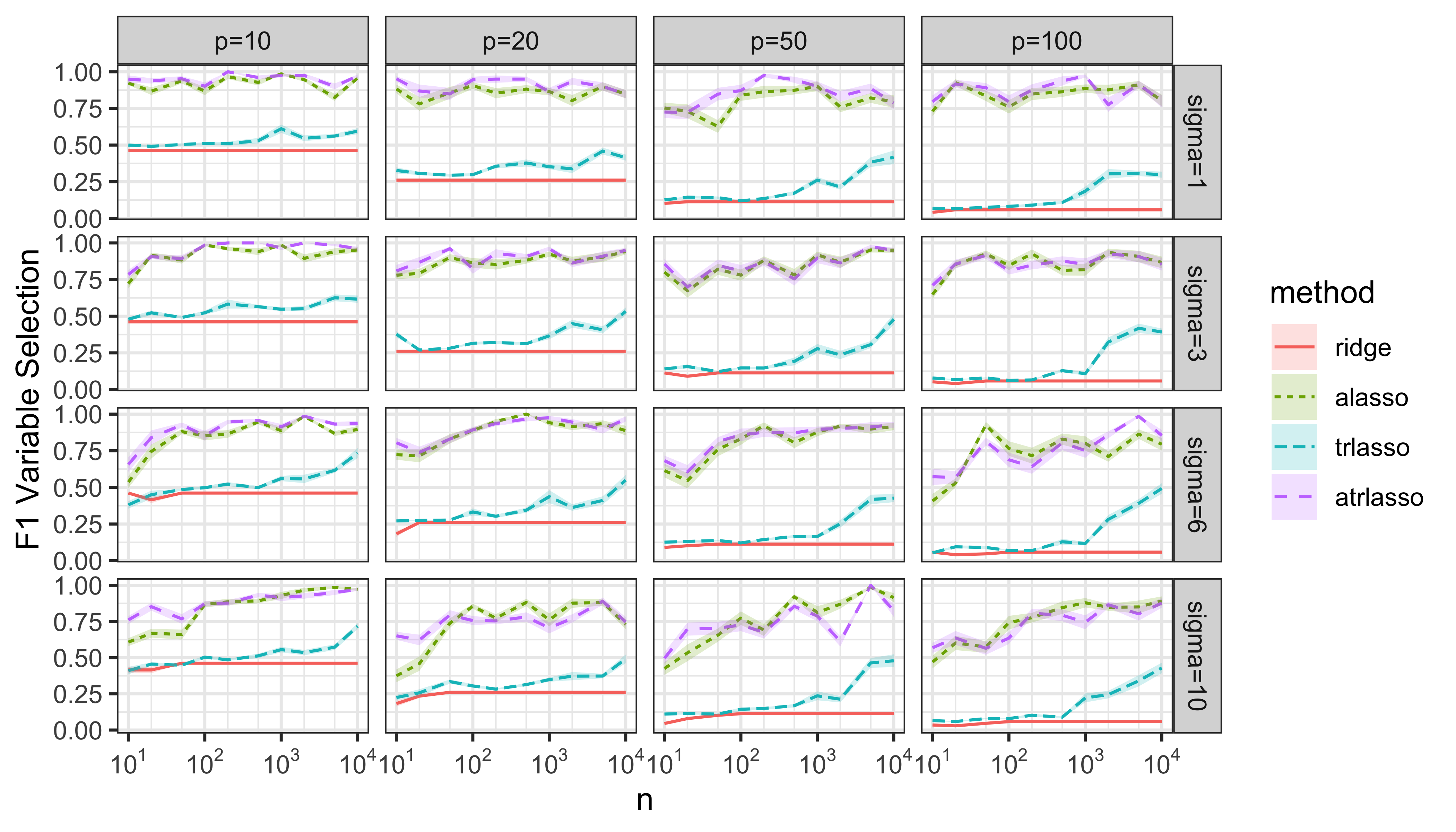}
  \caption{Variable selection F1-score for a large amount of source data and Ridge initial estimators.}
  \label{fig:large-f1-ridge}
\end{figure*}

\begin{figure*}[t]
  \centering
  \includegraphics[width=10cm]{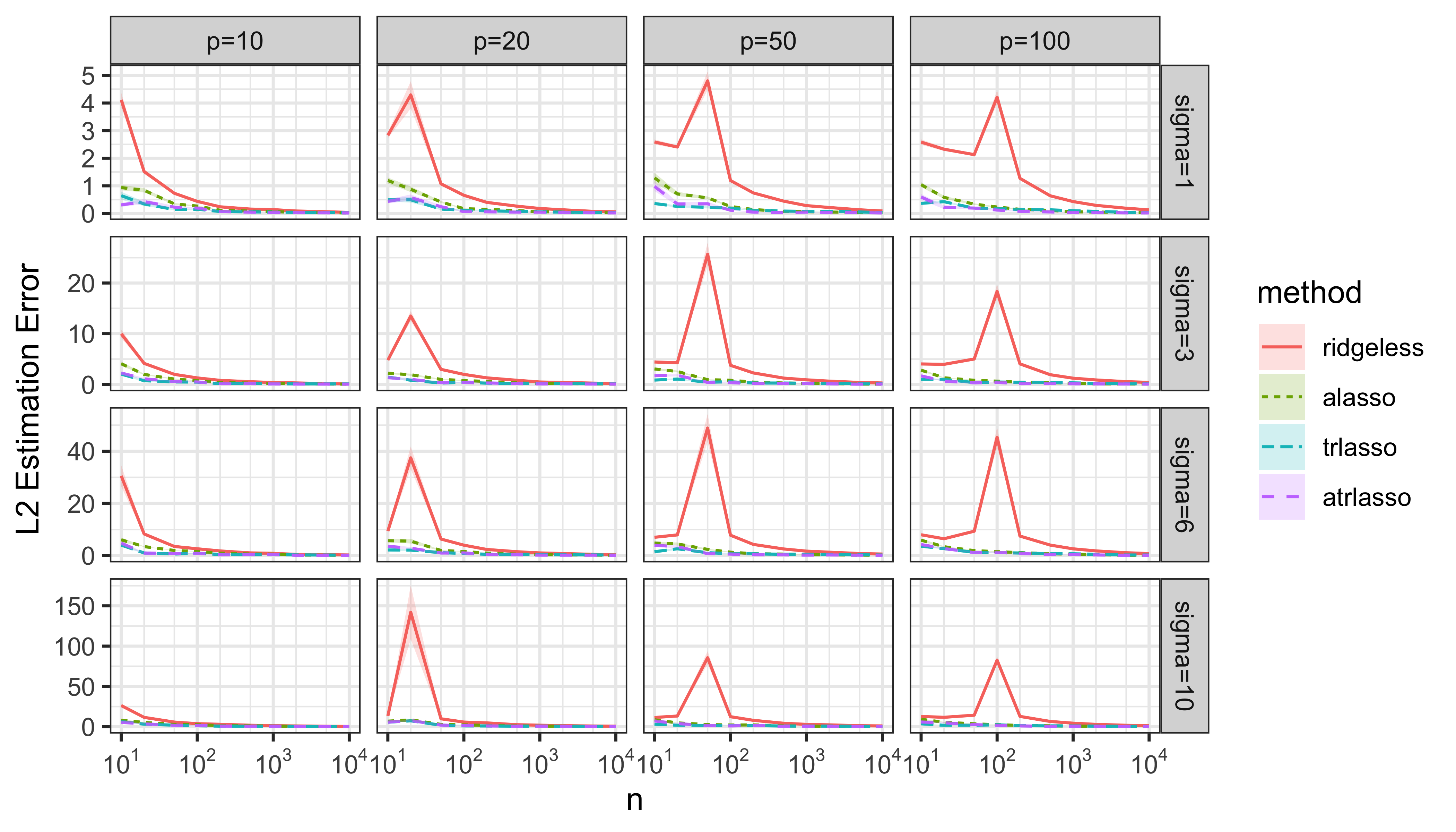}
  \caption{$\ell_2$ estimation errors for a large amount of source data and Ridgeless initial estimators.}
  \label{fig:large-l2-ridgeless}
  \vspace{0.5cm}
  \includegraphics[width=10cm]{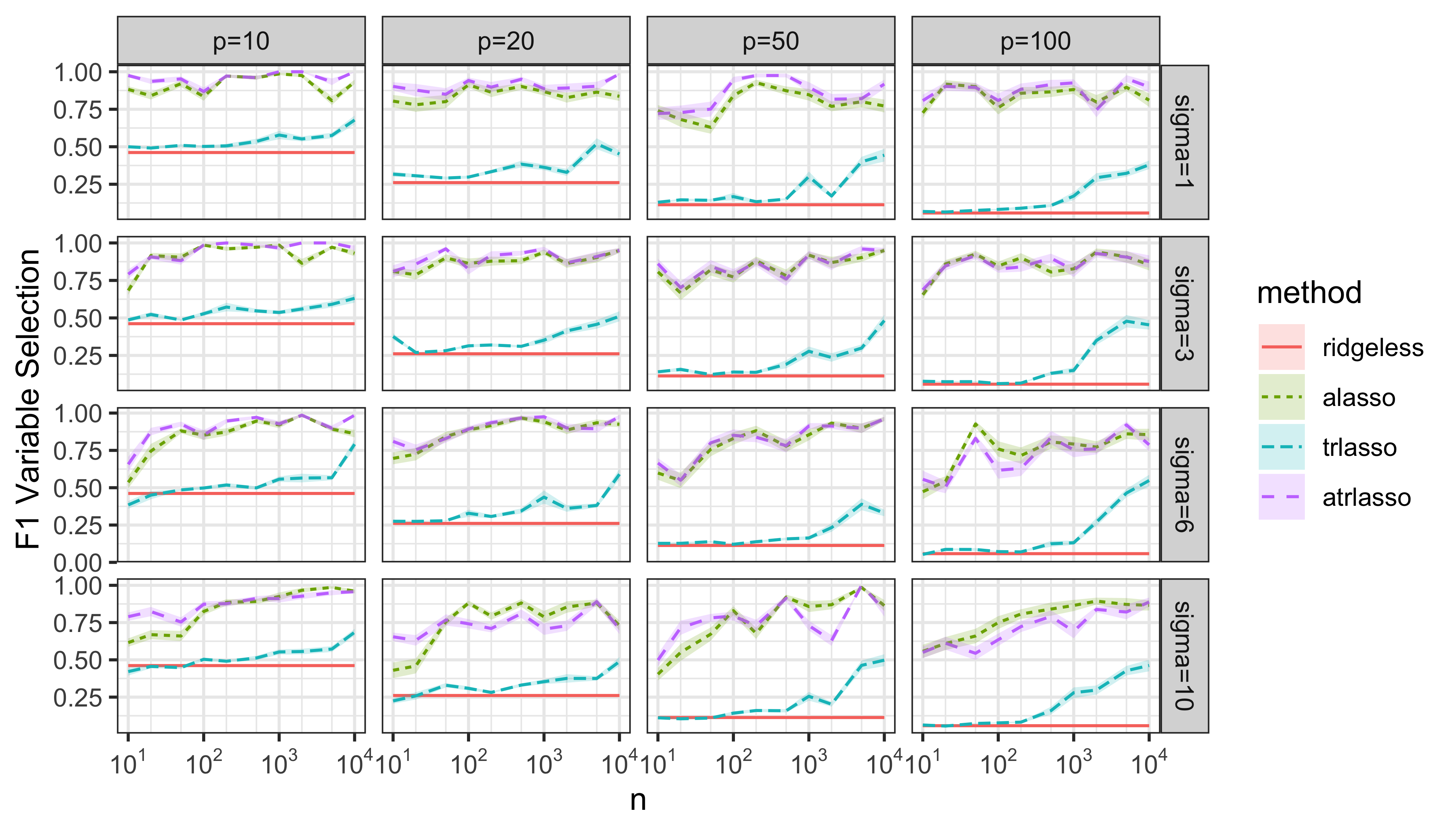}
  \caption{Variable selection F1-score for a large amount of source data and Ridgeless initial estimators.}
  \label{fig:large-f1-ridgeless}
\end{figure*}

\begin{figure*}[t]
  \centering
  \includegraphics[width=10cm]{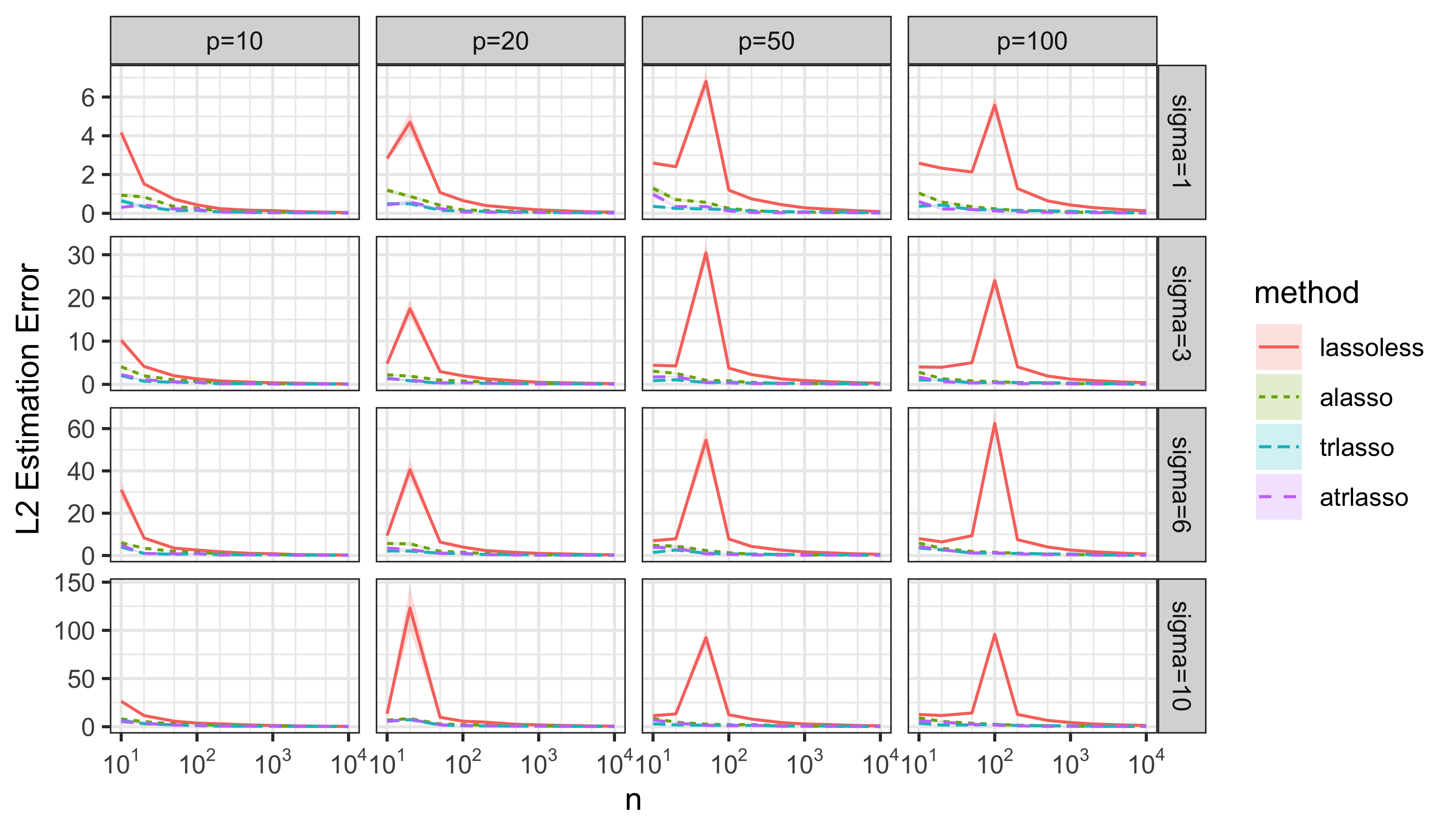}
  \caption{$\ell_2$ estimation errors for a large amount of source data and Lassoless initial estimators.}
  \label{fig:large-l2-lassoless}
  \vspace{0.5cm}
  \includegraphics[width=10cm]{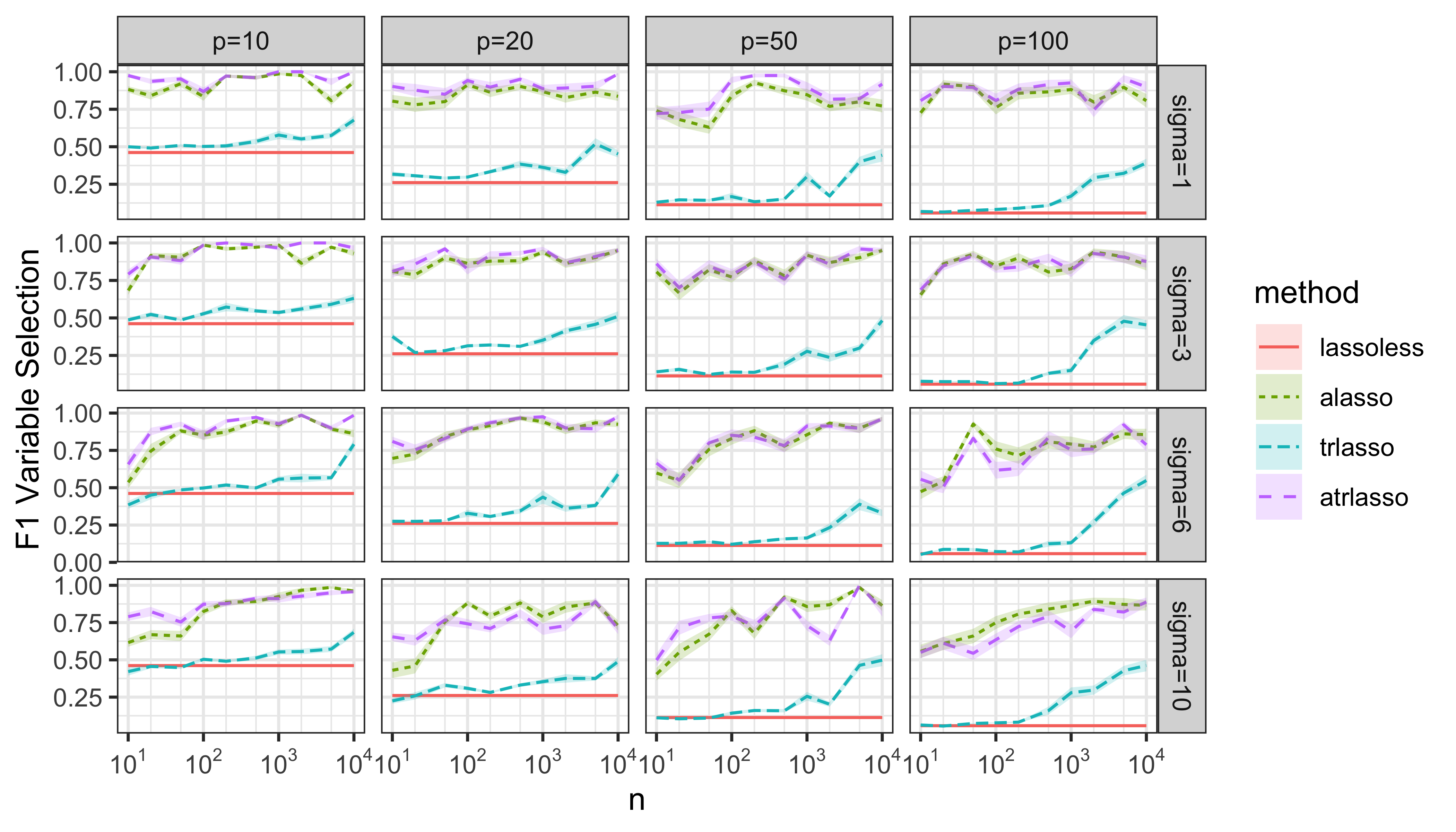}
  \caption{Variable selection F1-score for a large amount of source data and Lassoless initial estimators.}
  \label{fig:large-f1-lassoless}
\end{figure*}

\subsection{Other Metrics}

\label{sec:other-metrics}

We evaluated other metrics for Simulation I (Section~\ref{sec:sim-1}).
The metrics included
\begin{itemize}
    \item RMSE for prediction evaluation (Figure~\ref{fig:large-rmse})
    \item sensitivity (= (\# of correct selected variables) / (\# of true active variables)) for feature selection evaluation (Figure~\ref{fig:large-sens})
    \item specificity (= (\# of correctly not selected variables) / (\# of true inactive variables)) for feature selection (Figure~\ref{fig:large-spec})
    \item positive predictive value (= (\# of correctly selected variables) / (\# of selected variables)) for feature selection evaluation (Figure~\ref{fig:large-ppv})
    \item number of active variables for feature selection evaluation (Figure~\ref{fig:large-nav})
\end{itemize}

The results of RMSE were similar to those of $\ell_2$ estimation errors, but the difference among methods got smaller.
The results of 4 metrics for feature selection showed that Adaptive Lasso and Adaptive Transfer Lasso selected small number of variables and achieved superior performance especially on specificity and positive predictive value.

\begin{figure*}[t]
  \centering
  \includegraphics[width=10cm]{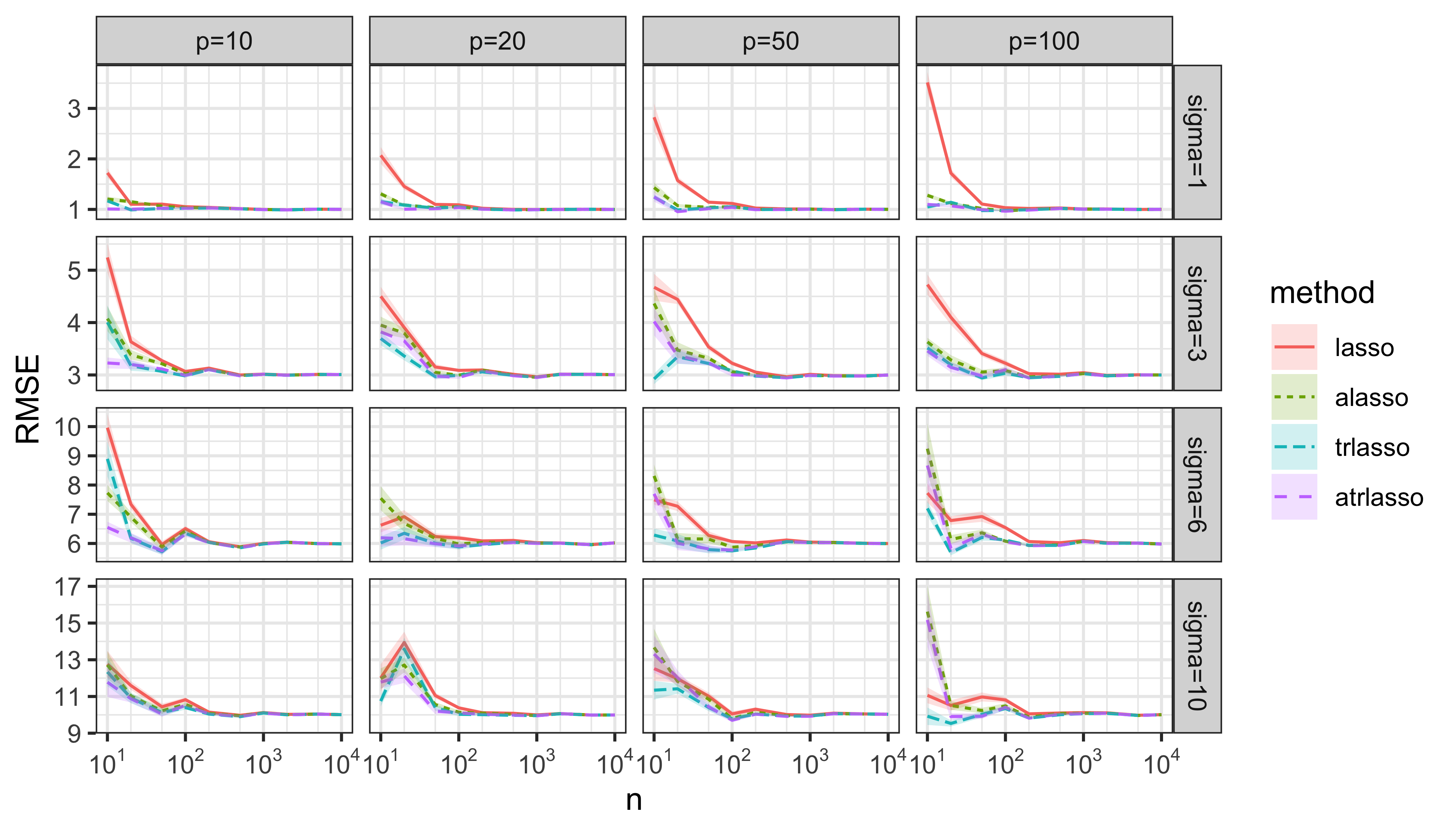}
  \caption{RMSE for a large amount of source data.}
  \label{fig:large-rmse}
\end{figure*}

\begin{figure*}[t]
  \centering
  \includegraphics[width=10cm]{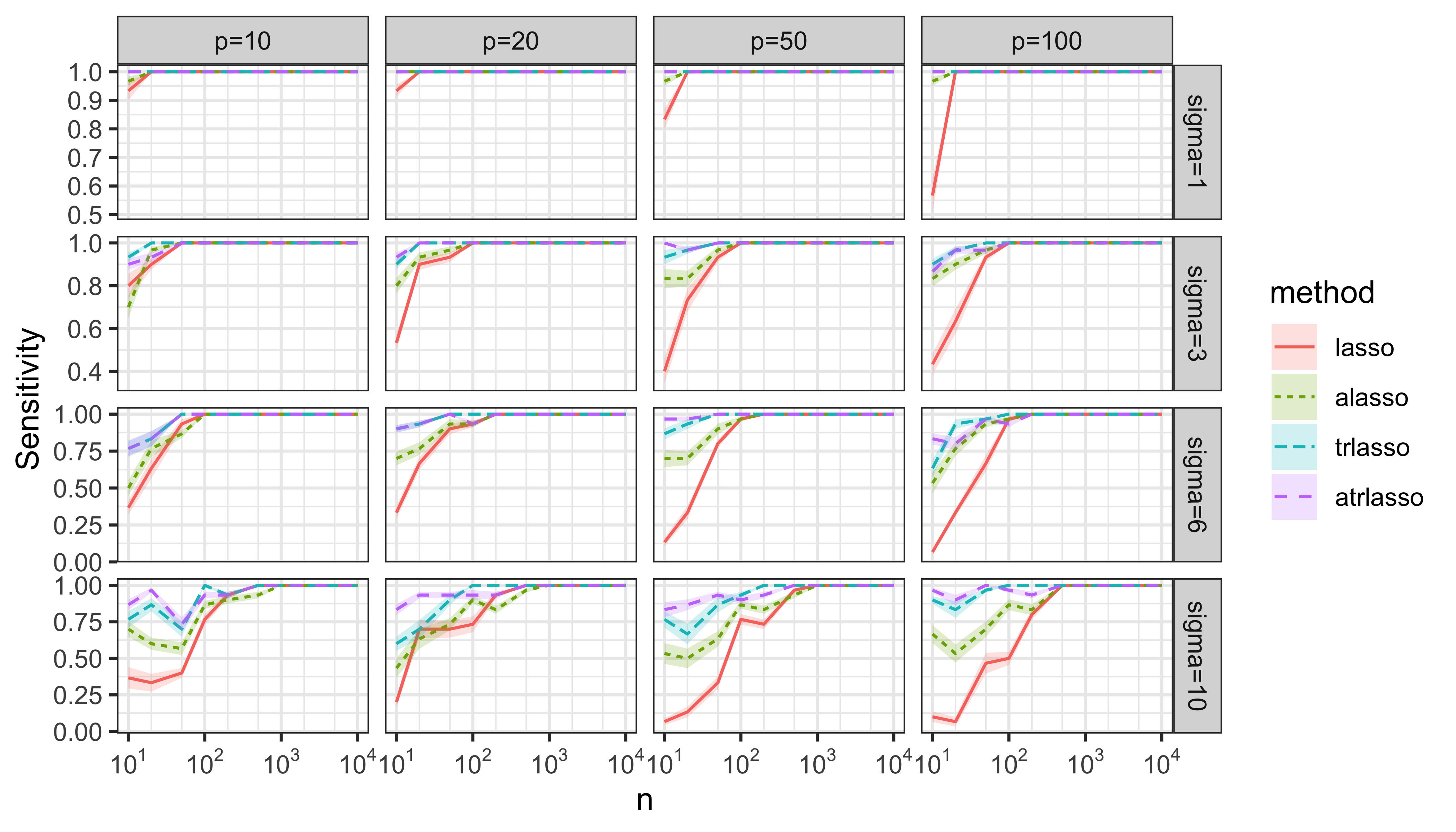}
  \caption{Sensitivity for a large amount of source data.}
  \label{fig:large-sens}
\end{figure*}

\begin{figure*}[t]
  \centering
  \includegraphics[width=10cm]{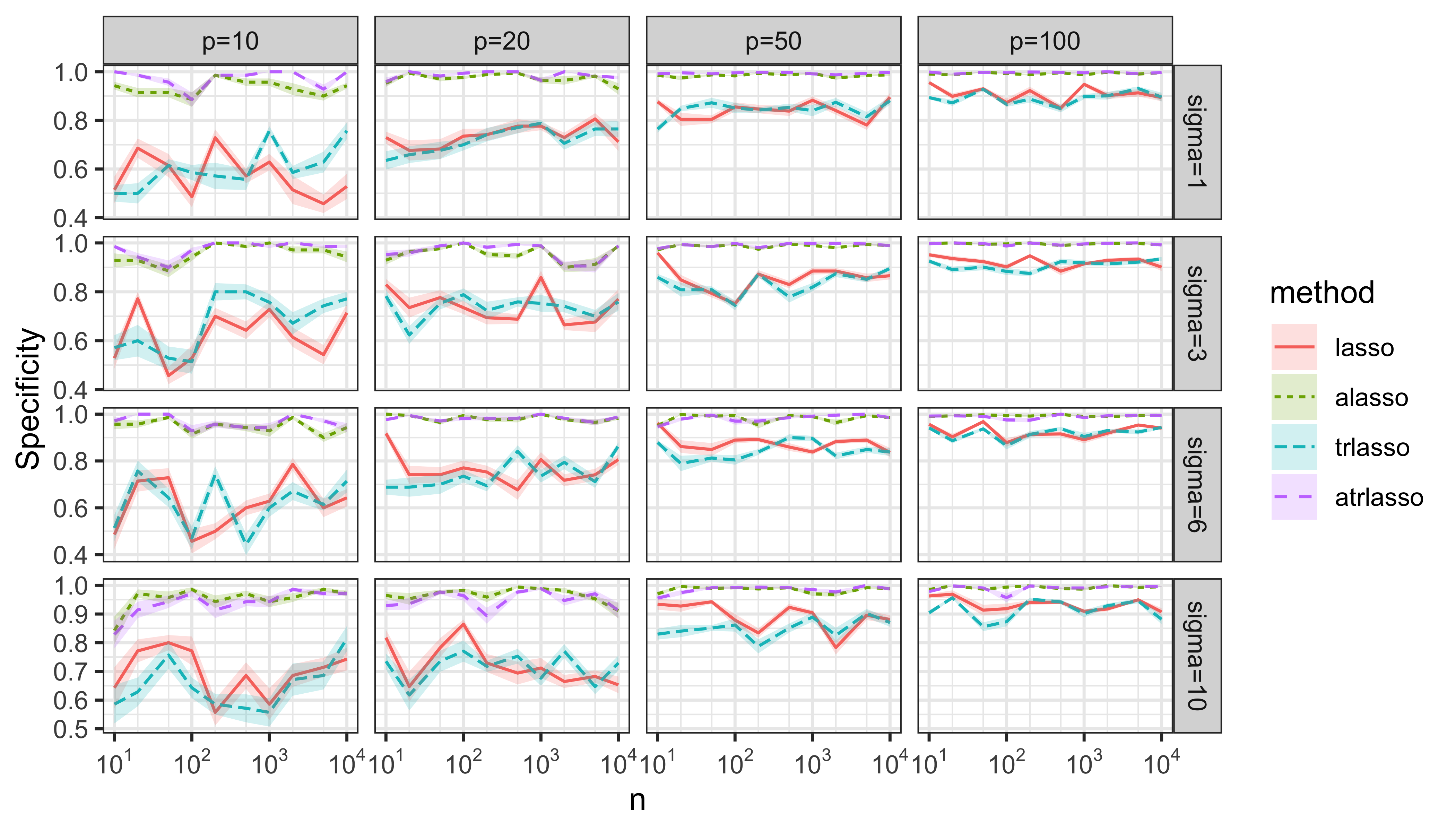}
  \caption{Specificity for a large amount of source data.}
  \label{fig:large-spec}
\end{figure*}

\begin{figure*}[t]
  \centering
  \includegraphics[width=10cm]{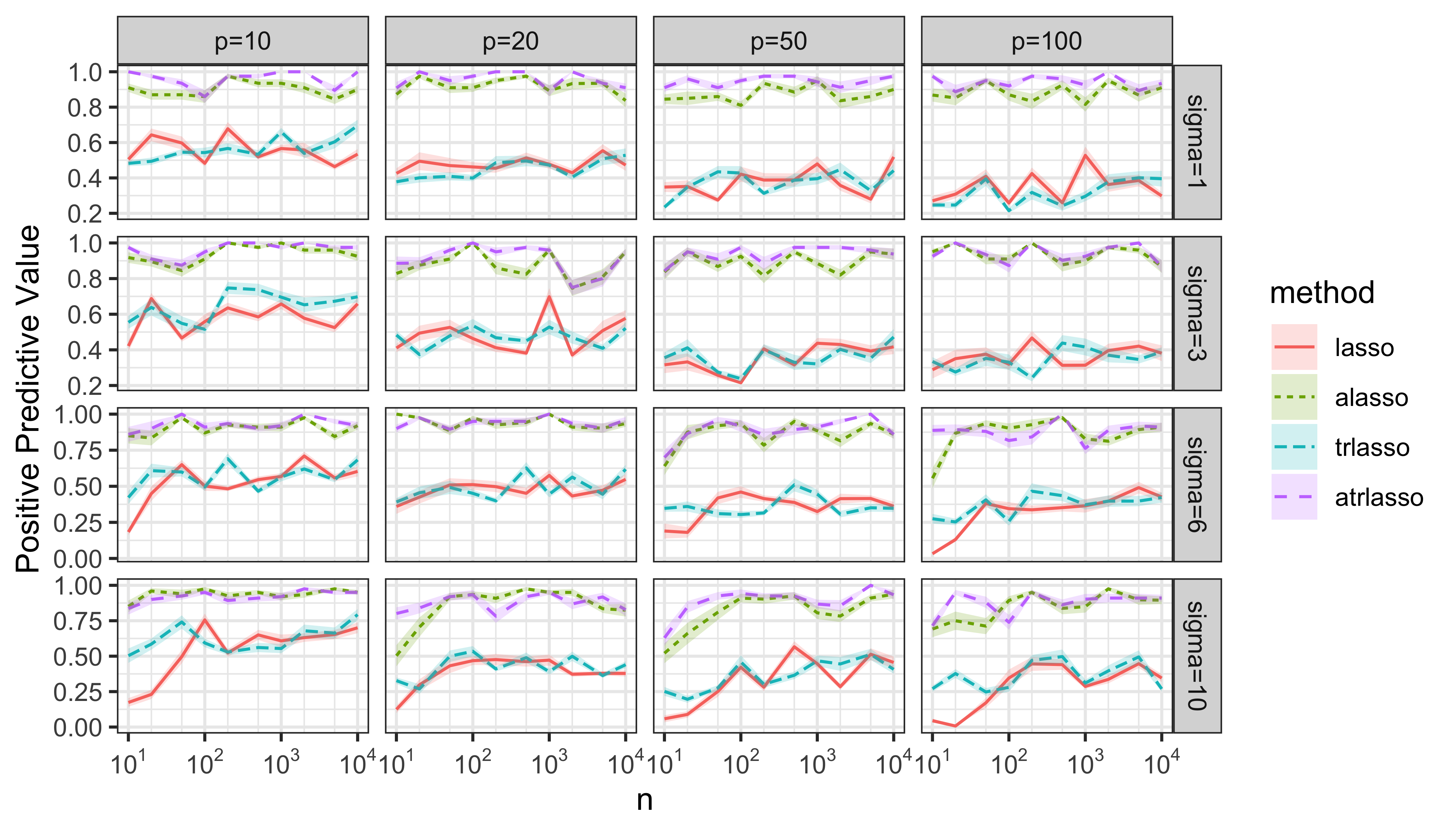}
  \caption{Positive predictive value for a large amount of source data.}
  \label{fig:large-ppv}
\end{figure*}

\begin{figure*}[t]
  \centering
  \includegraphics[width=10cm]{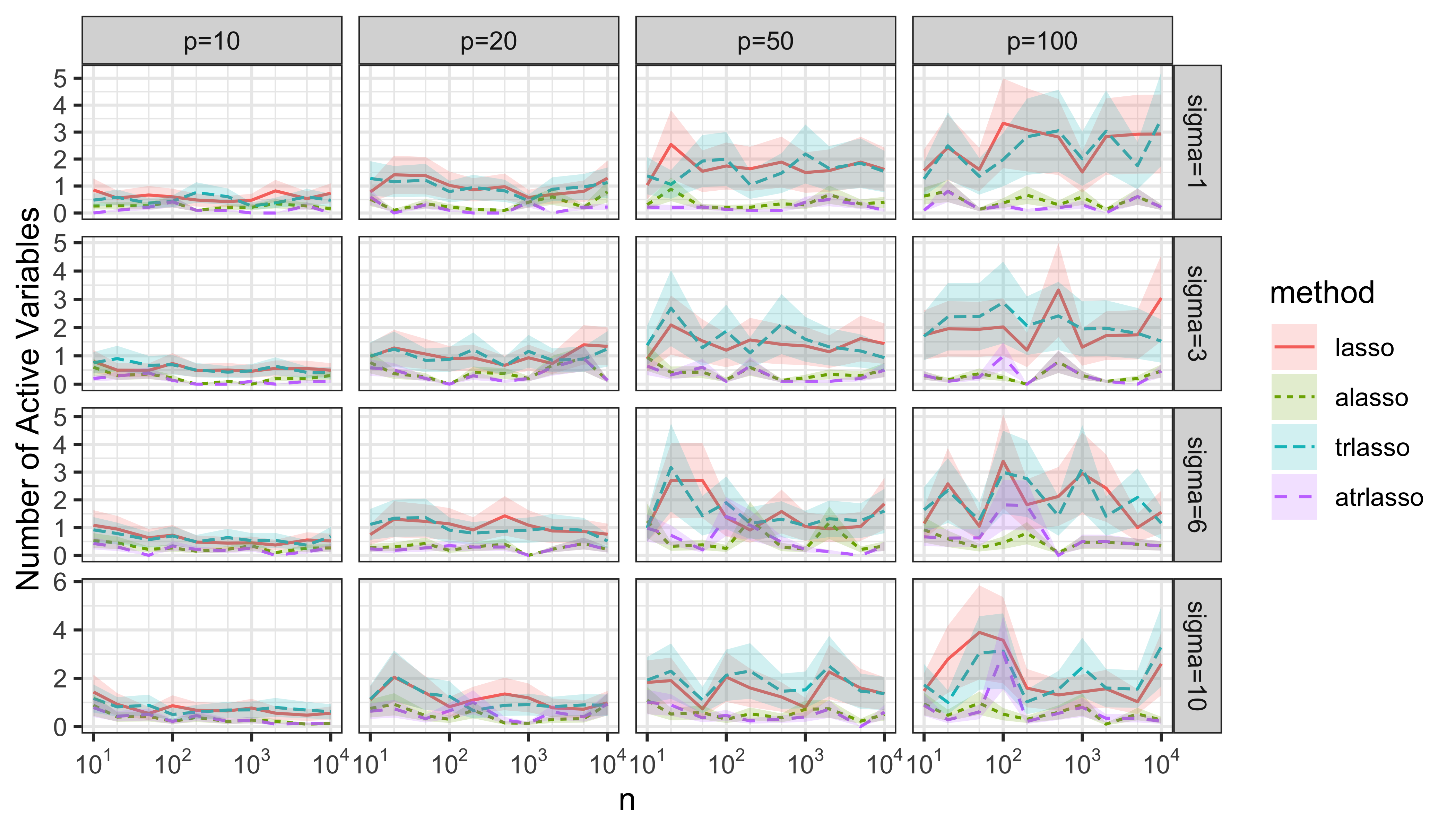}
  \caption{Number of active variables for a large amount of source data.}
  \label{fig:large-nav}
\end{figure*}

\end{appendix}

\end{document}